\documentclass{article}
\PassOptionsToPackage{numbers, compress}{natbib}
\usepackage[preprint]{neurips_2024}
\usepackage{microtype}
\usepackage{graphicx}
\usepackage{subfigure}
\usepackage{booktabs,dsfont}
\usepackage{amsmath} 
\usepackage{amssymb}
\usepackage{mathtools}
\usepackage{amsthm}
\usepackage[colorlinks = true, citecolor = blue]{hyperref}
\usepackage[capitalize,noabbrev]{cleveref}
\usepackage[utf8]{inputenc} 
\usepackage[T1]{fontenc}    

\usepackage{wrapfig,lipsum,booktabs}

\usepackage{paralist}
\usepackage{booktabs}       
\usepackage{nicefrac}       
\usepackage{multirow}
 \usepackage{xargs}

\usepackage{aliascnt}

\makeatletter
\newtheorem{theorem}{Theorem}
\crefname{theorem}{theorem}{Theorems}
\Crefname{Theorem}{Theorem}{Theorems}

\newtheorem{lemma}{Lemma}

\newaliascnt{corollary}{theorem}
\newtheorem{corollary}[corollary]{Corollary}
\aliascntresetthe{corollary}
\crefname{corollary}{corollary}{corollaries}
\Crefname{Corollary}{Corollary}{Corollaries}

\newaliascnt{proposition}{theorem}
\newtheorem{proposition}[proposition]{Proposition}
\aliascntresetthe{proposition}
\crefname{proposition}{proposition}{propositions}
\Crefname{Proposition}{Proposition}{Propositions}

\newaliascnt{definition}{theorem}
\newtheorem{definition}[definition]{Definition}
\aliascntresetthe{definition}
\crefname{definition}{definition}{definitions}
\Crefname{Definition}{Definition}{Definitions}

\newaliascnt{remark}{theorem}

\aliascntresetthe{remark}
\crefname{remark}{remark}{remarks}
\Crefname{Remark}{Remark}{Remarks}

\crefname{example}{example}{examples}
\Crefname{Example}{Example}{Examples}

\crefname{figure}{figure}{figures}
\Crefname{Figure}{Figure}{Figures}

\newcommand{\Yy}{\mathbf{Y}}

\newcommand{\Kk}{\mathbf{K}}
\newcommand{\Xx}{\mathbf{X}}

\newcommand{\Ii}{\mathbf{I}}

\newcommand{\LeftEqNo}{\let\veqno\@@leqno}

\newcommandx\sequence[3][2=,3=]
{\ifthenelse{\equal{#3}{}}{\ensuremath{( #1_{#2})}}{\ensuremath{( #1_{#2})_{ #2 \in #3}}}}

\newcommandx\sequencet[3][2=,3=]
{\ifthenelse{\equal{#3}{}}{\ensuremath{( #1_{#2})}}{\ensuremath{( #1_{#2})_{ #2 \geq #3}}}}

\newcommand{\opnorm}[1]{{\left\vert\kern-0.25ex\left\vert\kern-0.25ex\left\vert #1
    \right\vert\kern-0.25ex\right\vert\kern-0.25ex\right\vert}}

\newcommand\coupling[2]{\Gamma(\mu,\nu)}

\newcommand{\N}{\mathbb{N}}

\newcommand{\R}{\mathbb{R}}

\newcommand{\E}{\mathbb{E}}

\newcommand{\eqsp}{\,}  

\newcommand{\br}[1]{\left[{#1}\right]}

\makeatletter
\newcommand{\ostar}{\mathbin{\mathpalette\make@circled\star}}
\newcommand{\pstar}{\mathbin{\mathpalette\make@circled+}}
\newcommand{\make@circled}[2]{%
  \ooalign{$\m@th#1\smallbigcirc{#1}$\cr\hidewidth$\m@th#1#2$\hidewidth\cr}%
}
\newcommand{\smallbigcirc}[1]{%
  \vcenter{\hbox{\scalebox{0.77778}{$\m@th#1\bigcirc$}}}%
}
\makeatother

\newcommandx{\Vnorm}[2][1=V]{\| #2 \|_{#1}}

\usepackage{enumitem}
\usepackage[utf8]{inputenc} 
\usepackage[T1]{fontenc}    
\usepackage{url}            
\usepackage{booktabs}       
\usepackage{amsfonts}       
\usepackage{nicefrac}       
\usepackage{xcolor}  
\usepackage{algorithm}
\usepackage{algorithmic}
\usepackage{bbm}
\usepackage{multirow}

\DeclareMathOperator*{\sm}{<}
\DeclareMathOperator*{\bi}{>}

\theoremstyle{plain}
\newtheorem{assumption}{\textbf{H}\hspace{-3pt}}

\usepackage{autonum}
\allowdisplaybreaks

\title{DU-Shapley: A Shapley Value Proxy for \\ Efficient Dataset Valuation}

\author{%
  Felipe Garrido-Lucero*\\
  Inria, Fairplay joint team\\
  Palaiseau, France\\
  \texttt{felipe.garrido-lucero@irit.fr}\\
  * Equal contribution\\
  \And
  Benjamin Heymann*\\
  Criteo AI Lab\\
  Paris, France \\
  \texttt{b.heymann@criteo.com}\\
  * Equal contribution\\
  \And
  Maxime Vono* \\
  Criteo AI Lab \\
  Paris, France \\
  \texttt{m.vono@criteo.com} \\
  * Equal contribution\\
  \And
  Patrick Loiseau \\
  Inria, Fairplay joint team \\
  Palaiseau, France \\
  \texttt{patrick.loiseau@inria.fr} \\
  \And
  Vianney Perchet \\
  ENSAE, FairPlay joint team\\
  Palaiseau, France \\
  \texttt{vianney@ensae.fr}
}

\begin{document}

\maketitle

\begin{abstract}
We consider the \textit{dataset valuation problem}, that is, the problem of quantifying the incremental gain, to some relevant pre-defined utility of a machine learning task, of aggregating an individual dataset to others.
The Shapley value is a natural tool to perform dataset valuation due to its formal axiomatic justification, which can be combined with Monte Carlo integration to overcome the computational tractability challenges. Such generic approximation methods, however, remain expensive in some cases. In this paper, we exploit the knowledge about the structure of the dataset valuation problem to devise more efficient Shapley value estimators. We propose a novel approximation, referred to as discrete uniform Shapley, which is expressed as an expectation under a discrete uniform distribution with support of reasonable size. We justify the relevancy of the proposed framework via asymptotic and non-asymptotic theoretical guarantees and illustrate its benefits via an extensive set of numerical experiments.
\end{abstract}

\section{Introduction}\label{sec:introduction}

One of the main challenges for training machine learning (ML) models with enough generalization capabilities is to access a sufficiently large set of labeled training data. These data often exist but are commonly spread across many parties, impairing their usage in a direct and simple way. 
Real world examples range from the advertising industry, where different retailers hold sets of observations with either similar or complementary features from consented data about browsing and shopping habits of individual users; to the medical sector where hospitals may improve their diagnostics accuracy by sharing their data. By collaborating with each other and pooling their individual datasets together, these \textit{dataset owners} could learn better ML models for their applications. Naturally, many questions raise from such collaborations. Federated learning \citep{Donahue_Kleinberg_AAAI_2021,Donahue_Kleinberg_NeurIPS_2021}, for example, addresses the issues related to the practical ways that dataset owners can share their data. We consider a complementary problem to the one in federated learning: measuring the additional value each party would obtain by participating in the joint ML effort. In order to compute or estimate compensating rewards allowing to incentivize parties to share their data, a first stage that is commonly considered in the literature is to perform so-called \emph{dataset valuation} \cite{10.1145/3328526.3329589,10.5555/3524938.3525766,Tay2021IncentivizingCI}. 

Motivated by natural properties expected for fair valuation, different solution concepts from cooperative game theory \citep{2011Chalkiadakis} have been considered, the Shapley value \citep{P-295} being arguably the most broadly studied valuation scheme in ML due to its axiomatic justification. 
Agarwal et al. \cite{10.1145/3328526.3329589} designed a data marketplace and used the Shapley value to allocate the data among buyers. 
Tay et al. \cite{Tay2021IncentivizingCI} considered a cooperative environment where agents can jointly train a generative model, from which synthetic data are drawn and distributed to the parties according to their Shapley values.  
Sim et al. \cite{10.5555/3524938.3525766} rewarded parties based on the Shapley value and information gain on model parameters. 
The critical challenge when using the Shapley value is its well-known computational intractability. To cope with it,  \cite{10.1145/3328526.3329589,Tay2021IncentivizingCI} considered Monte Carlo (MC) approximations, while \cite{10.5555/3524938.3525766} worked with a small set of three players. This approximation methods, however, remain expensive whenever computing the marginal contributions involve retraining. Moreover, they are generic and do not use the specific structure of the dataset valuation problem at stake, leaving open the possibility to find more adapted approximations for that problem. 

The Shapley value was also used in the related problem of \textit{data valuation}. Data valuation measures the contribution of a single data point within a dataset in the training of a given prediction model. Several solution concepts based on the Shapley value have been proposed for the data valuation problem including \texttt{Data} \texttt{Shapley} \citep{pmlr-v97-ghorbani19c,pmlr-v89-jia19a}, \texttt{DShapley} \citep{pmlr-v119-ghorbani20a,pmlr-v130-kwon21a}, \texttt{Beta} \texttt{Shapley} \citep{pmlr-v151-kwon22a} or \texttt{CS-Shapley} \citep{schoch2022csshapley}, together with different MC variants to cope with the computational intractability issue. For the data valuation problem, the structure was exploited to give easier-to-compute solutions in certain cases, in particular for the $k$-nearest neighbor problem \cite{ghorbani2022data,10.14778/3342263.3342637,liang2020beyond,liang2021herald,pandl2021trustworthy,shim2021online,wang2023threshold}. Unlike data valuation, however, dataset valuation aims at quantifying the marginal contribution of a \textit{whole dataset} to a given ML task with respect to (w.r.t.) the datasets brought by other dataset owners. Although data and dataset valuation are related problems, they are different and the techniques developed for data valuation cannot be used for the dataset valuation problem that we study (we further develop this point in \Cref{sec:dataset_is_not_data_valuation}).


\textbf{Contributions.} 
We consider the dataset valuation problem.  Following the ML literature, we model it as a cooperative game whose value function relates to the considered ML task, and aim at estimating the Shapley value to measure the dataset owners contribution. We propose a new way to address the computational intractability issue of the Shapley value. Instead of relying on generic MC approximation schemes, our approximation method leverages the structure of the dataset valuation problem as well as a convergence result for a key random variable of the problem. Our approximation behaves well in many cases, both theoretically and empirically. More specifically, our main contributions can be summarized as follows:\vspace{-0.2cm}
\begin{enumerate}[itemsep = -0.5mm, leftmargin = *]
\item We propose \texttt{DU-Shapley} (\Cref{def:DU_shapley_use_cases}), a novel Shapley value approximation that exponentially reduces the number of utility function valuations required for the computation. This is the first dataset valuation approach leveraging the specific structure of the utility function.
    
\item Based on three different use-cases, we establish asymptotic and non-asymptotic theoretical guarantees for \texttt{DU-Shapley}, showing notably that it converges almost surely to the Shapley value as the number of dataset owners grows.

\item We assess the benefits of the proposed methodology using extensive numerical experiments on both Shapley value approximation and dataset valuation use-cases. We show, in particular, that \texttt{DU-Shapley} outperforms all considered MC approximations of the Shapley value. 
\end{enumerate}
\vspace{-0.2cm}

\textbf{Additional Related Work.}
Cooperative game theory has been applied to solve multi-agents ML problems beyond data and dataset valuation \cite{cong2020game,kang2019incentive,lyu2020collaborative,yu2020fairness}. In particular, the Shapley value has been used to solve several problems including variable selection \citep{10.5555/1642293.1642400}, feature importance \citep{10.5555/3495724.3497168,10.5555/3295222.3295230,lumdberg}, or model interpretation \citep{chen2018lshapley}. In these problems, similarly to the data and dataset valuation problems, the computational intractability issue of the Shapley value is usually addressed via MC \citep{CASTRO20091726,RM-2651,JMLR:v23:21-0439}.

\section{Problem Formulation and Main Concepts Involved}
\label{sec:preliminaries}

This section presents the dataset valuation problem we aim to solve, along with preliminaries
including the definition and classical approximations of the Shapley value. For $n \in \mathbb{N}$ and $A$, we denote $[n] := \{1,..,n\}$ and $\mathrm{U}(A)$ the uniform distribution with support on $A$.

\subsection{Generic Model}
\label{subsec:problem_formulation}

We consider a collaborative ML setting involving a set $\mathcal{I}$ of $I = |\mathcal{I}| \in \mathbb{N}^*$ dataset owners, also referred to as players in the sequel, who are willing to cooperate in order to solve a common ML problem. 
Each player $i \in \mathcal{I}$ is assumed to possess an individual dataset $\mathrm{D}_{i} = \{(x_i^{(j)},y_i^{(j)})\}_{j \in [n_i]}$ where $x_i^{(j)} \in \mathcal{X} \subset \mathbb{R}^d$ stands for a feature vector, $y_i^{(j)} \in \mathcal{Y}$ is a label, $n_i = |\mathrm{D}_{i}|$ refers to the number of data points in $\mathrm{D}_{i}$, and samples are drawn independently from a player-dependent distribution $p_i$, i.e., $(x_i^{(j)},y_i^{(j)}) \sim p_i$, for all $j \in [n_i]$ and  $i \in \mathcal{I}$. 


Our basic motivation is to quantify the incremental contribution that a given player $i \in \mathcal{I}$ brings by sharing her dataset $\mathrm{D}_i$ with other players towards solving some ML task. Hence, we are interested in scenarios in which, even though the data distribution might differ across players, they face a similar ML task, for instance the minimization of the expectation (with respect to $p_i$) of some loss function $\ell(\hat{Y},Y)$, where $\hat{Y}$ denotes a prediction of $Y$. In such cases, players can usually learn from others' datasets, in the sense that given some $X$,  the optimal prediction $\hat{Y}$ that minimizes $\mathbb{E}[\ell(\hat{Y},Y)|X]$ is the same for all player.
This holds, \emph{e.g.}, if the conditional distributions (or, in many cases, simply the conditional expectation) of $y^{(j)}$ given $x^{(j)}$ are the same but the marginal distributions of $x^{(j)}$ differ.



To model this problem with full generality, we assume that the players $i\in\mathcal{I}$ collaborate in solving an ML task whose success is measured through some abstract metric $u$ that maps any dataset to a real number (say, the prediction accuracy in a classification problem). 
With a slight abuse of notation, for any coalition of players  $\mathcal{S} \subseteq \mathcal{I}$, we define $u(\mathcal{S}) = u(\mathrm{D}_\mathcal{S})$, where $\mathrm{D}_\mathcal{S} := \cup_{i\in\mathcal{S}}\mathrm{D}_{i}$. Hence, $u: 2^{\mathcal{I}} \rightarrow \mathbb{R}$ can be seen as a game-theoretical utility function that quantifies how well coalitions of players can solve the considered ML task based on the union of their datasets.

The following subsections provide three theoretical use-cases that instantiate the generic model and give specific utility functions $u$ to illustrate the dataset valuation problem. Using different tools and techniques, \Cref{section:proposed-approach} provides theoretical guarantees in each of them. These theoretical results are then complemented in \Cref{sec:numerical_experiments} by
numerical evidence of our proposed approach in more intricate practical problems on real data.

\subsubsection{Theoretical use-case 1: Non-parametric Regression}\label{usecase1}

The first use case we shall investigate is quite generic and consists in non-parametric regression. We assume the existence of a function $f^*$ such that $y^{(j)}_i=f^*(x_i^{(j)})+\eta^{(j)}_i$ with $\eta^{(j)}_i$ i.i.d., and a quadratic loss function. Without regularity assumption on $f^*(\cdot)$, learning can be arbitrarily slow; hence it is usually assumed that this mapping is Lipschitz (or at least $\beta$-H\"older \citep{gyorfi2022rate,10.5555/1522486}). 

The standard estimation method of $f^*$ we shall consider is called the \textsl{regressogram} or \textsl{binning} (also applied in \citep{gyorfi2022rate} to study local differential privacy within regression) and consists in learning optimal piece-wise constant functions. More precisely, given some parameter $B \in \mathbb{N}$---chosen exogeneously as a function of the function regularity $\beta$, the ambient dimension $d$ and the total number $n$ of datapoints, typically $B\simeq n^{\nicefrac{d}{(d+2\beta)}}$---, the feature space $\mathcal{X}$ is partitioned into $B$ cubic bins. The excess risk of learning $f^*$ can then be decomposed into 
\begin{align}
\mathbb{E}\bigl[(\hat{f}(x) - f^*(x))^2\bigr] = \mathbb{E}\bigl[(\hat{f}(x) - \Bar{f}(x))^2\bigr] + \mathbb{E}\bigl[(\Bar{f}(x) - f^*(x))^2 \bigr], \eqsp\label{eq:estimation_and_approximation_error}
\end{align}
where $\hat{f}$ is the estimator of $f^*$, $\Bar{f}(x) := \sum_{b\in [B]} \Bar{f}_b\mathbbm{1}\{x\in b\}$, and $\Bar{f}_b$ is any value that $f^*$ can take on the bin $b$. The second term in \eqref{eq:estimation_and_approximation_error} being agnostic to the players' datasets, the problem of measuring the contributions of the players to estimating $f^*$ can be decomposed into measuring their contributions to estimating each $\Bar{f}_b$. In particular, the utility $u(\mathcal{S})$ of a coalition $\mathcal{S}$ can be defined, and split into the sum of $B$ sub-utilities $u_b(\mathcal{S})$ functions, as follows
\begin{align}
    u(\mathcal{S}) := - \mathbb{E}\bigl[(\hat{f}_{\mathcal{S}}(x) - \Bar{f}(x))^2 \bigr] 
    = \sum\nolimits_{b\in [B]} -\mathbb{E}\bigl[(\hat{f}_{\mathcal{S},b} - \Bar{f}_{b})^2 \bigr]\mathbb{P}(x \in b) =: \sum\nolimits_{b\in [B]} u_b(\mathcal{S})\mathbb{P}(x \in b),
\end{align}
where $\hat{f}_{\mathcal{S}}$ is the estimator of $\Bar{f}$ when using the datasets of all players in $\mathcal{S}$ and $\hat{f}_{\mathcal{S},b}$ is the estimator $\Bar{f}_b$ when using, for all players in $\mathcal{S}$, the datasets of points in the bin $b$. 
Interestingly, after this reduction, the problem is decomposed into $B$ independent sub-problems---one per bin---, where the utility is a sole function of the number of data points used to estimate $\Bar{f}_b$, i.e., we can write $u_b(\mathcal{S}) = w_b(\sum_{i \in \mathcal{S}}n_{i,b})$ for some function $w_b: \mathbb{N} \to \mathbb{R}$, where $n_{i,b}$ is the number of data points that player $i$ has in the bin $b$. This last property motivates our second theoretical use-case.

\subsubsection{Theoretical use-case 2: Homogeneous case}\label{usecase2}

The second theoretical setting considers a general learning problem (not necessarily restricted to regression) and supposes that all players have the same sampling distribution, i.e., it takes $p_i = p$ for all $i \in \mathcal{I}$. This homogeneity on the players allows to reduce the problem of measuring the contribution of the players to just counting the number of data points contributed by each of them. Formally, and similarly to the previous use-case, we suppose the existence of a function $w: \mathbb{N}\to\mathbb{R}$ such that $u(\mathcal{S}) = w(\sum_{i\in\mathcal{S}}n_i)$.

\subsubsection{Theoretical use-case 3: Heterogeneous Linear Regression - Local Differential Privacy}\label{sec:linear_regression_LDP}

The third theoretical setting we consider is linear regression with random design and different variance of the features and labels per player. Although the setting is more general, one of the motivations behind it is standard linear regression with homogeneous data between players, but where players can purposely add noise when sharing their dataset (in order to provide Local Differential Privacy, for instance). Formally, for any $i \in \mathcal{I}$, we consider the following  linear model that generates the dataset $\mathrm{D}_i$ of size $n_i$:
\begin{equation}
\label{eq:lin_reg1}
\begin{aligned}
    &y_i^{(j)} = x_i^{(j)}\theta + \eta_i^{(j)}\eqsp, \text{ where } \eta_i^{(j)} \sim \mathrm{N}(0,\varepsilon_i^2)\eqsp, \text{ and }x_i^{(j)} \sim \mathrm{N}(0_d,\sigma_i^2 \mathrm{I}_d)\eqsp, \text{ for any } j \in [n_i],\eqsp\label{eq:likelihood_linear1} 
\end{aligned}
\end{equation}
with $\theta \in \R^d$ a ground-truth parameter, $\sigma_i$ positive and known, and $\varepsilon_i$ the differential privacy level chosen by player $i$. 
Under the linear regression framework defined in \eqref{eq:lin_reg1}, and following \citep{Donahue_Kleinberg_AAAI_2021}, the utility function of a set $\mathcal{S} \subseteq \mathcal{I}$ of players is defined by the negative expected mean square error over a hold-out dataset, i.e., 
\begin{equation}
    \label{eq:utility_lin_reg}
    u(\mathcal{S}) = -\mathbb{E}\bigl[\bigl(x^\top \hat{\theta}_{\mathcal{S}} - x^\top \theta\bigr)^2\bigr]\eqsp,
\end{equation}
where the expectation is taken over the distribution $p_{\mathrm{test}}$ of a hold-out testing datum $x \in \R^d$, the sampling distributions $\mathrm{N}(0,\sigma_i^2\mathrm{I}_d)$ for all $i\in\mathcal{S}$, and the linear regression error distributions $\mathrm{N}(0,\varepsilon_i^2), \forall i\in\mathcal{S}, j\in[n_i]$, and $\hat{\theta}_{\mathcal{S}}$ stands for the generalized least square estimator defined by $\hat{\theta}_{\mathcal{S}} = (X_{\mathcal{S}}^\top \Sigma_{\mathcal{S}}^{-1} X_{\mathcal{S}})^{-1} X_{\mathcal{S}}^\top \Sigma_{\mathcal{S}}^{-1} Y_{\mathcal{S}}, \text{ where } \Sigma_\mathcal{S} = \mathrm{diag}((\varepsilon^2_i)_{i\in\mathcal{S}}) \in \mathbb{R}^{|\mathcal{S}|\times|\mathcal{S}|}.$ The notations $X_{\mathcal{S}}$ and $Y_{\mathcal{S}}$ refer to the concatenation of $\{X_i\}_{i \in \mathcal{S}}$ and $\{Y_i\}_{i \in \mathcal{S}}$, respectively, and $X_i \in \R^{n_i \times d}$ is defined by $X_i = ([x_i^{(1)}]^\top,\ldots,[x_{i}^{(n_i)}]^\top)^\top$ while $Y_i \in \R^{n_i}$ is defined by $Y_i = (y_i^{(1)},\ldots,y_i^{(n_i)})^\top$.

The following result provides a close-form expression for the utility function in this case:
\begin{proposition}\label{prop:close_form_utility_function} 
Let $\mathcal{S}$ be a coalition of players and consider the value function as above. It follows,
\begin{align}
u(\mathcal{S}) = \frac{-\mathrm{Tr}\bigl[\mathbb{E}\bigl[x x^\top\bigr]\bigr]}{q({\mathcal{S}}) - d - 1} , \text{ where } q(\mathcal{S}) := \left\lfloor\frac{\bigl(\sum_{i\in \mathcal{S}} (\nicefrac{\sigma_i}{\varepsilon_i}) n_i\bigr)^2}{\sum_{i\in \mathcal{S}} \bigl(\nicefrac{\sigma_i}{\varepsilon_i}\bigr)^2 n_i}\right\rfloor, \text{ with the convention } q(\varnothing) = 0.
\end{align}
In particular, considering $p_{\mathrm{test}} = \mathrm{N}(0,\mathrm{I}_d)$, we get $   u(\mathcal{S}) = \frac{d}{d + 1 - q({\mathcal{S}})}.$
\end{proposition}

\Cref{prop:close_form_utility_function} shows that, in this use-case, the utility function can be written as a function $w(q(\mathcal{S}))$ of a scalar quantity $q(\mathcal{S})$ that captures the datasets heterogeneity. 
Notice that in this use-case, if we add the homogeneity assumption that $\sigma_i/\varepsilon_i = \sigma/\varepsilon$, for all $i \in \mathcal{I}$, then the term $q(\mathcal{S})$ becomes $\sum_{i \in \mathcal{S}}n_i$ and, as a consequence, we get
\begin{align}\label{eq:linear_utility_homogeneous_case}
u(\mathcal{S}) = w(q(\mathcal{S})) = w\left(\sum\nolimits_{i\in \mathcal{S}} n_i\right) = \frac{d}{d+1 - \sum_{i\in \mathcal{S}} n_i}.
\end{align}
Recall that, in the non-parametric regression use-case, it holds $u(\mathcal{S})=\sum_{b \in [B]}\mathbb{P}(x\in b)w_b(q_b(\mathcal{S}))$ where $q_b(\mathcal{S})=\sum_{i \in \mathcal{S}}n_{i,b}$. Therefore, in our three uses-cases, the utility of a coalition can be summarized as the function of some scalar quantity of interest. This observation will be useful to state later our theoretical results.


\subsection{Shapley Value}\label{sec:Shapley_value}

The Shapley value \citep{P-295} is a classical solution concept in cooperative game theory to fairly allocate the total gains generated by a coalition of players. 
Given a utility function $u$, the Shapley value of a player $i$ is defined as the average marginal
contribution of her dataset $\mathrm{D}_i$ to all possible subsets of $\{\mathrm{D}_j\}_{j \in \mathcal{I}\setminus\{i\}}$, built by aggregating the datasets of the other players.
Formally, the Shapley value $\varphi_i$ of player $i$ writes
\begin{equation}
    \label{eq:Shapley_def1}
    \varphi_i(u) = \frac{1}{|\Pi(\mathcal{I})|} \sum\nolimits_{\pi \in \Pi(\mathcal{I})} [u(\mathcal{P}_i^{\pi} \cup \{i\}) - u(\mathcal{P}_i^{\pi})]\eqsp,
\end{equation}
where $\Pi(\mathcal{I})$ refers to the set of permutations over $\mathcal{I}$ and $\mathcal{P}_i^{\pi}$ to the set of predecessors of player $i \in \mathcal{I}$ in permutation $\pi \in \Pi(\mathcal{I})$. 
The Shapley value of player $i$ is equivalently expressed as
\begin{align}
     \label{eq:Shapley_def2}
    \varphi_i(u) = \frac{1}{I}\sum\nolimits_{\mathcal{S} \subseteq \mathcal{I} \setminus \{i\}} \binom{I-1}{|\mathcal{S}|}^{-1} [u(\mathcal{S} \cup \{i\}) - u(\mathcal{S})]\eqsp.
\end{align}
The Shapley value has been commonly used in ML and cooperative game theory as it uniquely satisfies the following set of desirable properties.
\begin{enumerate}[itemsep = 0mm,leftmargin = *]
    \item \emph{Efficiency.} $\sum_{i=1}^I \varphi_i(u) = u(\mathcal{I})$, i.e, the sum of all Shapley values is equal to the value of $\mathcal{I}$.
    \item \emph{Symmetry.} If, for any $\mathcal{S} \subseteq \mathcal{I} \setminus \{i_1,i_2\}$, $u(\mathcal{S} \cup \{i_1\}) = u(\mathcal{S} \cup \{i_2\})$, then $\varphi_{i_1}(u) = \varphi_{i_2}(u)$, i.e., whenever two players have the same marginal contributions, their Shapley values coincide.
    \item \emph{Dummy.}  If, for any $\mathcal{S} \subseteq \mathcal{I} \setminus \{i\}$, $u(\mathcal{S} \cup \{i\}) = u(\mathcal{S})$, then $\varphi_{i}(u) = 0$, i.e., whenever a player has null marginal contributions, her Shapley value is zero.
    \item \emph{Linearity.}  $\varphi_i(u_1 + u_2) = \varphi_i(u_1) + \varphi_i(u_2)$, i.e., the Shapley value of sums of games is the sum of the Shapley values of the respective games.
\end{enumerate}
\vspace{-0.2cm}

\textbf{MC approximation of the Shapley Value.} Evaluating the Shapley value is unfortunately computationally expensive in general.
As a consequence, many  MC approximations have been considered by sampling with replacement $T$ terms from the sum of either \eqref{eq:Shapley_def1} or \eqref{eq:Shapley_def2}.
Regarding \eqref{eq:Shapley_def1}, this boils down to considering the estimator 
\begin{equation}
\label{eq:MC}
    \hat{\varphi}_i(u) = \frac{1}{T} \sum\nolimits_{t=1}^T [u(\mathcal{P}_i^{\pi_t} \cup \{i\}) - u(\mathcal{P}_i^{\pi_t})]\eqsp, \text{where $\pi_t \sim \mathrm{U}(\Pi(\mathcal{I}))$.}
\end{equation}

\subsection{Data valuation vs Dataset valuation}\label{sec:dataset_is_not_data_valuation}

A tentative, but naive, approach to solve the dataset valuation problem could be to run an auxiliary data-valuation algorithm on all the data and to assign to each dataset the sum of the values of its datapoints. We highlight the cons of this idea on a very simple, yet insightful example. Consider two datapoints $x_1$ and $x_2$, three datasets $\mathrm{D}_1 = \{x_1\}$, $\mathrm{D}_2 = \{x_2\}$, $\mathrm{D}_3 = \{x_2,x_2\}$, and the following toy utility function $u(\mathrm{D}) = \mathbbm{1}{\{x_1,x_2 \in \mathrm{D}\}}$. In data valuation, any point $x_2$ shall have the same value, as they are identical. In particular, a naive summation would value $\mathrm{D}_3$ twice the value of $\mathrm{D}_2$. In dataset valuation, and for this toy problem at hand, it is quite clear that both datasets should have the same value. Moreover, the Shapley values are $1/6$ for $\mathrm{D}_2$ and $\mathrm{D}_3$ versus $2/3$ for $\mathrm{D}_1$.

The message here is twofold. Data valuation and dataset valuation are two fundamentally different concepts and one cannot directly reduce the latter to the former. This is actually true, and this is the second message, because the utility function $u$ is highly non-linear (even for the regression task).

\section{Discrete Uniform Shapley Value}
\label{section:proposed-approach}

This section introduces and studies our approximation scheme for the Shapley value. \Cref{sec:Insights_DU_Shapley}  shows an  asymptotic property that gives the general intuition behind our approximation. The result holds for the three use-cases of \Cref{usecase1,usecase2,sec:linear_regression_LDP}.
\Cref{sec:DU_Shapley_definition} presents a general approximation methodology for dataset valuation and shows its almost surely convergence as the number of players grows for our three uses-cases. \Cref{sec:non_asymptotic_guarantees} studies the rate of convergence, first for the homogeneous setting (\Cref{usecase2}), and then leverages this result to obtain a similar one for the non-parametric regression setting (\Cref{usecase1}). All proofs are postponed to the supplementary material. 


\subsection{Insights behind \texttt{DU-Shapley}}\label{sec:Insights_DU_Shapley}

The Shapley value, by re-arranging the coalitions $\mathcal{S}\subseteq \mathcal{I}\setminus\{i\}$ by their cardinality in the sum in \eqref{eq:Shapley_def2}, can be equivalently expressed as
\begin{align}\label{eq:Shapley_def4}
    \varphi_i(u) = \mathbb{E}_{K \sim \mathrm{U}(\{0,...,I-1\})}\mathbb{E}_{\mathcal{S} \sim \mathrm{U}\bigl(2^{\mathcal{I}\setminus\{i\}}_{K}\bigr)}\br{u({\mathcal{S}} \cup\{i\}) - u({\mathcal{S}})}\eqsp,
\end{align}
where $2^{\mathcal{I}\setminus\{i\}}_{K}$ denotes the subsets of $\mathcal{I}\setminus \{i\}$ of cardinality $K$. In our three uses-cases, it follows that 
\begin{align}
\varphi_i(u) &= \varphi_i(w) = \mathbb{E}_{K \sim \mathrm{U}(\{0,...,I-1\})}\mathbb{E}_{\mathcal{S} \sim \mathrm{U}\bigl(2^{\mathcal{I}\setminus\{i\}}_{K}\bigr)}\br{w(q(\mathcal{S} \cup\{i\})) -w(q(\mathcal{S}))}\eqsp,\label{eq:Shapley_def3}
\end{align}
where $w : \mathbb{R}_+ \to \mathbb{R}$ is such that $u(\mathcal{S}) = w(q(\mathcal{S}))$ for any $\mathcal{S}\subseteq \mathcal{I}$, and $q(\mathcal{S})$ is the scalar quantity of interest identified in \Cref{usecase1,usecase2,sec:linear_regression_LDP} for each use-case:
\begin{align}\label{eq:q(S)_definition}
q(\mathcal{S}) := \biggl\lfloor\frac{\bigl(\sum_{i\in \mathcal{S}} \gamma_i n_i\bigr)^2}{\sum_{i\in \mathcal{S}} \gamma_i^2 n_i}\biggr\rfloor, \text{ where, for any } i \in \mathcal{I}, \gamma_i = \left\{\begin{array}{cl}
1 & \text{for the second use-case},  \\
 \nicefrac{\sigma_i}{\varepsilon_i} &  \text{for the third use-case},
\end{array} \right.
\end{align}
and for the first use-case, $q_b(\mathcal{S})$ is analogously defined at every bin, with $\gamma_i^b = 1$ for all players and all bins. We remark that the definition of $q(\mathcal{S})$ in the first and second use-cases is not restricted to linear regression. Equation \eqref{eq:Shapley_def3} explicitly reveals a key random variable, namely $q(\mathcal{S})$. Interestingly, \Cref{fig:approx_uniform} suggests that $q(\mathcal{S})$ converges in distribution to a uniform random variable as the number of players increases (with i.i.d. datasets sizes). \Cref{theorem:convergence_uniform} proves this result formally for any $(\gamma_i)_{i\in\mathcal{I}}$.

\begin{theorem}\label{theorem:convergence_uniform}
Let $\{n_i,\gamma_i\}_{i\in [I]}$ be two sequences of positive numbers such that the following limits
\begin{align}
&\lim_{I\to \infty}\frac{1}{I}\sum\nolimits_{i\in [I]} n_i\gamma_i= \mu_A,
\quad \lim_{I\to \infty}\frac{1}{I}\sum\nolimits_{i\in [I]} (n_i\gamma_i - \mu_A)^2 = \sigma^2_A,\\
&\lim_{I\to \infty}\frac{1}{I}\sum\nolimits_{i\in [I]} n_i\gamma_i^2 = \mu_B,
\quad \lim_{I\to \infty}\frac{1}{I}\sum\nolimits_{i\in [I]} (n_i\gamma_i^2 - \mu_B)^2 = \sigma^2_B \eqsp,
\end{align}
all exist, for some constants $\mu_A,\mu_B, \sigma_A, \sigma_B > 0$. Let $K \sim \mathrm{U}(\{0,\ldots,I\})$, $\mathcal{S}_{K} \sim \mathrm{U}([2^{\mathcal{I}}_{K}])$. Then, almost surely, $\frac{q(\mathcal{S}_{K})}{q(\mathcal{I})} \xrightarrow{I \to \infty} \mathrm{U}([0,1])$.
\end{theorem}
\vspace{-0.4cm}
\begin{figure}[h]
\includegraphics[scale=0.33]{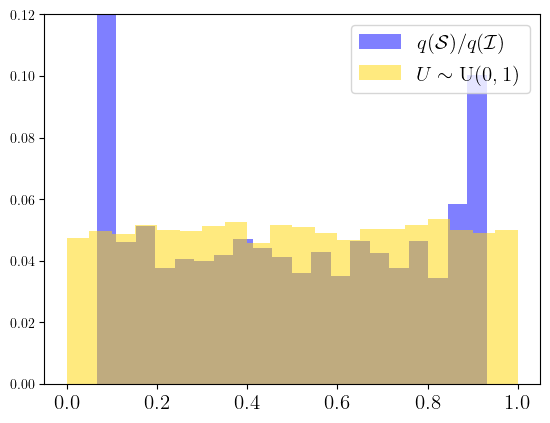}
\includegraphics[scale=0.33]{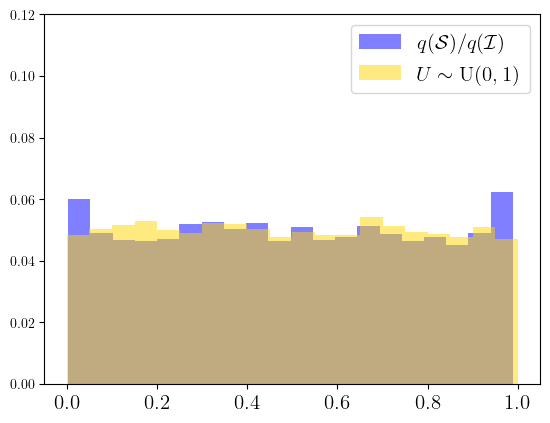}
\includegraphics[scale=0.33]{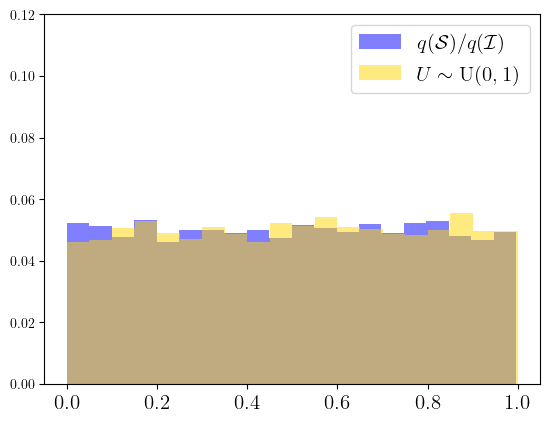}
\vspace{-0.5cm}
\caption{Distribution of $q(\mathcal{S})/q(\mathcal{I})$ when $\mathcal{S}$ is sampled as in \eqref{eq:Shapley_def3} (i.e., first sample a size $K$ uniformly, then sample a coalition $\mathcal{S}$ of size $K$ uniformly). (left) $I = 10$, (middle) $I=50$, (right) $I=500$. 
We considered $10^4$ samples for each random variable, and the third use-case with $n_i\sim\mathrm{U}([100])$ and $\sigma_i/\varepsilon_i \sim \mathrm{U}([10])$ for each $i \in \mathcal{I}$.}
\label{fig:approx_uniform}
\end{figure}
\vspace{-0.25cm}

\subsection{Discrete Uniform Shapley value}\label{sec:DU_Shapley_definition}

The Shapley value re-arrangement in \eqref{eq:Shapley_def4} exposes the main tool behind our approximation: it is enough to approximate the distribution of the random variable $\mathrm{D}_{\mathcal{S}}$ that takes values on the subsets of $\mathrm{D}_{-i} := \cup_{j\in \mathcal{I}\setminus\{i\}} \mathrm{D}_{j}$ (recall that $u(\mathcal{S}) = u(\mathrm{D}_{\mathcal{S}})$). \Cref{theorem:convergence_uniform}, taking the example of the second use-case for intuition, indicates that these datasets have uniformly distributed numbers of points in the limit. Generalizing this intuition, we propose to approximate $\mathrm{D}_{\mathcal{S}}$ by taking $I$ samples of increasing size from the pool $\mathrm{D}_{-i}$ by sampling data points uniformly. This leads to the following definition of \texttt{DU-Shapley} for our generic model:
\begin{definition}\label{def:DU_shapley_use_cases}[\texttt{DU-Shapley}]
For any $i \in \mathcal{I}$, the discrete uniform Shapley value (\texttt{DU-Shapley}) of the $i$-th player, denoted by $\psi_i$, is given by 
\begin{align}\label{eq:DU-shapley}
\psi_i(u) := \frac{1}{I} \sum\nolimits_{k=0}^{I-1} u(\mathrm{D}^{(k)} \cup \mathrm{D}_i) - u(\mathrm{D}^{(k)}),
\end{align}
where $\mathrm{D}^{(k)}$ is a set of data points uniformly sampled without replacement  from $\mathrm{D}_{-i}$ of size $k\mu_{-i}$, with $\mu_{-i} = \frac{1}{(I-1)}|\mathrm{D}_{-i}|$.
\end{definition}
\vspace{-0.25cm}

Compared to the Shapley value defined in \eqref{eq:Shapley_def2}, which involves $2^I$ terms to compute, note that \texttt{DU-Shapley} only involves $I$ terms and hence it presents an exponential reduction of the number of utility function evaluations.
Of course, these computational savings come at the cost of some bias.
The latter is precisely quantified in \Cref{sec:non_asymptotic_guarantees} for our first two use-cases.

By definition, \texttt{DU-Shapley} is a random variable which depends on the sampled data points. However, whenever $u(\mathcal{S}) = w(q(\mathcal{S}))$, with $q(\mathcal{S})$ some scalar quantify of interest, as in our use-cases, we can get rid of the stochastic nature of \texttt{DU-Shapley} by considering $I$ real values from well chosen intervals. In particular, in our uses-cases, \texttt{DU-Shapley} boils down to:
\begin{align}\label{eq:DU_Shapley_use_cases}
&\psi_i(w) = \frac{1}{I} \sum\nolimits_{k=0}^{I-1} w(\Bar{q}_i^k) - w( \Bar{q}_{-i}^k),
\end{align}
where $$\Bar{q}_i^k := \biggl\lfloor\frac{(\gamma_i n_i + \frac{k}{I-1}\sum_{j\in \mathcal{I}\setminus\{i\}} \gamma_j n_j)^2}{\gamma_i^2 n_i + \frac{k}{I-1}\sum_{j\in \mathcal{I}\setminus\{i\}} \gamma_j^2 n_j}\biggr\rfloor \text{ and } \Bar{q}_{-i}^k := \biggl\lfloor\frac{k}{I-1}\cdot\frac{(\sum_{j\in \mathcal{I}\setminus\{i\}} \gamma_j n_j)^2}{\sum_{j\in \mathcal{I}\setminus\{i\}} \gamma_j^2 n_j}\biggr\rfloor.$$ 
We remark the notation abuse as we should write $\psi(w\circ q)$. For simplicity, we omit the composition and only write $\psi(w)$. \Cref{eq:DU_Shapley_use_cases} coincides exactly with \Cref{def:DU_shapley_use_cases} in the first two use-cases, i.e., when $\gamma_j = \gamma$ for all $j \in \mathcal{I}$. Indeed, as the random datasets $\mathrm{D}^{(k)}$ have a fixed size and the value function only looks at the number of data points within the coalition, we obtain, 
\begin{align}
    \psi_i(u) &= \frac{1}{I}\sum_{k=0}^{I-1} u(\mathrm{D}^{(k)} \cup \mathrm{D}_i) -  u(\mathrm{D}^{(k)}) = \frac{1}{I}\sum_{k=0}^{I-1} w(|\mathrm{D}^{(k)} \cup \mathrm{D}_i|) -  w(|\mathrm{D}^{(k)}|)\\
    &= \frac{1}{I}\sum_{k=0}^{I-1} w(k\mu_{-i} + n_i) -  w(k\mu_{-i}) = \psi_i(w).
\end{align}
For the third use-case,~\Cref{eq:DU_Shapley_use_cases} is an approximation that comes from  assuming that, for any $j \in \mathcal{I}\setminus\{i\}$, $|\mathrm{D}_j \cap \mathrm{D}^{(k)}| = k \cdot \frac{n_j}{I-1}$, which holds with high probability for large values of $I$, since
\begin{align}
    q(\mathrm{D}^{(k)} \cup \mathrm{D}_i) = \biggl\lfloor\frac{\bigl(\gamma_i n_i + \sum_{j \in \mathcal{I}\setminus\{i\}} \gamma_j \cdot |\mathrm{D}_j\cap \mathrm{D}^{(k)}|\bigr)^2}{\gamma_i^2 n_i + \sum_{j \in \mathcal{I}\setminus\{i\}} \gamma_j^2 \cdot |\mathrm{D}_j\cap \mathrm{D}^{(k)}|}\biggr\rfloor.
\end{align}
\Cref{theorem:convergence_uniform} implies the following result.

\begin{corollary}\label{cor:asymptotic_result_uses_cases}
Let $\varphi_i$ and $\psi_i$ be, respectively, the Shapley value \eqref{eq:Shapley_def2} and the \texttt{DU-Shapley} \eqref{eq:DU_Shapley_use_cases} of player $i$. Then, in our three uses-cases, it holds, $\lim_{I\to\infty} |\varphi_i - \psi_i| = 0$ almost surely.
\end{corollary}

While our theoretical results are based on Equation \eqref{eq:DU_Shapley_use_cases} for the cases where $u(\mathcal{S}) = w(q(\mathcal{S}))$, we will see through numerical experiments that \Cref{def:DU_shapley_use_cases} gives good results in more general cases.

\subsection{Non-Asymptotic Theoretical Guarantees}\label{sec:non_asymptotic_guarantees}

\Cref{cor:asymptotic_result_uses_cases} states asymptotic guarantees for \texttt{DU-Shapley}. In this section, we show non-asymptotic results that give the convergence rate for the first two uses-cases.\footnote{A similar result, albeit more technical, can be shown with the same arguments for the third use-case.} Recall that in non-parametric estimation, the utility writes as $u(\mathcal{S}) = \sum\nolimits_{b\in [B]} u_b(\mathcal{S})\mathbb{P}(x \in b)$,
and therefore, by the linearity axiom of the Shapley value, for any $i \in \mathcal{I}, \varphi_i(u) = \sum\nolimits_{b\in[B]} \varphi_i(u_b)\mathbb{P}(x\in b)$. As a consequence, in order to estimate $\varphi_i(u)$, it is enough to compute each $\varphi_i(u_b)$. In particular, the Shapley value approximation error over the whole feature space becomes a simple aggregation of the Shapley value approximation errors over the bins. We focus firstly on bounding the bias of our method in the homogeneous use-case to then extend it to the non-parametric regression case.

As in the homogeneous use-case the utility function writes as $u(\mathcal{S})= w(\sum_{i \in \mathcal{I}}n_i)$, we consider the following regularity assumptions on $w$. 

\begin{assumption}\label{ass:homogeneity} 
The function $w: \mathbb{R}_+ \rightarrow \R$ is increasing, twice continuously differentiable, and such that $\lim_{n \to \infty} n^2|w^{(2)}(n)| \sm \infty$ (where $w^{(2)}$ represents the second derivative).
\end{assumption}
\vspace{-0.2cm}

Monotonicity is a natural assumption in our framework as, the more data, the more precise the ML prediction is expected to be. The condition over the limit aims at controlling the growth behavior of the utility function and it is automatically satisfied whenever $w$ is bounded and $w^{(2)}$ is monotone, by the mean value theorem. \Cref{theorem:DU_shapley_error_bound} bounds the bias of \texttt{DU-Shapley} for the homogeneous use-case.

\begin{theorem}\label{theorem:DU_shapley_error_bound}
Under Assumption \textbf{H}\ref{ass:homogeneity}, there exists a constant $\kappa > 0$, such that, for any $i \in \mathcal{I}$, it holds,
\begin{align}
    \bigl|\varphi_i - \psi_i \bigr| \leq \frac{\kappa}{(I-1) \mu_{-i}^2} \left(\sigma_{-i}^2 (1+\ln(I-1)) + \zeta_{-i}\right), \label{eq:bound_error_1}
\end{align}
where $\varphi_i$ and  $\psi_i$ are respectively the Shapley value and the \texttt{DU-Shapley} of player $i$, $\mu_{-i} = \frac{1}{(I-1)}|\mathrm{D}_{-i}|$ is the average dataset size of all players but $i$,  $\sigma^2_{-i} = \frac{1}{I-1}\sum_{j\in \mathcal{I}\setminus\{i\}} (n_{j}-\mu_{-i})^2$ their empirical variance, and $\zeta_{-i}$ measures the variability of the dataset sizes across players. Formally, it is defined as $\zeta_{-i}:={R_{-i}^2\tau_{-i}^2}/{4{n}^{\mathrm{max}}_{-i}}$ 
where $R_{-i} := \max_{j \in \mathcal{I}\setminus\{i\}} |n_{j} - \mu_{-i}|$, ${n}^{\max}_{-i} := \max_{j\in \mathcal{I}\setminus\{i\}} n_{j}$, and $\tau_{-i}:= {{n}^{\max}_{-i}}/{\min_{j \in \mathcal{I} \setminus\{i\}} n_{j}}$.
\end{theorem}

The full proof of \Cref{theorem:DU_shapley_error_bound} is included in \cref{sec:proof_thm_bias_DU_Shapley} and it relies on controlling the absolute value of $\mathbb{E}[w(\mu_{-i}K) - w(\sum_{j\in\mathcal{S}}n_j)]$, where $K \sim \mathrm{U}(\{0,...,I-1\})$ and $\mathcal{S}$ is the random variable in \eqref{eq:Shapley_def4}. Using a second order Taylor expansion, the problem is reduced to controlling the term related to the second derivative of $w$ by using the regularity assumptions in \textbf{H}\ref{ass:homogeneity}. 


As advertised before, \Cref{theorem:DU_shapley_error_bound} can be directly generalized to the non-parametric use-case, since,
\begin{align}
u(\mathcal{S}) = \sum\nolimits_{b\in [B]} w_b\biggl(\sum\nolimits_{i \in \mathcal{S}}n_{i,b}\biggr)\mathbb{P}(x \in b), \ \text{ for } n_{i,b} = |\{(x,y) \in \mathrm{D}_j, x \in b\}|.
\end{align}

\begin{corollary}\label{cor:DU_shapley_error_bound_reg_non_parametric}
Under Assumption \textbf{H}\ref{ass:homogeneity} for all functions $w_b$, there exist constants $\kappa_b > 0$, such that, for any $i \in \mathcal{I}$, it holds that
\begin{align}
    \bigl|\varphi_i - \psi_i \bigr| \leq \sum\nolimits_{b\in[B]} \frac{\kappa_b\mathbb{P}(x \in b)}{(I-1) \mu_{-i,b}^2} \left(\sigma_{-i,b}^2 (1+\ln(I-1)) + 2\zeta_{-i,b}\right), \label{eq:bound_error}
\end{align}
where $\varphi_i$ and  $\psi_i$ are respectively the Shapley value and the \texttt{DU-Shapley} of player $i$, and all terms are equivalently defined to \Cref{theorem:DU_shapley_error_bound} at each bin $b \in [B]$.
\end{corollary}

The upper bound in \eqref{eq:bound_error} depends on natural quantities related to the dataset valuation problem described in \Cref{subsec:problem_formulation} at each bin, such as the first two moments $\mu_{-i,b}$ and $\sigma_{-i,b}$ of the datasets' size distribution.
More precisely, the error increases when there are some outlier players with a very small or large dataset size.
This behavior is expected since, in this particular setting, the random variable inside of the Shapley value differs from a uniform random variable.
As showcased in \Cref{theorem:convergence_uniform}, the error vanishes when the number of players $I$ tends towards infinity.

\section{Numerical Experiments}\label{sec:numerical_experiments}

We illustrate the benefits of \texttt{DU-Shapley} by measuring numerically three properties: (1) how well \texttt{DU-Shapley} approximates the Shapley value in real data, (2) how many (theoretical) iterations need other methods to achieve the same accuracy level than \texttt{DU-Shapley}, and (3) how well \texttt{DU-Shapley} performs in classical dataset valuation tasks with real data. \Cref{sec:appendix_DU_shapley_and_SVARM} complements the results by a complexity comparison between our method and \texttt{SVARM} \cite{kolpaczki2024approximating} and \Cref{subsec:synthetic} by experiments on synthetic data. The experiments strongly suggest that \texttt{DU-Shapley} performs well in all tasks.
\vspace{-0.2cm}

\subsection{Approximating the Shapley Value in Real-World Data}
\label{subsec:expe2}

We consider the real-world datasets in \citet{JMLR:v23:21-0439}, whose details are provided in Table \ref{table:dataset} in the appendix.
To tackle these problems we consider logistic regression models and gradient-boosted decision trees (GBDT). 
For classification tasks, the utility function has been taken as the expected accuracy of the trained logistic regression model over a hold-out testing set while for regression tasks, the utility function corresponds to the averaged MSE over a hold-out testing set. In both cases we took a hold-out testing set with 10\% of the size of the training dataset. 
For each dataset, we considered two worst-case scenarios for our method, namely $I = 10$ players and $I = 20$ players. 

Starting from the datasets in Table \ref{table:dataset}, we heterogeneously allocate datasets to the players. We compare ourselves with two approaches, referred to as \texttt{MC-Shapley}, for the standard MC approximation defined in \eqref{eq:MC}, and \texttt{MC-anti-Shapley} that considers, in addition, antithetic sampling \cite{JMLR:v23:21-0439}. We compute the averaged MSE across all players between the true Shapley value and each estimator. 

Since computing the marginal contributions in this experiment requires re-training, which is clearly not feasible for a large number of epochs, we chose to restrict ourselves to 20 steps of stochastic gradient descent for logistic regression and 20 boosting iterations for GBDTs. 
For MC-based approaches, we considered $I$ samples to compare those approximations with the proposed methodology on a fair basis, i.e., associated to the same computational budget. 

\Cref{table:expe2} depicts the results. We clearly see that, even in the worst-case scenario where the number of players is small and far from the theoretical assumptions from \Cref{sec:non_asymptotic_guarantees}, \texttt{DU-Shapley} competes favorably with the MC-based methods.

\begin{table}[ht]
\caption{Worst-case comparison between \texttt{DU-Shapley} and competitors, for real-world datasets considered in Table \ref{table:dataset}. We report the averaged MSE across all players w.r.t. the exact Shapley value.}
\centering
\small{\begin{tabular}{c|ll|ll|ll|ll} \toprule
Dataset         & \multicolumn{2}{c|}{adult} & \multicolumn{2}{c|}{breast-cancer}    & \multicolumn{2}{c|}{bank}  & \multicolumn{2}{c}{cal-housing}      \\ \midrule
Players    & \multicolumn{1}{c}{10} & \multicolumn{1}{c|}{20} & \multicolumn{1}{c}{10} & \multicolumn{1}{c|}{20} & \multicolumn{1}{c}{10} & \multicolumn{1}{c|}{20} & \multicolumn{1}{c}{10} & \multicolumn{1}{c}{20} \\ \midrule
DU-Shapley      & $\mathbf{2.10^{-3}}$ & $\mathbf{6.10^{-4}}$ & $\mathbf{3.10^{-3}}$ & $\mathbf{1.10^{-4}}$ & $\mathbf{5.10^{-2}}$ & $\mathbf{4.10^{-3}}$ & $\mathbf{1.10^{-2}}$ & $\mathbf{3.10^{-3}}$ \\
MC-Shapley      & $1.10^{-2}$ & $4.10^{-3}$ & $3.10^{-2}$ & $1.10^{-3}$ & $9.10^{-2}$ & $6.10^{-2}$ & $5.10^{-2}$ & $2.10^{-2}$ \\
MC-anti-Shapley & $8.10^{-3}$ & $2.10^{-3}$ & $1.10^{-2}$ & $8.10^{-4}$ & $8.10^{-2}$ & $4.10^{-2}$ & $3.10^{-2}$ & $1.10^{-2}$\\
\toprule
\end{tabular}

\begin{tabular}{c|ll|ll} \toprule
Dataset         & \multicolumn{2}{c|}{make-regression}  & \multicolumn{2}{c}{year}  \\ \midrule
Players    & \multicolumn{1}{c}{10} & \multicolumn{1}{c|}{20} & \multicolumn{1}{c}{10} & \multicolumn{1}{c}{20} \\ \midrule
DU-Shapley      & $\mathbf{9.10^{-2}}$ & $\mathbf{2.10^{-2}}$ & $\mathbf{1.10^{-3}}$ & $\mathbf{7.10^{-4}}$ \\
MC-Shapley      & $4.10^{-1}$ & $3.10^{-1}$ & $5.10^{-3}$ & $1.10^{-3}$ \\
MC-anti-Shapley & $4.10^{-1}$ & $2.10^{-1}$ & $5.10^{-3}$ & $1.10^{-3}$ \\
\toprule
\end{tabular}}
\label{table:expe2}
\end{table}\vspace{-0.3cm}

\subsection{Complexity of Computing the Shapley Values of all Players}\label{sec:complexity_numeral_results_paper}

We have looked at the number of iterations that \texttt{DataShapley} and the \textit{Improved Group Testing-Based} method \cite{wang2023threshold} (\texttt{IGTB}) require to achieve \texttt{DU-Shapley}'s accumulated bias, formally given by 
\begin{align}
    \mathrm{DU bias}(I) := \frac{\kappa}{I-1}\biggl(\sum\nolimits_{i\in \mathcal{I}} \frac{(9\sigma_{-i}^2(1+\log(I-1)) + \zeta_{-i})^2}{(\mu_{-i})^4}\biggr)^{1/2}.
\end{align}
To do so, we have replaced $\varepsilon = \mathrm{DU bias}(I)$, respectively, in the formula in Section 4.1 in \cite{pmlr-v89-jia19a} and Equation 5 in \cite{wang2023threshold}, with a value function motivated from our third use-case under the homogeneity assumption $\sigma_i/\varepsilon_i = \sigma/\varepsilon$ for all $i\in \mathcal{I}$. The results are illustrated in \Cref{fig:Jia_et_al_iterations}. Remark \texttt{DU-Shapley} requires $I^2$ iterations to compute all Shapley values. We observe that in all tested instances, both methods require a higher number of iterations to achieve the same error than DU-Shapley. 

\begin{figure}[ht]
    \centering
    \includegraphics[scale = 0.37]{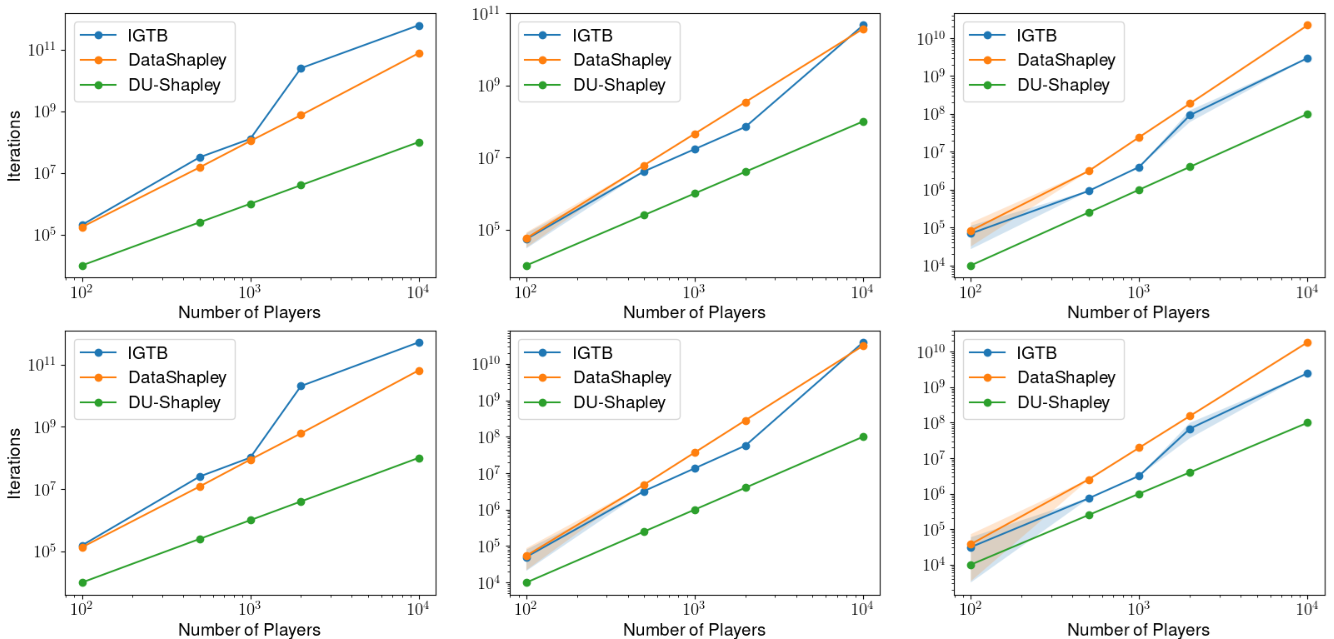}
    \caption{Iterations required by DataShapley and the Improved Group Testing-Based method to achieve DU-Shapley's accumulated bias with function $w(n_S) = 1 - \frac{10^{k(\mathcal{I})}}{10^{k(\mathcal{I})} + n_S}$, where $n_S$ is the number of data points of the coalition $S \subseteq \mathcal{I}$, and $k(\mathcal{I}) := \lfloor \log(\sum_{i\in\mathcal{I}}n_{i}) \rfloor - 1$ is a normalization factor. (top) $\delta = 0.01$, (bottom) $\delta = 0.1$, (left) $n_{\mathrm{max}} = 10$, (middle) $n_{\mathrm{max}} = 50$, (right) $n_{\mathrm{max}} = 100$.}
    \label{fig:Jia_et_al_iterations}
\end{figure}
\vspace{-0.1cm}

\subsection{Applying \texttt{DU-Shapley} to dataset valuation problems}\label{sec:numerical_experiments_dataset_valuation}

We considered non-tabular datasets used in \cite{jiang2023opendataval}, namely bbc-embedding, IMDB-embedding, both text datasets, and CIFAR10-embedding, an image dataset. Feature embedding have been generated using pretrained DistilBERT and ResNet50 models, respectively. In addition we have adapted three baselines from data valuation to our setting: Leave-One-Out (LOO), DataShapley, and KNN-Shapley. \Cref{sec:opendataval_implementations} gives the implementations details. For these datasets associated to classification problems, we used a multi-layer perceptron classifier as prediction model. 

We have considered three dataset valuation problems, none of them needing the real Shapley values, which allows us to increase the number of players w.r.t. the experiments in \Cref{subsec:expe2}. We investigated noisy label detection (NLD), dataset removal (DR), and dataset addition (DA) \citep{jiang2023opendataval}. For NLD, we used as a metric the F1-score (the larger the better). For DR, we used the testing accuracy (the lesser the better). For DA, we used the testing accuracy (the lesser the better). 

We considered splitting the dataset across $I = 100$ players. The results are summarized in \Cref{table:exp_dataset_use_cases}. We observe that \texttt{DU-Shapley} has competitive results compared to classical baselines despite of the fact that none of the considered cases verifies the structural assumptions from \Cref{sec:non_asymptotic_guarantees}. In addition, we can see that \texttt{DU-Shapley} tends to have similar and even better results as \texttt{Data} \texttt{Shapley} (which is a MC based method). This is in line with our theory as, for larger number of players, \texttt{DU-Shapley} tends to better estimate the true Shapley value.

\begin{table}[ht]
\caption{Comparison between \texttt{DU-Shapley} and competitors for real-world datasets considered in \cite{jiang2023opendataval} in Noisy label detection, Dataset Removal and Dataset Addition.}
\centering
\small{\begin{tabular}{l|llllll|llllll}\toprule
Dataset & \multicolumn{6}{c|}{CIFAR 10} & \multicolumn{6}{c}{BBC} \\ \midrule
\multirow{2}{*}{Problem} & \multicolumn{2}{c}{NLD} & \multicolumn{2}{c}{DR} & \multicolumn{2}{c|}{DA} & \multicolumn{2}{c}{NLD} & \multicolumn{2}{c}{DR} & \multicolumn{2}{c}{DA} \\
& 5\% & 15\% & 5\% & 15\% & 5\% & 15\% & 5\% & 15\% & 5\% & 15\% & 5\% & 15\%   \\ \midrule
\texttt{Random}  & 0.11    & 0.19    & 0.61    & 0.60   & 0.25    & 0.41   & 0.11    & 0.19    & 0.90    & 0.88   & 0.68    & 0.81   \\
\texttt{LOO}  & 0.13    & 0.18    & 0.62    & 0.60   & 0.15    & 0.32   & 0.11    & 0.17    & 0.90    & 0.88   & 0.61    & 0.77   \\
\texttt{DataShapley}    & 0.13    & 0.25    & 0.61    & 0.59   & 0.12    & 0.18   & 0.12    & 0.20    & 0.89    & 0.87   & 0.08    & 0.12   \\
\texttt{KNN-Shapley}    & \textbf{0.14}    & 0.28    & \textbf{0.60}    & 0.57   & 0.12    & 0.15   & \textbf{0.19}    & 0.29    & \textbf{0.88}    & 0.86   & 0.13    & 0.12   \\
\texttt{DU-Shapley}  & \textbf{0.14}    & \textbf{0.30}    & 0.61    & \textbf{0.55}   & \textbf{0.11}    & \textbf{0.14}   & 0.18    & \textbf{0.34}    & 0.89    & \textbf{0.85}   & \textbf{0.07}    & \textbf{0.11} \\
\toprule 
\end{tabular}
\smallskip 

\begin{tabular}{l|llllll} \toprule
Dataset & \multicolumn{6}{|c}{IMBD} \\  \midrule
\multirow{2}{*}{Problem} & \multicolumn{2}{c}{NLD} & \multicolumn{2}{c}{DR} & \multicolumn{2}{c}{DA} \\
 & 5\% & 15\%    & 5\% & 15\%   & 5\% & 15\%   \\ \midrule
\texttt{Random}  & 0.10    & 0.16    & 0.77    & 0.75   & 0.62    & 0.68   \\
\texttt{LOO}  & 0.11    & 0.18    & 0.77    & 0.74   & 0.53    & 0.59   \\
\texttt{DataShapley}    & 0.17    & 0.28    & \textbf{0.75}    & 0.69   & 0.36    & \textbf{0.33}   \\
\texttt{KNN-Shapley}    & \textbf{0.18 }   & 0.29    & 0.76    & 0.68   & 0.41    & 0.37   \\
\texttt{DU-Shapley}  & \textbf{0.18}    & \textbf{0.32}    & 0.76    & \textbf{0.66}   & \textbf{0.33}    & 0.34 \\
\toprule
\end{tabular}}
\label{table:exp_dataset_use_cases}
\end{table}
\vspace{-0.3cm}

\section{Conclusion}\label{sec:conclusion}

We model the dataset valuation problem as a cooperative game and design a Shapley value approximation, named \texttt{DU-Shapley}, that exploits the underlying structure of the utility function and exponentially reduces the number of functions valuations required for the computation. In three different uses-cases, \texttt{DU-Shapley} is proved to almost surely converge to the real Shapley value as the number of players grows. Moreover, we find the rate of convergence, which depends only on natural parameters of dataset valuation.
Numerical experiments showcase that \texttt{DU-Shapley} performs well in approximating the Shapley value and performing dataset valuation tasks, even when the assumptions needed for the theoretical guarantees do not hold, and it has a good complexity when computing the Shapley values of all players.

\textbf{Limitations of our method}. 
Our non-asymptotic bound for the non-parametric regression setting in \Cref{cor:DU_shapley_error_bound_reg_non_parametric} indicates that \texttt{DU-Shapley} works better when agents' datasets are \textit{regular} in the sense that they have similar sizes. Hence, a limitation of our approximation is that it may work poorly in settings where some players have large datasets compared to others, as the distribution of the random variable within the Shapley value drives apart from being uniform. Moreover, our convergence result in \Cref{theorem:convergence_uniform} (for all use-cases) assume the existence of limits, which roughly requires that heterogeneity between players---in terms of both dataset size and variance---can be bounded. This also indicates that convergence may be not be guaranteed if the heterogeneity is arbitrarily high.

\section*{Acknowledgments}
This research was supported in part by the French National Research Agency (ANR) in the framework of the PEPR IA FOUNDRY project (ANR-23-PEIA-0003) and through the grant DOOM ANR-23-CE23-0002. It was also funded by the European Union (ERC, Ocean, 101071601). Views and opinions expressed are however those of the author(s) only and do not necessarily reflect those of the European Union or the European Research Council Executive Agency. Neither the European Union nor the granting authority can be held responsible for them.

\bibliography{paper_icml/biblio}
\bibliographystyle{plainnat}


\appendix


\section{Complementary Numerical results}\label{sec:appendix_complementry_numerical_experiments}

All experiments were executed on a laptop running macOS 13.3.1 and equipped with Apple M1 chip with 16GB of RAM. The minimum amount of compute was roughly 5 minutes while the maximum one roughly 10 hours.

\subsection{DU-Shapley vs SVARM}\label{sec:appendix_DU_shapley_and_SVARM}

We have looked at the probability at which SVARM (Theorem 4 \cite{kolpaczki2024approximating}) can ensure, after $I^2$ iterations (without considering the warm up as part of the budget), an error equal to DU-Shapley's accumulated bias. We have considered the same value function than in \Cref{sec:complexity_numeral_results_paper} with $n_{max} \in \{2\cdot 10^3,3\cdot 10^3,5\cdot 10^3,10^4\}$ and 100 simulations of sets of players at each time. \Cref{fig:Proba_SVARM} shows the results. We observe how SVARM cannot ensure, with high enough probability, an approximation error equal to the one of \texttt{DU-Shapley}.

\begin{figure}[ht]
    \centering
    \includegraphics[scale = 0.3]{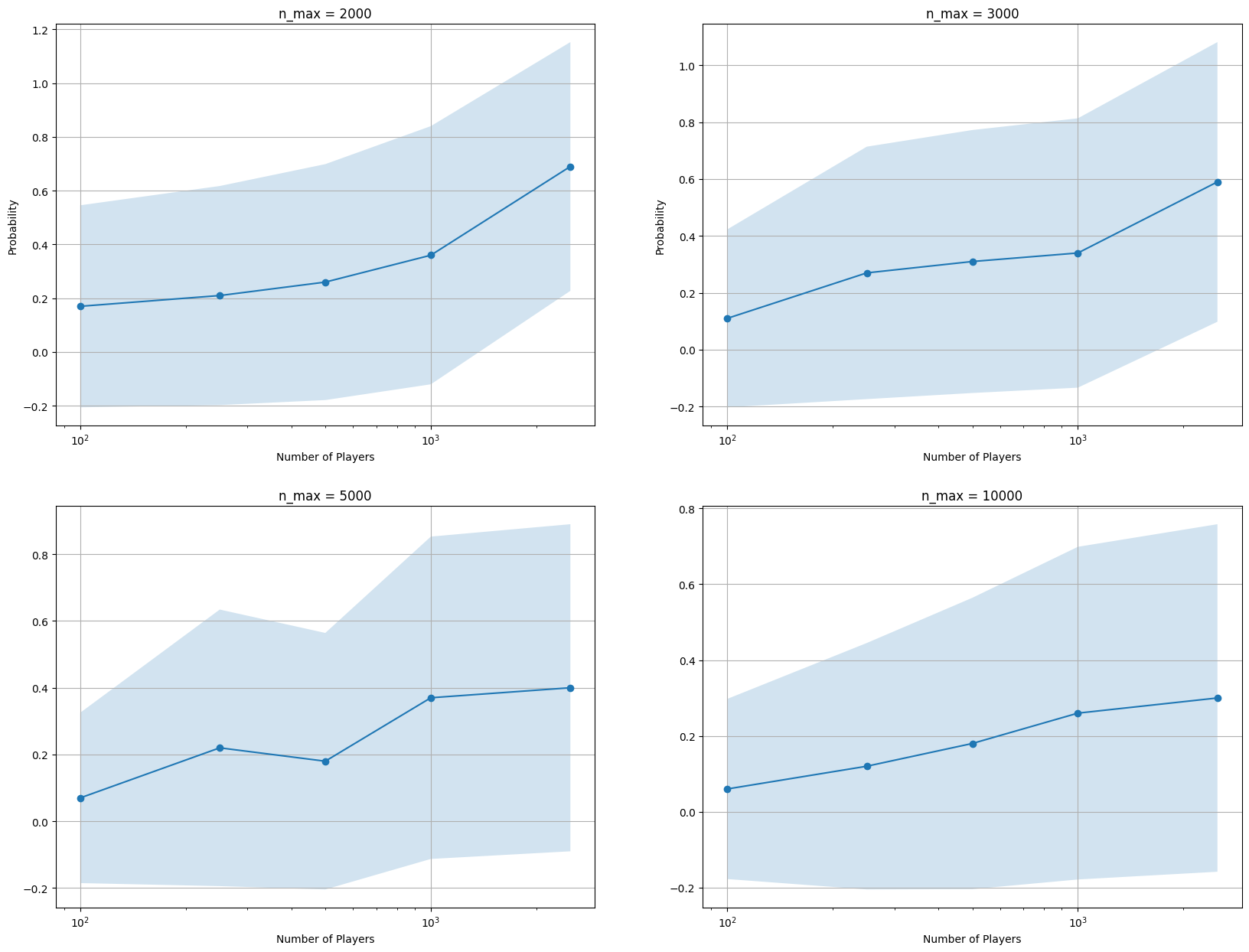}
    \caption{Probability that SVARM guarantees an error equal to DU-Shapley's bias}
    \label{fig:Proba_SVARM}
\end{figure}

\subsection{Approximating the Shapley value in Synthetic Data}\label{subsec:synthetic}

We consider a toy dataset valuation problem associated to our heterogeneous linear regression with local differential privacy use-case (\Cref{sec:linear_regression_LDP}) and we measure the value of a coalition $\mathcal{S}$ with the utility function in close-form from \Cref{prop:close_form_utility_function}. We consider $d = 10$. 

In order to benchmark the performances of \texttt{DU-Shapley}, we consider four competitive approaches, relying on Monte Carlo (MC) approximation strategies \citep{JMLR:v23:21-0439}. The first one, referred to as \texttt{MC-Shapley} is the standard MC approximation defined in \eqref{eq:MC}.
The second one, coined \texttt{MC-anti-Shapley} is a variance-reduced version of \texttt{MC-Shapley} that considers antithetic sampling.
The third one coined \texttt{Owen-Shapley} stands for the multilinear extension of \citet{Owen} which represents the Shapley value as two nested expectations (further explained in \Cref{sec:Owen_Shapley}). Finally, the fourth approach, coined \texttt{Orthogonal-Shapley}, relies on efficient permutation sampling techniques on the hypersphere to draw permutations in \eqref{eq:Shapley_def1} in a dependent way.
To assess the performance of the aforementioned Shapley value estimators, we used the mean square error (MSE) averaged over all players.
\texttt{DU-Shapley} is computed exactly by using \eqref{eq:DU_Shapley_use_cases} while, for each MC-based estimator, we performed 25 Shapley value estimations to compute the MSE, and did it 10 times to obtain confidence intervals for the MSE.

\Cref{fig:toy_example} compares \texttt{DU-Shapley} (the horizontal line which does not depend on the sampling budget as we compute it exactly) and the MC-based methods, which are computed at several different budgets. The x-axis corresponds to the sampling budget allowed to the MC-bases methods w.r.t. \texttt{DU-Shapley}, i.e., $10^{-1}$ means a budget equal to $10$\% the one of DU-Shapley, $10^0$ means same budget (indicated by the black vertical line), and $10^1$ means 10 times the \texttt{DU-Shapley} budget. Remark that, even when the MC-methods use 10 times the budget of \texttt{DU-Shapley}, our method keeps approximating better the Shapley value. 

\begin{figure}[ht]
    \centering
    \includegraphics[scale = 0.4]{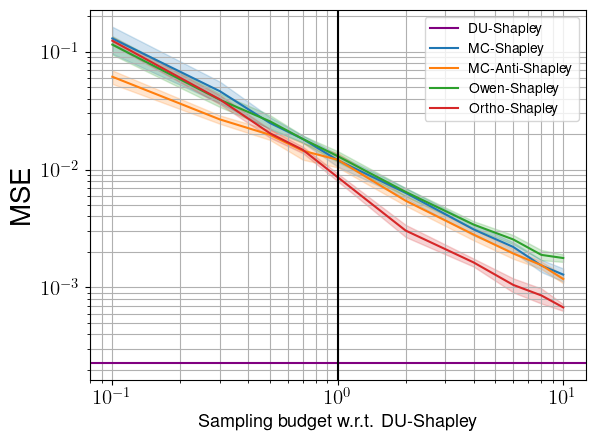}
    \includegraphics[scale = 0.4]{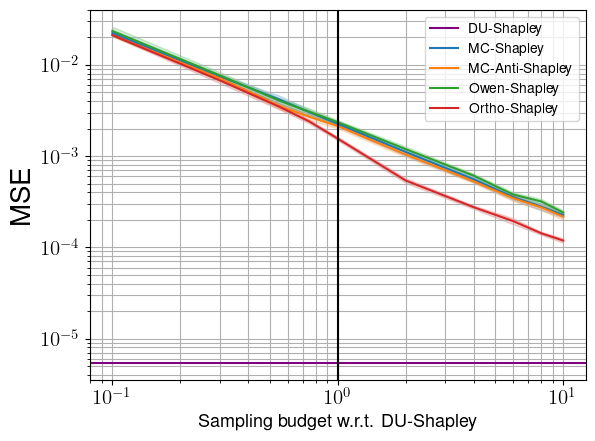}
    \includegraphics[scale = 0.4]{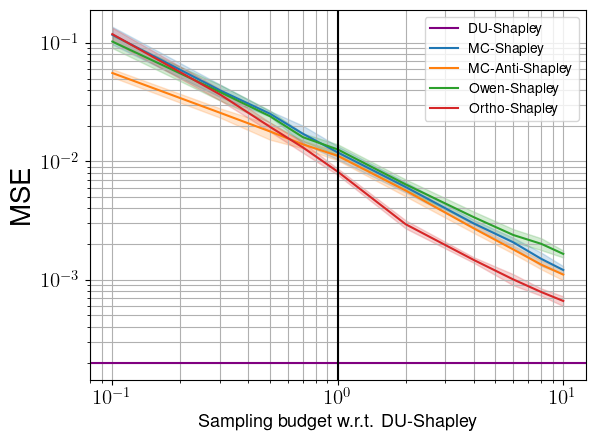}
    \includegraphics[scale = 0.4]{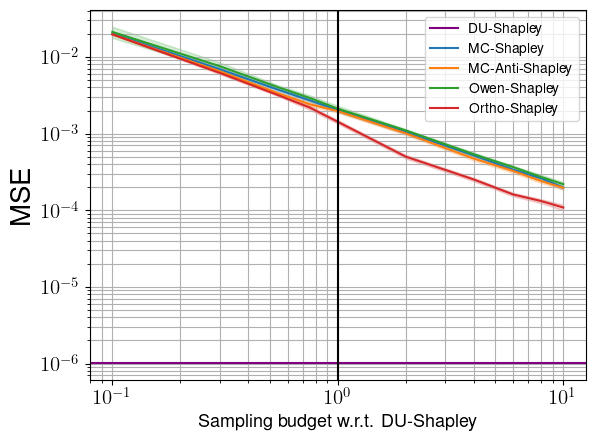}
    \caption{Worst-case comparison between \texttt{DU-Shapley} and MC-based approximations with different budgets on synthetic datasets. From left to right, $I = 10$ and $I = 20$, $n_i \sim \mathrm{U}([10^3]), \forall i \in \mathcal{I}$. (top) Scenario with small heterogeneity, $\nicefrac{\sigma_i}{\varepsilon_i} \sim \mathrm{U}([0,10]), \forall i \in \mathcal{I}$, (bottom) scenario with high heterogeneity, $\nicefrac{\sigma_i}{\varepsilon_i} \sim \mathrm{U}([0,100]), \forall i \in \mathcal{I}$.}
    \label{fig:toy_example}
\end{figure}




\section{Further details about numerical implementations}\label{sec:appendix_numerical_results}

\subsection{Datasets considered in \Cref{subsec:expe2}.}

\Cref{table:dataset} summarizes the real-world datasets considered in \Cref{subsec:expe2}.

\begin{table*}[h!]
\caption{Datasets considered in \Cref{subsec:expe2}.}
\centering%
{\small
\begin{tabular}{lccc}%
\toprule
Dataset & Size & $d$ & Task\\
\midrule
adult \citep{10.5555/3001460.3001502} & 48,842 & 107 & classification\\
breast-cancer \citep{10.1287/opre.43.4.570} & 699 & 30 & classification \\
bank \citep{MORO201422} & 45,211 & 16 & classification \\
cal-housing \citep{KELLEYPACE1997291} & 20,640 & 8 & regression \\
make-regression \citep{JMLR:v12:pedregosa11a}& 1,000 & 10 & regression  \\
year \citep{JMLR:v12:pedregosa11a} & 515,345 & 90 & regression  \\
\midrule
\end{tabular}
}
\label{table:dataset}
\end{table*}

\subsection{OpenDataVal implementations}\label{sec:opendataval_implementations}

In this section we describe more in detail the implementations of DataShapley, Leave-One-Out (LOO), and KNN-Shapley for our numerical results in \Cref{sec:numerical_experiments_dataset_valuation}. 

DataShapley, applied to the dataset valuation problem, simply corresponds to the method coined MC in \Cref{subsec:expe2}. Therefore, we sample datasets and output the averaged marginal contribution.

Regarding LOO, notice that it corresponds to compute just one marginal contribution, usually computed on the big coalition, i.e.,
\begin{align}
    \mathrm{LOO}_i := u(\mathcal{I}) - u(\mathcal{I}\setminus\{i\}).
\end{align}
As players' marginal contributions to large datasets tend to be small, we have preferred to sample one dataset $\mathrm{D}$ from $\mathrm{D}_{-i}$ and to output
\begin{align}
    \mathrm{LOO}_i := u(\mathrm{D} \cup \mathrm{D}_i) - u(\mathrm{D}).
\end{align}
Finally, regarding KNN-Shapley, we refer the reader to \cite{10.14778/3342263.3342637}, Section E.3 of the appendix who explain how to adapt the method to dataset valuation.

\subsection{Owen's Shapley value approximation}\label{sec:Owen_Shapley}

In Section \ref{subsec:synthetic}, we considered the Shapley value approximation referred to as \texttt{Owen-Shapley} as a state-of-the-art competitor to \texttt{DU-Shapley}. 
We provide in the following additional details regarding \texttt{Owen-Shapley}.
For the other competitors, we directly refer the interested reader to \citet{JMLR:v23:21-0439}. Owen \citep{Owen} studied the multilinear extension of a cooperative game and an alternative way to express the Shapley value. Formally, a cooperative game $G = (\mathcal{I},u)$ consists on a set of $I$ players $\mathcal{I} = \{1,2,...,I\}$ and a value function $u: 2^{\mathcal{I}} \to \mathbb{R}$ such that, for any $S \subseteq \mathcal{I}$, $u(S)$ corresponds to the value generated by the coalition $S$. The multilinear extension of $G$, denoted $\bar{G} = (\mathcal{I},\bar{u})$, is obtained when considering the value function $\bar{u} : [0,1]^{\mathcal{I}} \to \mathbb{R}$ given by,
\begin{align}
    \bar{u}(x_1,x_2,...,x_I) = \sum_{S \subseteq \mathcal{I}} \prod_{i \in S} x_i \prod_{j\notin S} (1-x_i) u(S).
\end{align}
Intuitively, $\bar{u}(x_1,x_2,...,x_I)$ corresponds to the expected value of a coalition when each player $i \in \mathcal{I}$ joins the coalition with probability $x_i$. Theorem 5 in \cite{Owen} gives an alternative way to compute the Shapley value $\varphi_i(u)$ of player $i$ in game $G$, namely,
\begin{align}
    \varphi_i(u) &= \int_{0}^1 \frac{\partial \bar{u}}{\partial x_i}(\tau,...,\tau) \mathrm{d}\tau = \int_{0}^1 \sum_{S \subseteq \mathcal{I}\setminus\{i\}} \tau^{|S|}(1-\tau)^{I-|S|-1}[u(S\cup\{i\}) - u(S)] \mathrm{d}\tau \\
    &=\int_0^1 \mathbb{E}\bigl[u(\mathcal{E}_i(\tau) \cup i) - u(\mathcal{E}_i(\tau))\bigr]\mathrm{d}\tau = \mathbb{E}_{\tau\sim \mathrm{U}([0,1])} \biggl[ \mathbb{E}\bigl[u(\mathcal{E}_i(\tau) \cup i) - u(\mathcal{E}_i(\tau))\bigr]\biggr],
\end{align}
where $\mathcal{E}_i(\tau)$ is a random subset of $\mathcal{I}\setminus\{i\}$, such that, $\forall j \in \mathcal{I}\setminus\{i\}$, $j$ is included in $\mathcal{E}_i(\tau)$ with probability $\tau$. In words, the Shapley value of player $i$ corresponds to her expected marginal contribution to the random set $\mathcal{E}_i(\tau)$, when $\tau$ is uniformly distributed on $[0,1]$. This brings an alternative way to use Monte Carlo to approximate the Shapley value $\varphi_i(u)$, coined Owen-Shapley, as,
\begin{align}
    \hat{\varphi}_i^{\text{Owen}}(u) = \frac{1}{T}\sum_{t=1}^T u(\mathcal{E}_i(\tau_t) \cup i) - u(\mathcal{E}_i^t(\tau_t)),
\end{align}
where for each $t \in \{1,...,T\}$, we draw $\tau_t$ independently and uniformly in   $[0,1]$ and then, create a random set $\mathcal{E}_i(\tau_t)$ by adding each player $j \in \mathcal{I} \setminus\{i\}$ to it with probability $\tau_t$.

\section{Missing proofs}

\subsection{Proof of \Cref{prop:close_form_utility_function}}

\textbf{\Cref{prop:close_form_utility_function}}. \textit{Let $\mathcal{S}\subseteq \mathcal{I}$ be a coalition of players and consider the value function $u$ as in \eqref{eq:utility_lin_reg}. It follows,
\begin{align}
u(\mathcal{S}) = \frac{-\mathrm{Tr}\bigl[\mathbb{E}\bigl[x x^\top\bigr]\bigr]}{q({\mathcal{S}}) - d - 1} , \text{ where } q(\mathcal{S}) := \left\lfloor\frac{\bigl(\sum\limits_{i\in \mathcal{S}} \frac{\sigma_i}{\varepsilon_i} n_i\bigr)^2}{\sum\limits_{i\in \mathcal{S}} \bigl(\frac{\sigma_i}{\varepsilon_i}\bigr)^2 n_i}\right\rfloor, \text{with the convention } q(\varnothing) = 0.
\end{align}
In particular, considering $p_{\mathrm{test}} = \mathrm{N}(0,\mathrm{I}_d)$, we get,
\begin{align}
    u(\mathcal{S}) = \frac{d}{d + 1 - q({\mathcal{S}})}.
\end{align} 
}

\begin{proof}
Let $\mathcal{S} \subseteq \mathcal{I}$ be a coalition of players and $X_{\mathcal{S}}, Y_{\mathcal{S}}$ be the concatenation of their datasets. The linear model can be rewritten in matrix form as
\begin{align}
    Y_\mathcal{S} = X_\mathcal{S}\theta + \eta_\mathcal{S},
\end{align}
where $\eta_\mathcal{S}$ is the concatenation of $\eta_i^{(j)}$ for all $i\in \mathcal{S}$ and $j \in [n_i]$. Take $\hat{\theta}_{\mathcal{S}} = (X_{\mathcal{S}}^\top \Sigma_{\mathcal{S}}^{-1} X_{\mathcal{S}})^{-1} X_{\mathcal{S}}^\top \Sigma_{\mathcal{S}}^{-1} Y_{\mathcal{S}}$ where $\Sigma_\mathcal{S} = \mathrm{diag}((\varepsilon^2_i)_{i\in\mathcal{S}})$, and let  $x \sim p_{\mathrm{test}}$ be a hold-out testing datum in $\mathbb{R}^d$. It follows,
\begin{align}
\bigl(x^\top&(\theta-\hat{\theta}_\mathcal{S})\bigr)^2 
= \biggl(\sum_{i\in\mathcal{S}} \eta_i \varepsilon_{i}^{-2}X_i\biggr) \biggl(\sum_{i\in \mathcal{S}} X_i^\top\varepsilon_{i}^{-2}X_i\biggr)^{-1} x x^\top \biggl(\sum_{i\in \mathcal{S}} X_i^\top\varepsilon_{i}^{-2}X_i\biggr)^{-1}\biggl(\sum_{i\in\mathcal{S}} \eta_i \varepsilon_{i}^{-2}X_i\biggr)\\
&= \mathrm{Tr}\biggl[\biggl(\sum_{i\in\mathcal{S}} \eta_i\varepsilon_{i}^{-2} X_i\biggr) \biggl(\sum_{i\in \mathcal{S}} X_i^\top\varepsilon_{i}^{-2} X_i\biggr)^{-1} x x^\top \biggl(\sum_{i\in \mathcal{S}} X_i^\top\varepsilon_{i}^{-2} X_i\biggr)^{-1}\biggl(\sum_{i\in\mathcal{S}} \eta_i\varepsilon_{i}^{-2} X_i\biggr) \biggr]\\
&= \mathrm{Tr}\biggl[x x^\top \biggl(\sum_{i\in \mathcal{S}} X_i^\top \varepsilon_{i}^{-2}X_i\biggr)^{-1} \biggl(\sum_{i\in\mathcal{S}} \eta_i\varepsilon_{i}^{-2} X_i\biggr)\biggl(\sum_{i\in\mathcal{S}} \eta_i \varepsilon_{i}^{-2}X_i\biggr)\biggl(\sum_{i\in \mathcal{S}} X_i^\top\varepsilon_{i}^{-2} X_i\biggr)^{-1}\biggr]\\
&= \mathrm{Tr}\biggl[x x^\top \biggl(\sum_{i\in \mathcal{S}} X_i^\top\varepsilon_{i}^{-2} X_i\biggr)^{-1} \biggl(\sum_{i\in\mathcal{S}}\sum_{j\in\mathcal{S}} X_i^\top\varepsilon_{i}^{-2} \eta_i\eta_j^\top\varepsilon_{j}^{-2} X_j\biggr)\biggl(\sum_{i\in \mathcal{S}} X_i^\top \varepsilon_{i}^{-2}X_i\biggr)^{-1}\biggr]
\end{align}
We take expectation with respect to the different stochastic terms. Since $\eta_i^{(k)} \sim \mathrm{N}(0,\varepsilon_i^2)$ for any $i \in \mathcal{S}, k \in [n_i]$, it holds,
\begin{align}
\mathbb{E}_{(\eta_i^{(k)}\sim \mathrm{N}(0,\varepsilon_i^2))_{i\in\mathcal{S}}^{k\in[n_i]}}
&\bigl[\bigl(x^\top(\theta-\hat{\theta}_\mathcal{S})\bigr)^2\bigr] \\
&= \mathrm{Tr}\biggl[x x^\top \biggl(\sum_{i\in \mathcal{S}} X_i^\top\varepsilon_{i}^{-2} X_i\biggr)^{-1} \biggl(\sum_{i\in\mathcal{S}} X_i^\top\varepsilon_{i}^{-2} X_i\biggr)\biggl(\sum_{i\in \mathcal{S}} X_i^\top\varepsilon_{i}^{-2} X_i\biggr)^{-1}\biggr]\\
&= \mathrm{Tr}\biggl[x x^\top \biggl(\sum_{i\in \mathcal{S}} X_i^\top\varepsilon_{i}^{-2} X_i\biggr)^{-1}\biggr]
\end{align}
Since players distributions differ on their variances, $\sum_{i\in \mathcal{S}} X_i^\top \varepsilon_i^{-2}X_i$ corresponds to a semi-correlated Wishart random variable where each $\frac{1}{\varepsilon_i}X_i \sim \mathrm{N}(0,(\frac{\sigma_i}{\varepsilon_i})^2 \mathrm{I}_d)$. In particular, the semi-correlated Wishart random variable can be approximated by a central Wishart distribution \cite{pivaro2017exact,tan1983approximating}, whose precision depends on the homogeneity of the coefficients $\sigma_i/\varepsilon_i$ over all $i \in \mathcal{I}$, as showed in \cite{khuri1994test}. It follows,
\begin{align}
\mathbb{E}_{(X_i\sim\mathrm{N}(0,\sigma_i^2\mathrm{I}_d))_{i\in\mathcal{S}}} \left[\biggl(\sum_{i \in \mathcal{S}} X_i^\top\varepsilon_i^{-2} X_i \biggr)^{-1}\right]
\approx \frac{\mathrm{I}_d}{(q({\mathcal{S}}) - d - 1)}\eqsp,
\end{align}
where
\begin{align}
    q(\mathcal{S}) := \left\lfloor\frac{\bigl(\sum\limits_{i\in \mathcal{S}} \frac{\sigma_i}{\varepsilon_i} n_i\bigr)^2}{\sum\limits_{i\in \mathcal{S}} \bigl(\frac{\sigma_i}{\varepsilon_i}\bigr)^2 n_i}\right\rfloor.
\end{align}
With all this in mind, it follows,
\begin{align}
\mathbb{E}_{\substack{(\eta_i^{(j)}\sim \mathrm{N}(0,\varepsilon_i^2))_{i\in\mathcal{S}}^{j\in[n_i]}\\(X_i\sim\mathrm{N}(0,\sigma_i^2\mathrm{I}_d))_{i\in\mathcal{S}}}}
\bigl[ \bigl(x^\top(\theta-\hat{\theta}_\mathcal{S})\bigr)^2\bigr] 
&= \mathrm{Tr}\biggl[x x^\top \frac{\mathrm{I}_d}{(q({\mathcal{S}}) - d - 1)}\biggr] 
= \frac{1}{q({\mathcal{S}}) - d - 1} \mathrm{Tr}\bigl[x x^\top\bigr].
\end{align}
In particular, considering $p_{\mathrm{test}} = \mathrm{N}(0,\mathrm{I}_d)$, we get,
\begin{align}
    u(\mathcal{S}) = \frac{d}{d + 1 - q({\mathcal{S}})}. 
\end{align}
\end{proof}

\subsection{Proof of \Cref{theorem:convergence_uniform}}

\textbf{\Cref{theorem:convergence_uniform}.}
\textit{Let $\{n_i,\gamma_i\}_{i\in [I]}$ be two sequences of positive numbers such that the following limits
\begin{align}
&\lim_{I\to \infty}\frac{1}{I}\sum\nolimits_{i\in [I]} n_i\gamma_i= \mu_A,
\quad \lim_{I\to \infty}\frac{1}{I}\sum\nolimits_{i\in [I]} (n_i\gamma_i - \mu_A)^2 = \sigma^2_A,\\
&\lim_{I\to \infty}\frac{1}{I}\sum\nolimits_{i\in [I]} n_i\gamma_i^2 = \mu_B,
\quad \lim_{I\to \infty}\frac{1}{I}\sum\nolimits_{i\in [I]} (n_i\gamma_i^2 - \mu_B)^2 = \sigma^2_B \eqsp,
\end{align}
all exist, for some constants $\mu_A,\mu_B, \sigma_A, \sigma_B > 0$. Let $\Kk \sim \mathrm{U}(\{0,\ldots,I\})$, $\mathcal{S}_{\Kk} \sim \mathrm{U}([2^{\mathcal{I}}_{\Kk}])$, and define $q(\mathcal{S}_\Kk)$ as in \eqref{eq:q(S)_definition} for the third use-case. Then, almost surely, $\nicefrac{q(\mathcal{S}_{\Kk})}{q(\mathcal{I})} \xrightarrow{I \to \infty} \mathrm{U}([0,1])$.}

\begin{proof}
Introduce, for any $t,t_0 \in (0,1)$ and any $s \bi 0$, 
\begin{align}
&\mu_A(I) = \frac{1}{I}\sum_{i \in [I]} n_i\gamma_i,\quad 
\mu_B(I) = \frac{1}{I}\sum_{i \in [I]} n_i\gamma_i^2,\\
&Y_A(t,I) = \sum_{i\in \mathcal{S}_{{\lfloor I t\rfloor}}}n_i\gamma_i,\quad 
Y_B(t,I) = \sum_{i\in \mathcal{S}_{{\lfloor I t\rfloor}}}n_i\gamma_i^2.\\
&R_A(I,t_0,s) = \mathbb{P}\biggl( \sup_{t>t_0} \biggl| \frac{Y_A(t,I)}{\lfloor I t\rfloor} - \mu_A(I)\biggr| > s\biggr),\\
&R_B(I,t_0,s) = \mathbb{P}\biggl( \sup_{t>t_0} \biggl| \frac{Y_B(t,I)}{\lfloor I t\rfloor} - \mu_B(I)\biggr| > s\biggr).
\end{align}
By construction, $Y_A(t,I)$ and $Y_B(t,I)$ are sums of sampling without replacement of $\lfloor It\rfloor$ elements. Therefore, by Corollary 1.3 in \cite{serfling1974probability}, for $s$ fixed, there exists $I_0^A, I_0^B \in \mathbb{N}$ such that,
\begin{align}
    R_A(I,t_0,s)\leq \frac{(1-t_0)\sigma_A^2}{\lfloor It_0\rfloor s^2}, \forall I\geq I_0^A \text{ and } R_B(I,t_0,s)\leq \frac{(1-t_0)\sigma_B^2}{\lfloor It_0\rfloor s^2}, \forall I\geq I_0^B.
\end{align}
In other words, for any $s \bi 0$ and $I$ large enough, almost surely, it holds,
\begin{align}
\biggl\vert \frac{Y_A(t,I)}{{\lfloor I t\rfloor}} - \mu_A(I) \biggr\vert \leq s \text{ and } \biggl\vert \frac{Y_B(t,I)}{{\lfloor I t\rfloor}} - \mu_B(I) \biggr\vert \leq s.
\end{align}
It follows,
\begin{align}
\biggl\vert \frac{q(\mathcal{S}_{\lfloor It \rfloor})}{\lfloor It \rfloor} - \frac{\mu_A(I)^2}{\mu_B(I)} \biggr\vert &= 
\biggl\vert \frac{1}{\lfloor It \rfloor}\cdot\frac{Y_A(t,I)^2}{Y_B(t,I)} - \frac{\mu_A(I)^2}{\mu_B(I)} \biggr\vert \\
&= \biggl\vert\biggl(\frac{Y_A(t,I)^2}{\lfloor It \rfloor^2} - \mu_A(I)^2 + \mu_A(I)^2\biggr) \biggl(\frac{\lfloor It \rfloor}{Y_B(t,I)} - \frac{1}{\mu_B(I)} \biggr)\\
&\quad+ \biggl(\frac{Y_A(t,I)^2}{\lfloor It \rfloor^2} - \mu_A(I)^2\biggr)\frac{1}{\mu_B(I)} \biggr\vert \\
&\leq s\biggl( s + \mu_A(I) + \frac{1}{\mu_B(I)}\biggr),
\end{align}
which is arbitrarily small as $\mu_A(I),\mu_B(I) \to \mu_A,\mu_B \sm \infty$. Therefore, almost surely,
\begin{align}
    \lim_{I\to\infty} \frac{q(\mathcal{S}_{\lfloor It \rfloor})}{I\mu_A(I)^2/\mu_B(I)} = t.
\end{align}
The proof concludes noticing that
\begin{align}
    q(\mathcal{I}) = \frac{I\mu_A(I)^2}{\mu_B(I)},
\end{align}
and that $\Kk = \lfloor I U\rfloor$ with $U\sim \mathrm{U}([0,1])$.
\end{proof}

\subsection{Proof of Theorem \ref{theorem:DU_shapley_error_bound}}\label{sec:proof_thm_bias_DU_Shapley}

To prove Theorem \ref{theorem:DU_shapley_error_bound}, we need two preliminary results: Lemma \ref{lemma:DU_Shapley_approx_error_bound}, which itself needs two supplementary results (Lemmas \ref{theorem:Hoeffing} and \ref{lemma:pre_error_bound_result}), and Lemma \ref{lemma:Assumption_on_second_derivative}, which is directly proved. 


\subsubsection{Technical lemmata}

\begin{lemma}
\label{theorem:Hoeffing}
Consider a set of $I$ values $N = \{n_1, \ldots, n_I\}$. Let $X_1, \ldots, X_k$ and $Y_1, \ldots, Y_k$ denote, respectively, $k$ random samples with and without replacement from $N$. For any continuous and convex function $f$, it follows,
\begin{align}
\mathbb{E}\biggl[f\biggl(\sum_{i = 1}^k Y_i\biggr)\biggr] \leq \mathbb{E}\biggl[f\biggl(\sum_{i = 1}^k X_i\biggr)\biggr] 
\end{align}
\end{lemma}
\begin{proof}
    The proof follows from \cite{doi:10.1080/01621459.1963.10500830}.
\end{proof}

\begin{lemma}\label{lemma:pre_error_bound_result}
Let $I \in \N$, $N:= \{n_1,\ldots ,n_I\}\in\R_+^I$, $\mu=\frac{1}{I}\sum_{i=1}^I n_i$ be their mean value and $\sigma^2=\frac{1}{I}\sum_{i=1}^I (n_i-\mu)^2$ be their variance. 
For $k \in \{0,\ldots,n\}$, let $\mathcal{S}_k \sim \mathrm{U}(\{S_k \subseteq [I]: |S_k| = k\})$ be a uniform random variable on the subsets of $\{1,\ldots,I\}$ of size $k$, and $n_{\mathcal{S}_k} =\sum_{i\in\mathcal{S}_k} n_i$ be the random variable defined by the sum of the elements of $\mathcal{S}_k$.
Let $\Kk \sim \mathrm{U}(\{0,\ldots,I\})$ and define $\mathbf{Y} = n_{\mathcal{S}_\Kk}$. Then, 
\begin{align}
&\E [\Yy - \mu \Kk \mid \Kk = k] = 0,\label{eq:Y_and_muK_have_the_same_exp_value}\\
&\E \bigl[\left( \Yy - \mu \Kk\right)^2\mid \Kk = k\bigr]\leq k\sigma^2.\label{eq:upper_bound_variance_Y_and_muK}
\end{align}
\end{lemma}

\begin{proof}
We prove \eqref{eq:Y_and_muK_have_the_same_exp_value} directly.
\begin{align}
\E[\Yy \mid \Kk = k] &= \sum_{S_k \subseteq [I]:  |S_k| = k} n_{S_k} \frac{1}{\binom{I}{k}} = \frac{1}{\binom{I}{k}} \sum_{S_k \subseteq [I]: |S_k| = k} \sum_{i \in S_k} n_i\\
&= \frac{1}{\binom{I}{k}} \sum_{i \in [I]} \sum_{\substack{S_k \subseteq [I] : |S_k| = k \\i \in S_k}} n_i \\
&= \frac{1}{\binom{I}{k}} \sum_{i \in [I]} n_i \binom{I-1}{k -1}\\
&= \frac{(I-k)!k!}{I!}\cdot\frac{(I-1)!}{(k-1)!(I-k)!} \sum_{i \in [I]} n_i = \mu k. 
\end{align}
Thus, \eqref{eq:Y_and_muK_have_the_same_exp_value} follows as $\E[\mu\Kk \mid \Kk = k ] = \mu k$. 
To prove \eqref{eq:upper_bound_variance_Y_and_muK}, let $(\Xx_i)_{i=1}^k$ be $k$ independent samples from the set $N$. From Lemma~\ref{theorem:Hoeffing} it holds,
\begin{align}
\E \bigl[\left( \Yy - \mu \Kk\right)^2\mid \Kk = k\bigr] &\leq \E \biggl[\bigl(\mu\Kk-\sum_{i=1}^{\Kk}\Xx_i\bigr)^2 \mid \Kk = k\biggr]= \E \biggl[\biggl(\sum_{i=1}^{\Kk} \left(\mu- \Xx_i\right)\biggr)^2 \mid \Kk = k\biggr].
\end{align}
Therefore,
\begin{align}
\E \bigl[( \Yy - &\mu \Kk)^2\mid \Kk = k\bigr] \leq \E \biggl[\biggl(\sum_{i=1}^{\Kk}\sum_{j=1}^{\Kk} \left(\mu- \Xx_i\right)\left(\mu- \Xx_j\right)\biggr) \mid \Kk = k\biggr]\\
&= \E \biggl[\biggl(\sum_{i=1}^{\Kk}\sum_{j=1}^{\Kk} \left(\mu^2 - \mu(\Xx_i + \Xx_j) + \Xx_i\Xx_j\right)\biggr) \mid \Kk = k\biggr]\\
&= \sum_{i=1}^k\sum_{j=1}^k \left(\mu^2 - \mu(\E[\Xx_i\mid \Kk = k] + \E[\Xx_j\mid \Kk = k]) + \E[\Xx_i\Xx_j\mid \Kk = k]\right)\\
&= \sum_{i=1}^k\sum_{j=1}^k \left(\mu^2 - \mu(\E[\Xx_i] + \E[\Xx_j]) + \E[\Xx_i\Xx_j]\right)\\
&= \sum_{i=1}^k \left(\mu^2 - 2\mu\E[\Xx_i] + \E[\Xx_i^2]\right) + \sum_{i=1}^k\sum_{\substack{j=1\\ j\neq i}}^k \left(\mu^2 - \mu(\E[\Xx_i] + \E[\Xx_j]) + \E[\Xx_i]\E[\Xx_j]\right)\\
&= \sum_{i=1}^k \E\bigl[\left(\mu - \Xx_i\right)^2\bigr] + \sum_{i=1}^k\sum_{\substack{j=1\\ j\neq i}}^k \left(\mu^2 - 2\mu^2 + \mu^2\right)\\
&= \sum_{i=1}^k \E\bigl[\left(\mu - \Xx_i\right)^2\bigr] = \sum_{i=1}^k \text{Var}\left(\mu - \Xx_i\right) = k \sigma^2.
\end{align}
The steps come from rearranging the terms, using the independence of $\Xx_i$ with respect to $\Kk$, the independence of $\Xx_i, \Xx_j$ for $i \neq j$, and finally that $\E[\Xx_i] = \mu$ and $\text{Var}\left(\Xx_i\right) = \sigma^2$.
\end{proof}

\begin{lemma}\label{lemma:DU_Shapley_approx_error_bound}
Let $I\in\N$, $N := \{n_1,\ldots ,n_I\} \in \R_+^I$, and define,
\begin{align}
&\mu = \frac{1}{I} \sum_{i=1}^I n_i,\quad \sigma^2 = \frac{1}{I}\sum_{i=1}^I (n_i-\mu)^2,\quad {n}^{\text{max}} = \max_{i\in \mathcal{I}} n_i,\\
&R := \max_{i \in [I]} |n_i - \mu|, \quad \tau = \max_{i\in [I]} n_i/\min_{i\in [I]} n_i.
\end{align}
Consider $\mathcal{S}_\Kk$, $n_{\mathcal{S}_\Kk}$, $\Kk$, and $\Yy$ as in Lemma \ref{lemma:pre_error_bound_result}. 
Let $w:\R_+\to \R$ be a function in $\mathcal{C}^2$, increasing, and suppose there exists $\kappa \in \mathbb{R}_+$, such that,
\begin{align}
\bigl|w^{(2)}(n)\bigr|\leq \frac{\kappa}{n^2}, \forall n \bi 0,
\end{align} 
where $w^{(k)}$ is the k-th derivative of $w$. Then, it holds,
\begin{align}\label{eq:upper_and_lower_bound_general_Shaple_value}
\bigl|\E[w(\mu\Kk) - w(\Yy)]\bigr| \leq \frac{\kappa}{2 \mu^2I} \left(9\sigma^2 (1+\ln(I)) + \frac{2R^2\tau^2}{{n}^{\text{max}}}\right).
\end{align}
\end{lemma}

\begin{proof}
The proof considers a second-order Taylor extension of $w$ at $\mu k$ to recover the expected value of $\E[w(\mu\Kk) - w(\Yy)]$. Noticing that the first derivative has a null expected value, the upper bound stated on the Lemma comes from bounding the expected value of the second derivative.

The Taylor-Lagrange Theorem on $w$ at $\mu k>0$ provides,
\begin{align}
       w(y)  =  w(\mu k ) + w^{(1)}(\mu k )(\mu k - y ) + w^{(2)}(\tau)\frac{(\mu k - y )^2 }{2},
\end{align}
for some $\tau$ between $y$ and $\mu k$. Therefore, there exists a random variable $\mathrm{T}$, almost surely between $\mu \Kk_+$ and $\Yy$, such that,
\begin{align}
      \E[w(\Yy) - w(\mu\Kk_+)] = \E\biggl[ w^{(1)}(\mu\Kk_+)(\mu\Kk_+ -\Yy) + \frac{1}{2} w^{(2)}(\mathrm{T})(\mu\Kk_+ -\Yy) ^2\biggr],
\end{align}
where $\Kk_+$ corresponds to $\Kk$ conditioned to be positive. To avoid overcharging the notation, we drop the index from $\Kk_+$. We observe that,
\begin{align}
    \E \biggl[w^{(1)}(\mu\Kk)(\mu\Kk -\Yy)\biggr] &= \E\biggl[\E \bigl[w^{(1)}(\mu\Kk)(\mu\Kk -\Yy)\mid \Kk = k \bigr] \biggr] \\
    &= \E\biggl[w^{(1)}(\mu k) \E \bigl[(\mu\Kk -\Yy)\mid \Kk = k\bigr] \biggr] = 0,
\end{align}
by Lemma \ref{lemma:pre_error_bound_result}, Equation \eqref{eq:Y_and_muK_have_the_same_exp_value}. Therefore, 
 \begin{align}
\bigl|\E[w(\Yy) - w(\mu\Kk)]\bigr| &= \frac{1}{2}\bigl| \E\bigl[ w^{(2)}(\mathrm{T})(\mu\Kk -\Yy)^2\bigr]\bigr|\\
&\leq \frac{1}{2} \E\bigl[ \bigl| w^{(2)}(\mathrm{T}) \bigr| (\mu\Kk -\Yy) ^2 \bigr]\\
&\leq \frac{1}{2} \E\biggl[\frac{\kappa}{\mathrm{T}^2}(\mu \Kk - \Yy)^2 \biggr] = \frac{\kappa}{2} \E\biggl[\frac{1}{\mathrm{T}^2}(\mu \Kk - \Yy)^2 \biggr].
\end{align}
Setting $\Ii:=\bigl\{|\mu\Kk -\Yy|\leq\frac{1}{2}(\mu\Kk+\Yy)\bigr\}$, the previous expected value can be expressed as,
\begin{align}
 \E\biggl[\frac{1}{\mathrm{T}^2}(\mu \Kk - \Yy)^2 \biggr] = \E \biggl[ \frac{1}{\mathrm{T}^2}(\mu\Kk -\Yy) ^2 \cdot \Ii\biggr] + \E \biggl[ \frac{1}{\mathrm{T}^2}(\mu\Kk -\Yy) ^2 \cdot \Ii^c\biggr].
\end{align}
We deal with each term separately. Notice that, as $\mathrm{T}$ is almost surely between $\Yy$ and $\mu\Kk$,
\begin{align}
|\mu\Kk - \Yy| \leq \frac{1}{2}(\mu\Kk + \Yy) \Longrightarrow \mathrm{T} \geq \frac{1}{3}\mu\Kk.
\end{align}
Thus,
\begin{align}
\E\biggl[ \frac{(\mu\Kk -\Yy) ^2}{\mathrm{T}^2} \cdot  \Ii \biggr] &\leq  \E\left[ \frac{(\mu\Kk -\Yy) ^2}{(\frac{\mu\Kk}{3})^2} \cdot  \Ii \right] = \frac{9}{\mu^2} \sum_{k=1}^I \frac{1}{I}\cdot\E\left[ \frac{(\mu k -\Yy) ^2}{k^2} \cdot  \Ii \mid\Kk=k\right]\\
&\leq \frac{9}{I\mu^2} \sum_{k=1}^I\E\left[ \frac{(\mu k -\Yy) ^2}{k^2} \mid\Kk=k\right]\\
&\leq \frac{9}{I\mu^2} \sum_{k=1}^I \frac{k\sigma^2}{k^2}\\
&\leq \frac{9\sigma^2}{I\mu^2}\cdot(1+\ln(I)).
\end{align}
Regarding the second term, denote $\overline{n} := \max_{i\in \mathcal{I}} n_i$ and $\underline{n} := \min_{i\in \mathcal{I}} n_i$. As $\Kk\underline{n} \leq \min\{\mu\Kk,\Yy\} \leq \mathrm{T}$, we have,
\begin{align}
\E\biggl[ \frac{(\mu\Kk -\Yy) ^2}{\mathrm{T}^2} \cdot\Ii^c\biggr] 
&\leq \E\left[\frac{(R\Kk)^2}{\mathrm{T}^2} \cdot\Ii^c\right] \leq \E\left[ \frac{(R\Kk)^2}{(\underline{n}\Kk)^2} \cdot\Ii^c\right]\\
&= \frac{R^2}{I\underline{n}^2}  \sum_{k=1}^I \E\left[\frac{1}{k^2} k^2 \cdot\Ii^c\mid \Kk =k\right]\\  
&=\frac{R^2}{I\underline{n}^2}  \sum_{k=1}^I \mathbb{P}\left(|\mu k -\Yy| > \frac{1}{2}(\mu k+\Yy)\mid \Kk =k\right)\\
&\leq \frac{R^2}{I\underline{n}^2} \sum_{k=1}^I \mathbb{P}\left(|\mu k -\Yy| > \frac{ \mu k}{2}\mid \Kk =k\right)\\
&\leq \frac{R^2}{I\underline{n}^2} \sum_{k=1}^I \exp\biggl(- \frac{\mu^2 k}{2\overline{n}}\biggr)\\
&= \frac{2R^2\tau^2}{I\mu^2\overline{n}} \sum_{k=1}^{I} \frac{\mu^2}{2\overline{n}} \exp\biggl(- \frac{\mu^2 k}{2\overline{n}}\biggr)\\
&\leq \frac{2R^2\tau^2}{I\mu^2\overline{n}} \int_{0}^{\infty} \frac{\mu^2}{2\overline{n}} \exp\biggl(- \frac{\mu^2 k}{2\overline{n}}\biggr) dk
= \frac{2R^2\tau^2}{I\mu^2\overline{n}},
\end{align}
as the integral corresponds to the cumulative distribution function of an exponential random variable of parameter $\lambda = \mu^2/2\overline{n}$. The upper bound on the theorem's statement is obtained when gathering all together.
\end{proof}

\begin{lemma}\label{lemma:Assumption_on_second_derivative}
Let $w: \mathbb{R}_+ \to \mathbb{R}_+$ be a smooth and increasing function such that $$\lim_{n\to \infty} n^2 |w^{(2)}(n)| \sm \infty.$$ Then, there exists $\kappa \bi 0$ such that $n^2 |w^{(2)}(n)|\leq \kappa$.
\end{lemma}

\begin{proof}
Notice that the assumptions imply, in particular, that $|w^{(2)}(n)|$ is bounded. We argue by contradiction. Suppose that for any $m \bi 0$, there exists $n_{m}$ such that 
$$n_{m}^2 |w^{(2)}(n_{m})| > m.$$ 
Suppose the sequence $(n_m)_m$ converges to a point $n^*$. Then,
\begin{align}
    \lim_{m\to \infty} n_m^2 |w^{(2)}(n_m)| \bi \lim_{m\to \infty} m = \infty,
\end{align}
which is a contradiction with $|w^{(2)}(n)|$ being bounded. Therefore, necessarily $(n_m)_m$ has to diverge. However, this implies,
\begin{align}
    \lim_{n\to \infty} n^2 |w^{(2)}(n)| &= \lim_{m\to \infty} n_m^2 |w^{(2)}(n_m)| \bi \lim_{m\to \infty} m = \infty,
\end{align}
obtaining again a contradiction.
\end{proof}

\subsubsection{Proof of Theorem \ref{theorem:DU_shapley_error_bound}}

We are ready to prove Theorem \ref{theorem:DU_shapley_error_bound}.

\textbf{Theorem} \ref{theorem:DU_shapley_error_bound}.
\textit{
Under Assumption \textbf{H}\ref{ass:homogeneity}, there exists a constant $\kappa > 0$, such that, for any $i \in \mathcal{I}$, it holds,
\begin{align}
    \bigl|\varphi_i - \psi_i \bigr| \leq \frac{\kappa}{(I-1) \mu_{-i}^2} \left(\sigma_{-i}^2 (1+\ln(I-1)) + \zeta_{-i}\right),
\end{align}
where $\varphi_i$ and  $\psi_i$ are respectively the Shapley value and the \texttt{DU-Shapley} of player $i$, $\mu_{-i} = \frac{1}{I-1}\sum_{j\in \mathcal{I}\setminus\{i\}} n_{j}$ is the average dataset size of other players,  $\sigma^2_{-i} = \frac{1}{I-1}\sum_{j\in \mathcal{I}\setminus\{i\}} (n_{j}-\mu_{-i})^2$ its empirical variance, and $\zeta_{-i}$ measures the variability of the dataset sizes across players. Formally, it is defined as 
$$\zeta_{-i}:=R_{-i}^2\frac{\tau_{-i}^2}{4{n}^{\mathrm{max}}_{-i}}$$ where
$R_{-i} := \max_{j \in \mathcal{I}\setminus\{i\}} |n_{j} - \mu_{-i}|$, ${n}^{\max}_{-i} := \max_{j\in \mathcal{I}\setminus\{i\}} n_{j}$, and $\tau_{-i}:= \frac{{n}^{\max}_{-i}}{\min_{j \in \mathcal{I} \setminus\{i\}} n_{j}}$.}

\begin{proof}
Under Assumption \textbf{H}\ref{ass:homogeneity}, Lemma \ref{lemma:Assumption_on_second_derivative} implies the existence of $\kappa \bi 0$ such that the value function $w$ satisfies all assumptions from Lemma \ref{lemma:DU_Shapley_approx_error_bound}. Theorem \ref{theorem:DU_shapley_error_bound} comes from (a) noticing that 
\begin{align}
\varphi_i = \mathbb{E}[w(\Yy_{-i} + n_i) - w(\Yy_{-i})],\quad \psi_i = \mathbb{E}[w(\Kk\mu_{-i} + n_i) - w(\Kk\mu_{-i})],
\end{align}
where $\Kk \sim \mathrm{U}([I-1])$ and $\Yy_{-i} = n_{\mathcal{S}^{(i)}_\Kk}$ with $\mathcal{S}^{(i)}_\Kk$ taking values on the subsets of $\mathcal{I}\setminus\{i\}$ of size $\Kk$, (b) writing
\begin{align}
|\varphi_i - \psi_i| \leq\ &|\mathbb{E}[w(\Yy + n_i) - w(\Kk\mu_{-i} + n_i) ]|
+  |\mathbb{E}[w(\Yy) - w(\Kk\mu_{-i})]|,
\end{align}
and (c) applying Lemma \ref{lemma:DU_Shapley_approx_error_bound} to each of the expected values, as the function $n \to w(n + n_i)$ also satisfies \textbf{H}\ref{ass:homogeneity}.
\end{proof}

\end{document}


\maketitle

\begin{abstract}

Many machine learning problems require performing \emph{dataset valuation}, \emph{i.e.} to quantify the incremental gain, to some relevant pre-defined utility, of aggregating an individual dataset to others.
As seminal examples, dataset valuation has been leveraged in collaborative and federated learning to create incentives for data sharing across several data owners.
The Shapley value has recently been proposed as a principled tool to achieve this goal due to formal axiomatic justification.
Since its computation often requires exponential time, standard approximation strategies based on Monte Carlo integration have been considered. 
Such generic approximation methods, however, remain expensive in some cases. 
In this paper, we exploit the knowledge about the structure of the dataset valuation problem to devise more efficient Shapley value estimators. 
We propose a novel approximation of the Shapley value, referred to as \emph{discrete uniform Shapley} (\texttt{DU-Shapley}) which is expressed as an expectation under a discrete uniform distribution with support of reasonable size.
We justify the relevancy of the proposed framework via asymptotic and non-asymptotic theoretical guarantees and show that \texttt{DU-Shapley} tends towards the Shapley value when the number of data owners is large.
The benefits of the proposed framework are finally illustrated on several dataset valuation benchmarks. \texttt{DU-Shapley} outperforms other Shapley value approximations, even when the number of data owners is small.  
\end{abstract}

\section{Introduction}
\label{sec:introduction}

One of the main challenges for training machine-learning (ML) models with enough generalisation capabilities is to access  a sufficiently large set of labeled training data.
These data often exist but are commonly spread across many parties impairing their usage in a direct and simple way. 
A seminal example lies in the advertising industry, where consented data about browsing and shopping habits of individual users are distributed and owned by several websites including advertisers' and publishers' ones, each of them holding a set of observations with either similar or complementary features.
By collaborating with each other and pooling their individual data, these websites could learn better ML models for their applications than by only leveraging their local data. 
However, such collaboration naturally raises questions regarding the additional (positive or negative) values each party would obtain by participating in this joint machine-learning effort. 
In order to compute or estimate compensating rewards allowing to incentivise parties to share data, a first stage that is commonly considered in the literature is to perform so-called \emph{dataset valuation} \citep{10.5555/3524938.3525766, Tay2021IncentivizingCI,10.1145/3328526.3329589}.

Dataset valuation aims at quantifying the marginal contribution of a specific dataset to a given ML task with respect to (w.r.t.) datasets brought by other data owners.
Motivated by natural properties expected for equitable data valuation, value notions from cooperative game theory \citep{2011Chalkiadakis} have been leveraged to achieve this objective, including the core \citep{gillies1959solutions}, the Shapley value \citep{P-295} or the Banzhaf value \citep{Banzhaf_1965_5380}.
Among the latter, the Shapley value has received a large attention and is arguably the most widely studied data valuation scheme, since it is the unique value notion that satisfies a set of four important axioms \citep{P-295}. 
For instance, under the specific setting where all parties only possess one datum, several works leveraged the Shapley value or close variants to measure the average change in a trained ML model's performance when a particular datum is removed \emph{e.g.,}  \texttt{Data Shapley} \citep{pmlr-v97-ghorbani19c,pmlr-v89-jia19a}, \texttt{DShapley} \citep{pmlr-v119-ghorbani20a,pmlr-v130-kwon21a}, \texttt{Beta Shapley} \citep{pmlr-v151-kwon22a} or \texttt{CS-Shapley} \citep{schoch2022csshapley}.
To cope with the computational intractability of the Shapley value, the typical technique consists in using Monte Carlo integration, possibly enhanced with variance reduction techniques \emph{e.g.,} antithetic sampling \citep{JMLR:v23:21-0439}. 
However, such generic approximation strategies might still remain expensive in some cases.
It is notably the case in ML applications involving complex models since each sample requires re-training them, hence drastically limiting the number of Monte Carlo samples that can be drawn.

Instead of relying on generic Monte Carlo approximation schemes of the Shapley value, we  aim at finding more efficient estimators by leveraging the actual underlying structure of the dataset valuation problem.
More precisely, we propose to explicitly exploit the dependence of the utility function, used to define the Shapley value, on the number of data points of a given dataset.
This leads to a new Shapley value approximation, referred to as \emph{discrete uniform Shapley} (\texttt{DU-Shapley}), which stands for an expectation under a discrete uniform distribution whose support  size corresponds to the number of data owners.
Compared to computing the exact Shapley value of a given dataset, \texttt{DU-Shapley} has  an exponential reduction of the number of utility evaluations. 
Interestingly, \texttt{DU-Shapley} admits the appealing property of converging almost surely towards the exact Shapley value when the number of data owners increases; a setting where generic Monte Carlo strategies fail at providing good estimates under a limited budget.
On the other hand, for a fixed number of data owners and under mild assumptions on the utility function, we show that the error between \texttt{DU-Shapley} and the Shapley value is bounded and depends explicitly on key quantities of the dataset valuation problem, namely (i) the average number of data points in parties' datasets and associated variance, and (ii) constants associated to the growth of the utility function w.r.t. the
dataset size.

\noindent\textbf{Contributions.} We summarise our main contributions as follows:
\begin{enumerate}
    \item We propose \texttt{DU-Shapley}, an efficient proxy of the Shapley value to perform dataset valuation. This is the first dataset valuation approach leveraging the specific structure of the utility function.
    
    \item We provide both asymptotic and non-asymptotic guarantees for \texttt{DU-Shapley}, showing notably that it converges almost surely towards the Shapley value in the specific regime where the number of parties is large.

    \item We instantiate and justify all our statements and assumptions using a running example, standing for a linear regression problem commonly studied in the literature \citep{Donahue_Kleinberg_NeurIPS_2021}.
    In particular, we obtain closed-form expressions regarding the dependence of the utility function w.r.t. the dataset size.
    
    \item We assess the benefits of the proposed methodology using numerical experiments on both toy and real-world dataset valuation problems. In particular, we show that \texttt{DU-Shapley} outperforms generic Monte Carlo approximations of the Shapley value, and its variants. 
\end{enumerate}

\noindent\textbf{Related Works.}
The Shapley value has been recently applied to several ML problems including variable selection \citep{10.5555/1642293.1642400}, feature importance \citep{lumdberg,10.5555/3295222.3295230,10.5555/3495724.3497168}, model interpretation \citep{chen2018lshapley} or to provide data sharing incentives in collaborative ML problems \citep{10.5555/3524938.3525766, Tay2021IncentivizingCI}. 
Regarding datum or dataset valuation, several approaches have been considered in the literature.
The main lines of work include Shapley-based methods \citep{pmlr-v97-ghorbani19c,pmlr-v89-jia19a,pmlr-v119-ghorbani20a,pmlr-v130-kwon21a,pmlr-v151-kwon22a,schoch2022csshapley,10.14778/3342263.3342637}, leave-one-out-based methods using influence functions \citep{LOO_1,LOO_2}, the use of other solutions concepts from cooperative game theory by relaxing the efficiency axiom verified by the Shapley value \citep{Yan_Procaccia_2021,pmlr-v206-wang23e}, and volume-based methods \citep{NEURIPS2021_59a3adea}.

\noindent \textbf{Conventions and Notations.} 
For $n \ge 1$, we define $[n]:=\{1,\ldots,n\}$.
The $d$-multidimensional Gaussian probability distribution with mean $\mu \in \Rd$ and covariance matrix $\Sigma \in \mathbb{R}^{d \times d}$ is denoted by $\gauss\parentheseLigne{\mu,\Sigma}$.
The Uniform distribution over a set $\mathcal{A}$ is referred to as $\mathrm{U}(\mathcal{A})$. 

\section{Preliminaries}
\label{sec:preliminaries}

This section presents the dataset valuation problem we aim at solving, along with preliminaries including the definition of the Shapley value.

\subsection{Problem Formulation}
\label{subsec:problem_formulation}

We consider a collaborative machine learning setting involving a set $\mathcal{I}$ of $I=|\mathcal{I}| \in \mathbb{N}^*$ data owners (also referred to as players in the sequel) who are willing to cooperate in order to solve a common machine learning problem.
These $I$ players are assumed to possess individual datasets $\{\mathrm{D}_{i}\}_{i \in \mathcal{I}}$ such that, for any $i \in \mathcal{I}$, $\mathrm{D}_{i} = \{(x_i^{(j)},y_i^{(j)})\}_{j \in [n_i]}$ where $x_i^{(j)}$ stands for a feature vector, $y_i^{(j)}$ is a label and $n_i = |\mathrm{D}_{i}|$ refers to the number of data points in $\mathrm{D}_{i}$.
%
%
We are interested in quantifying the incremental contribution that a given player $i \in \mathcal{I}$ brings by sharing its dataset $\mathrm{D}_i$ with other players towards solving the ML task at stake.
To meet this objective, we assume the availability of a trusted central entity that aims at collecting and aggregating data from these players to quantify the value of their individual datasets.
To measure the individual value of a dataset, many valuation metrics, such as the Shapley value, compute the marginal contribution of a given player $i$ to an existing coalition of players $\mathcal{S} \subseteq \mathcal{I}$, defined by $u(\mathcal{S} \cup \{i\}) - u(\mathcal{S})$, where $u: 2^I \rightarrow \mathbb{R}$ stands for a utility function.
For any coalition $\mathcal{S} \subseteq \mathcal{I}$ of players, $u(\mathcal{S})$ quantifies how well players in $\mathcal{S}$ could solve the considered ML task.
As an example, for a classification problem, $u(\mathcal{S})$ is typically chosen as the prediction accuracy, calculated on a hold-out testing dataset and associated to the best model the parties in $\mathcal{S}$ could train by aggregating their training data.
Without loss of generality, we assume in the sequel that $u(\emptyset) = 0$.

\noindent \textbf{Running Example: Linear Regression.} In order to ease understanding and illustrate our statements, we describe the following running example, that corresponds to a simple linear regression problem. This example was used for instance in \citet{Donahue_Kleinberg_NeurIPS_2021} to characterise players' incentives regarding data sharing under the federated learning paradigm.
For any $i \in \mathcal{I}$, we consider the following generative linear model regarding the dataset $\mathrm{D}_i$:
\begin{equation}
\label{eq:lin_reg}
\begin{aligned}
    &Y_i = X_i\theta + \eta_i\eqsp, \ \eta_i \sim \mathrm{N}(0_{n_i},\varepsilon_i^2\mathrm{I}_{n_i})\eqsp,\label{eq:likelihood_linear}\\
    &x_i^{(j)} \sim p_X\eqsp, \text{ for any } j \in [n_i], \text{ and } \varepsilon_i \sim p_\varepsilon\eqsp,
\end{aligned}
\end{equation}
where $\theta \in \R^d$ is a ground-truth parameter, $X_i \in \R^{n_i \times d}$ is defined by $X_i = ([x_i^{(1)}]^\top,\ldots,[x_{i}^{(n_i)}]^\top)^\top$ and $Y_i \in \R^{n_i}$ is defined by $Y_i = (y_i^{(1)},\ldots,y_i^{(n_i)})^\top$.
For the sake of simplicity, both the one-dimensional $p_\varepsilon$ distribution and the $d$-dimensional distribution $p_X$ are common to the $I$ players.
Under the linear regression framework defined in~\eqref{eq:lin_reg}, the utility function of a set $\mathcal{S} \subseteq \mathcal{I}$ of players is defined by the negative expected mean square error over a hold-out dataset, that is
\begin{equation}
    \label{eq:utility_lin_reg}
    u(\mathcal{S}) = -\mathbb{E}\br{\pr{x^\top \hat{\theta}_{\mathcal{S}} - x^\top \theta}^2}\eqsp,
\end{equation}
where the expectation is taken over the distribution $p_X^{\mathrm{test}}$ of a hold-out testing datum $x \in \R^d$ and $\hat{\theta}_{\mathcal{S}}$, defined by $\hat{\theta}_{\mathcal{S}} = ( X_{\mathcal{S}}^\top X_{\mathcal{S}})^{-1} X_{\mathcal{S}}^\top Y_{\mathcal{S}}$, stands for the maximum likelihood estimator. 
The notations $X_{\mathcal{S}}$  and $Y_{\mathcal{S}}$ refer to the concatenation of $\{X_i\}_{i \in \mathcal{S}}$ and $\{Y_i\}_{i \in \mathcal{S}}$, respectively.
Similarly to \citep{Donahue_Kleinberg_NeurIPS_2021}, we chose to define the utility as a negative prediction error since there is no equivalent notion of accuracy for regression tasks.

\subsection{Shapley Value}

\noindent \textbf{Definition.} The Shapley value \citep{P-295} is a classical solution concept in cooperative game theory to allocate the total gains generated by a
coalition of players. 
Given a utility function $u$, the Shapley value of a player $i$ is defined as the average marginal
contribution of her dataset $\mathrm{D}_i$ to all possible subsets of $\mathrm{D}_{-i} = \{\mathrm{D}_j\}_{j \in \mathcal{I}\setminus\{i\}}$ built by aggregating datasets of other players.
Formally, the Shapley value $\varphi_i$ of player $i$ writes
\begin{equation}
    \label{eq:Shapley_def1}
    \varphi_i(u) = \frac{1}{|\Pi(\mathcal{I})|} \sum_{\pi \in \Pi(\mathcal{I})} [u(\mathcal{P}_i^{\pi} \cup \{i\}) - u(\mathcal{P}_i^{\pi})]\eqsp,
\end{equation}
where $\Pi(\mathcal{I})$ refers to the set of permutations over $\mathcal{I}$ and by $\mathcal{P}_i^{\pi}$ to the set of predecessors of player $i \in \mathcal{I}$ in permutation $\pi \in \Pi(\mathcal{I})$. 
The Shapley value of player $i$ can be equivalently expressed as
\begin{align}
     \label{eq:Shapley_def2}
    \varphi_i(u) = \frac{1}{I}\sum_{\mathcal{S} \subseteq \mathcal{I} \setminus \{i\}} \frac{1}{\binom{I-1}{|\mathcal{S}|}} [u(\mathcal{S} \cup \{i\}) - u(\mathcal{S})]\eqsp.
\end{align}
The Shapley value has been commonly used for data valuation and more generally in cooperative game theory as it uniquely satisfies the following set of desirable properties.
\begin{enumerate}
    \item \emph{Efficiency.} $\sum_{i=1}^I \varphi_i(u) = u(\mathcal{I})$. The sum of Shapley values for each data owner is the value of the grand coalition $\mathcal{I}$.
    \item \emph{Symmetry.} If, for any $\mathcal{S} \subseteq \mathcal{I} \setminus \{i_1,i_2\}$, $u(\mathcal{S} \cup \{i_1\}) = u(\mathcal{S} \cup \{i_2\})$, then $\varphi_{i_1}(u) = \varphi_{i_2}(u)$. If two players have the same marginal effect on each coalition, their Shapley values coincide.
    \item \emph{Dummy.}  If, for any $\mathcal{S} \subseteq \mathcal{I} \setminus \{i\}$, $u(\mathcal{S} \cup \{i\}) = u(\mathcal{S})$, then $\varphi_{i}(u) = 0$. The player whose marginal impact is always zero has a Shapley value of zero.
    \item \emph{Linearity.}  $\varphi_i(u_1 + u_2) = \varphi_i(u_1) + \varphi_i(u_2)$. The Shapley values of sums of games are the sum of the Shapley values of the respective games.
\end{enumerate}

\noindent \textbf{Monte Carlo Approximation.}
Evaluation of the Shapley value is known to be computationally expensive  in general \citep{10.2307/3690220}.
As such, many works \citep{RM-2651,CASTRO20091726,JMLR:v23:21-0439} proposed to approximate it via Monte Carlo by sampling with replacement $T$ terms from the sum of either \eqref{eq:Shapley_def1} or \eqref{eq:Shapley_def2}.
Regarding \eqref{eq:Shapley_def1}, this boils down to considering the estimator 
\begin{equation}
\label{eq:MC}
    \hat{\varphi}_i(u) = \frac{1}{T} \sum_{t=1}^T [u(\mathcal{P}_i^{\pi_t} \cup \{i\}) - u(\mathcal{P}_i^{\pi_t})]\eqsp,
\end{equation}
where $\pi_t \sim \mathrm{U}(\Pi(\mathcal{I}))$.
By using Hoeffding’s bound \citep{doi:10.1080/01621459.1963.10500830}, one can show that a lower bound on the
number of permutations $T$ such that $\mathbb{P}(\|\varphi_i(u) - \hat{\varphi}_i(u)\|_2 \leq \varepsilon) \geq 1 - \delta$ is given by $T_{\text{perm}}(\varepsilon,\delta) = (2r_u^
2I/\varepsilon^2)\log(2I/ \delta)$, where $\varepsilon, \delta \in (0,1)$ and $r_u = \max_{\mathcal{S}_1, \mathcal{S}_2 \subseteq \mathcal{I}}\{|u(\mathcal{S}_1) - u(\mathcal{S}_2)|\}$ is the range of the utility function $u$.
Note that $T_{\text{perm}}(\varepsilon,\delta) = \mathrm{O}(I \log(I))$, which is much lower than the $2^I$ terms involved in the sum of the Shapley value in \eqref{eq:Shapley_def2} in the regime where the number of players $I$ is large.
Such Monte Carlo approximation has for instance been used in the seminal paper introducing \texttt{Data Shapley} for datum valuation in ML \citep{pmlr-v97-ghorbani19c}.

\section{Proposed Approach}
\label{section:proposed-approach}
In this section, we revisit the generic definition of the Shapley value in \eqref{eq:Shapley_def1}-\eqref{eq:Shapley_def2} by exploiting the dependence of the utility function on the number of data points brought by a set of players.
This  leads  to the introduction of a novel proxy of the Shapley value referred to as \texttt{DU-Shapley}. All proofs are postponed to the supplementary material. 

\subsection{A Novel Perspective of the Shapley Value for Dataset Valuation} 

In Section \ref{sec:preliminaries}, we have considered a generic but commonly used definition for the Shapley value using a utility function $u$ defined over a set of players.
However, recall that players are cooperating by sharing their datasets $\{\mathrm{D}_i\}_{i \in \mathcal{I}}$.
The latter are contributing to the performances of the considered ML model via their cardinality $\{n_i\}_{i \in \mathcal{I}}$ and quality (\emph{e.g.,} defined by the discrepancy between the distribution of $\mathrm{D}_i$ and that of the hold-out testing dataset).
For the sake of clarity of exposure, we shall consider in the remaining that the qualities of the $I$ datasets are identical so that the only lever of performance to the ML model is the number of data points brought by each player.
This statement is illustrated below, thanks to the running example described in Section \ref{subsec:problem_formulation}.

\noindent \textbf{Running Example.} Under mild assumptions on $p_X$ and $p_X^{\text{test}}$, the following result shows that a set of players $\mathcal{S}$ only contributes to the utility $u(\mathcal{S})$ via its aggregated number of data points $n_{\mathcal{S}} = \sum_{i \in \mathcal{S}} n_i$.
\begin{proposition}\label{lemma:MSE_homogeneous_case}
    Define $C = \mathbb{E}_{x \sim p_X^{\mathrm{test}}}[x x^\top]$ and assume that $p_X = \mathrm{N}(0_d, \Sigma)$ where $\Sigma \in \mathbb{R}^{d \times d}$ is a positive definite matrix. 
    Then, whenever $n_{\mathcal{S}} > d + 1$,
    \begin{equation}
      u(\mathcal{S}) = \frac{\sigma^2_\varepsilon}{d+1 - n_{\mathcal{S}}}\mathrm{Tr}\br{C\cdot\Sigma^{-1}}\eqsp,\label{eq:closed_form_homogeneous}
    \end{equation}
    where $u$ is defined in \eqref{eq:utility_lin_reg} and $\sigma_{\varepsilon}^2 = \mathbb{E}_{\varepsilon \sim p_\varepsilon}[\varepsilon^2 ]$. 
    For the specific choice $p_X^{\mathrm{test}} = p_X$ then, whenever $n_{\mathcal{S}} > d + 1$,  we have,
    \begin{equation}
      u(\mathcal{S}) = \frac{d\sigma^2_\varepsilon}{d+1-n_{\mathcal{S}}}\eqsp.\label{eq:closed_form_homogeneous_2}
    \end{equation}
\end{proposition}

Such a result motivates a refined version of the utility function $u$ taking into account the explicit dependence on the aggregated number of data points associated to a set of players. 
As such, we introduce a function $w: \mathbb{R}_+ \rightarrow \R$ such that for any $\mathcal{S} \subseteq \mathcal{I},
u(\mathcal{S}) = w(n_{\mathcal{S}})$.
Under this re-parametrisation, the Shapley value defined in \eqref{eq:Shapley_def2} can be equivalently written as
\begin{align}
    \varphi_i(u) 
    &= \frac{1}{I}\sum_{k=0}^{I-1} \sum_{\mathcal{S} \subseteq \mathcal{I} \setminus \{i\}: |\mathcal{S}| = k} \frac{1}{\binom{I-1}{|\mathcal{S}|}} [w(n_\mathcal{S} + n_i) - w(n_\mathcal{S})] \\
    &= \mathbb{E}_{K \sim \mathrm{U}(\{0,\ldots,I-1\})}\br{\mathbb{E}_{\mathcal{S} \sim \mathrm{U}\bigl(\bigl[2^{\mathcal{I}\setminus\{i\}}_{K}\bigr]\bigr)}\br{w(n_\mathcal{S} + n_i) - w(n_\mathcal{S})}} \eqsp,\label{eq:Shapley_def3}
\end{align}
where the first equality is obtained by re-arranging all the possible sets of players by their cardinality $k \in \{0,\ldots,I-1\}$, and the second one by introducing the notation $2^{\mathcal{I}\setminus\{i\}}_{K}$ to refer to all subsets of $\mathcal{I}\setminus \{i\}$ of cardinality $K$, with $K \sim U(\{0,...,I-1\})$.

\subsection{Discrete Uniform Shapley Value}
\label{subsec:DU-SHAPLEY}

Equation \eqref{eq:Shapley_def3} explicitly reveals a key random variable, namely $n_\mathcal{S}$, obtained by first drawing uniformly a set cardinality $k$, then drawing a subset $\mathcal{S}$ of $k$ players uniformly from $\mathcal{I}\setminus \{i\}$, and finally setting $n_\mathcal{S} = \sum_{j \in \mathcal{S}} n_j$.
Interestingly, we show in the next result that this random variable tends towards a uniform distribution as the number of players grows. 

\begin{theorem}
    \label{theorem:convergence_uniform}
    Let $n = \{n_i\}_{i\in [I]}$ be a sequence of positive integer numbers such that the following limits exist, 
    $$\lim_{I\to \infty}\frac{1}{I}\sum_{j=1}^{I}n_j= \mu \text{ and } \lim_{I\to \infty}\frac{1}{I}\sum_{j=1}^I (n_j -\mu)^2=\sigma^2\eqsp,$$
    where $\mu, \sigma > 0$.
    For any $i \in \mathcal{I}$, let $K \sim \mathrm{U}(\{0,\ldots,I-1\})$ and $\mathcal{S}_{K}^{(i)} \sim \mathrm{U}([2^{\mathcal{I}\setminus\{i\}}_{K}])$; and define $n_{\mathcal{S}_K^{(i)}} = \sum_{j \in \mathcal{S}_K^{(i)}} n_j$ to be the random variable corresponding to the total number of data points brought by the random set $\mathcal{S}_K^{(i)}$ of $K$ players. 
    Then, for any $i \in \mathcal{I}$, it holds,
\begin{align}
    \frac{n_{\mathcal{S}_{K}^{(i)}}}{\sum_{j \in \mathcal{I} \setminus \{i\}}n_j} \xrightarrow{I \to \infty} \rm{U}([0,1])\eqsp, \quad \text{almost surely} \eqsp.
\end{align}    
\end{theorem}

\begin{figure}
  \begin{center}
    \includegraphics[scale=0.35]{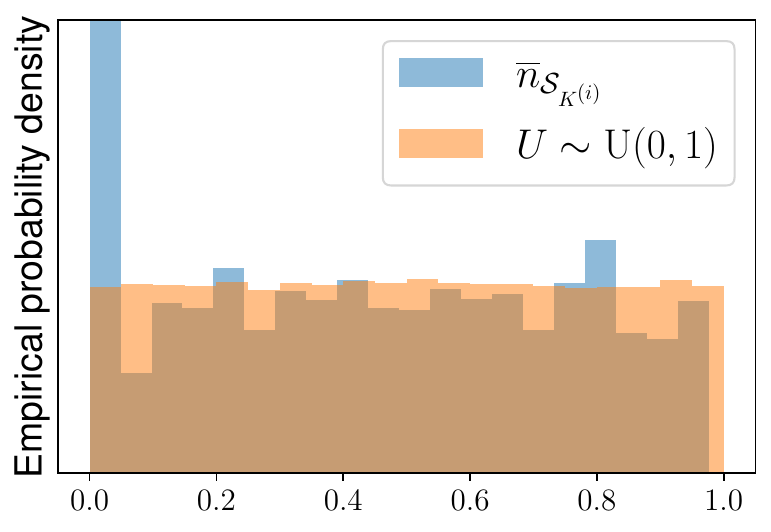}
    \includegraphics[scale=0.35]{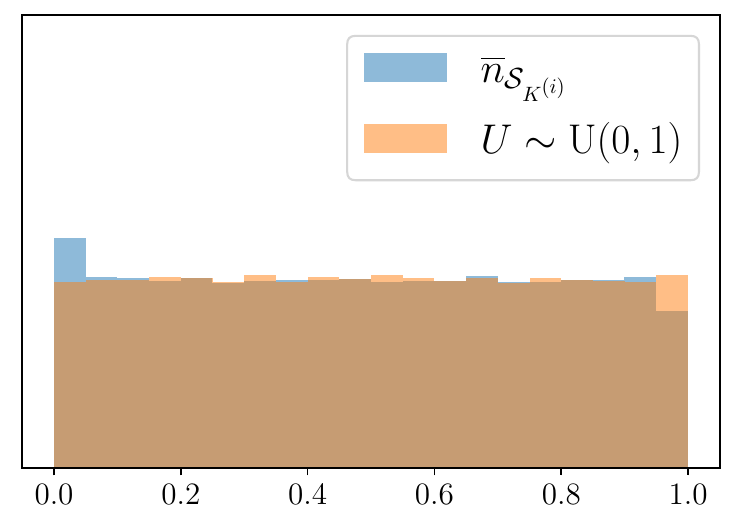}
    \includegraphics[scale=0.35]{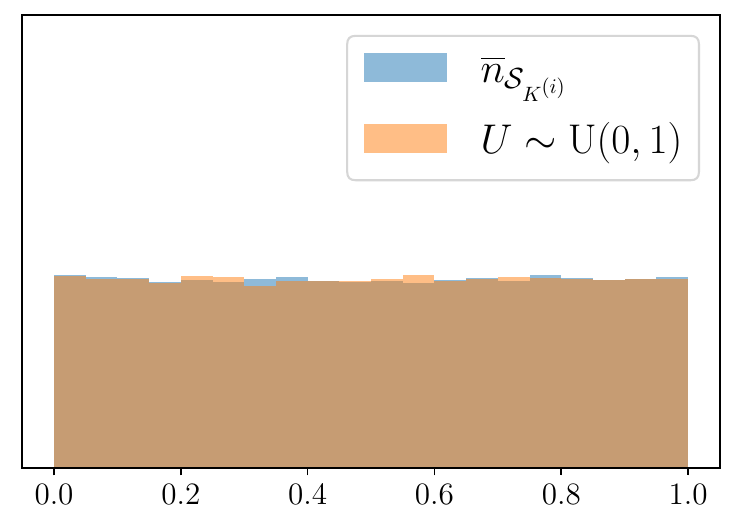}
  \end{center}
  \caption{Illustration of Theorem \ref{theorem:convergence_uniform} -- (left) $I = 10$, (middle) $I=50$, (right) $I=500$. 
  We choose the player of index $i=500$ to be the $i$-th one, considered $10^5$ samples for each random variable and a number of data points per player drawn from $\mathrm{U}([100])$. The random variable $\bar{n}_{\mathcal{S}_K^{(i)}}$ stands for $n_{\mathcal{S}_K^{(i)}}$ normalised by the total number of data points as in Theorem \ref{theorem:convergence_uniform}.}
  \label{fig:approx_uniform}
\end{figure}
Such a result, illustrated in Figure \ref{fig:approx_uniform}, motivates to approximately consider that $n_{\mathcal{S}}$ in \eqref{eq:Shapley_def3} stands for a uniform random variable supported on $\{0,\ldots,\sum_{j \in \mathcal{I}\setminus\{i\}}n_j\}$.
Albeit smaller than the size of all subsets of $\mathcal{I}$, the size of this new support may intractably increase when the aggregated number of data of all players is important.
Building upon the fact $n_{\mathcal{S}}$ has a uniform distribution over this support, this naturally leads us to  approximate the Shapley value in \eqref{eq:Shapley_def3} by (i) calculating the quantities $w(n_\mathcal{S} + n_i) - w(n_\mathcal{S})$ taking $n_\mathcal{S}$ on a regularly-spaced grid between $0$ and $\sum_{j \in \mathcal{I}\setminus\{i\}}n_j$ and, then, (ii) average them. 
A natural candidate for the number of discretisation bins is the total number of players when excluding the $i$-th one, \emph{i.e.} $I-1$. 
Indeed, whenever all agents have the same number of data points, \emph{i.e.,} $n_j = n$, for any $j \in \mathcal{I}$, the random variable $n_\mathcal{S}$ is uniformly distributed on $\{0,n, 2n,...,(I-1)n\}$. 
This leads to the proposal of a novel approximation for the Shapley value defined hereafter.

\begin{definition}
    For any $i \in \mathcal{I}$, the discrete uniform Shapley value (\texttt{DU-Shapley}) of the $i$-th player, denoted by $\psi_i$, is defined by
    \begin{align}
        \label{eq:DU-shapley}
        \psi_i(u) = \frac{1}{I}\sum_{k=0}^{I-1} \br{w(k\mu_{-i} + n_i) - w(k\mu_{-i})}\eqsp,\label{eq:DU_Shapley}
    \end{align}
where $\mu_{-i} = \frac{1}{I-1}\sum_{j \in \mathcal{I}\setminus\{i\}}n_j$.
\end{definition}

Compared to the Shapley value defined in \eqref{eq:Shapley_def2}, which involves $2^I$ terms to compute, note that \texttt{DU-Shapley} only involves $I$ terms and hence an exponential reduction of the number of utility evaluations.
Of course, these computational savings come at the cost of some bias.
The latter is  quantified precisely in the next section.

\subsection{Non-Asymptotic Theoretical Guarantees}

In order to provide statistical guarantees on the described procedure, we need to make some structural assumptions. Precisely, we are going to consider utility functions satisfying the following:

\begin{assumption}\label{ass} Let $w: \mathbb{R}_+ \rightarrow \R$ such that,
    \begin{enumerate}[wide, labelwidth=!, labelindent=0pt,label=(\roman*),noitemsep,nolistsep]

    \item \label{ass:1} the function $w$ is increasing;
    \item \label{ass:2} the function $w$ is twice continuously differentiable and its second derivative $w^{(2)}$ satisfies $\lim_{n \to \infty} n^2|w^{(2)}(n)| \sm \infty$.
    
    \end{enumerate}
\end{assumption}

The first assumption is not restrictive, as we expect that the more data, the more precise the ML prediction will be.
The second one aims at controlling the growth behavior of the utility function $w$; it is also not that restrictive, as it is for instance automatically satisfied if $w$ is bounded (and $w^{(2)}$ monotone), thanks to the mean value theorem.
Under, these assumptions, we have the following non-asymptotic result:

\begin{theorem}\label{theorem:DU_shapley_error_bound}
Assume \textbf{H}\ref{ass}. 
Then, there exists a constant $\rho \bi 0$ such that, for any $i \in \mathcal{I}$, the approximation error of \texttt{DU-Shapley} for player $i \in \mathcal{I}$ is upper bounded by
\begin{align}
    \bigl|\varphi_i - \psi_i \bigr| \leq \frac{\rho |w(n_{\mathcal{I}\setminus\{i\}})|}{(I-1) \mu_{-i}^2} \left(9\sigma_{-i}^2 (1+\ln(I-1)) + 2R_{-i}^2n^{\mathrm{max}}_{-i}\right), \label{eq:bound_error}
\end{align}
where, $\varphi_i$ and $\psi_i$ are defined in \eqref{eq:Shapley_def2} and \eqref{eq:DU-shapley}, respectively; and where $\mu_{-i} = \frac{1}{I-1}n_{\mathcal{I}\setminus\{i\}}$, $\sigma^2_{-i} = \frac{1}{I-1}\sum_{j\in \mathcal{I}\setminus\{i\}} (n_j-\mu_{-i})^2$, $R_{-i} := \max_{j \in \mathcal{I}\setminus\{i\}} |n_j - \mu_{-i}|$, and $n^{\mathrm{max}}_{-i} := \max_{j \in \mathcal{I} \setminus\{i\}} n_j$.
\end{theorem}

It is worth mentioning that the upper bound in \eqref{eq:bound_error} depends on natural quantities related to the dataset valuation problem described in Section \ref{subsec:problem_formulation}.
Indeed, it depends on the first two moments $\mu_{-i}$ and $\sigma_{-i}$ of the datasets' size distribution.
More precisely, the error increases when there are some outlier players with a very small or large dataset size.
This behavior is expected since, in this particular setting, the random variable $n_{\mathcal{S}}$ defined in Section \ref{subsec:DU-SHAPLEY} differs from a uniform random variable.
In addition, we can interestingly observe that the error vanishes in two specific regimes, namely (i) when the number of players $I$ tends towards infinity as showcased in Theorem \ref{theorem:convergence_uniform}, and (ii) when all players have the same number of data points.

\noindent \textbf{Running Example.} We  emphasise that Assumption \textbf{H}\ref{ass} is verified by our running example.
Recall from \eqref{eq:utility_lin_reg} that $w(n) = d\sigma_{\varepsilon} / (d+ 1 - n)$. 
Then, $w$ is indeed increasing, twice differentiable with second derivative given by $w^{(2)}(n) = 2d\sigma_{\varepsilon} / (d+ 1 - n)^3$ which satisfies \textbf{H}\ref{ass}-\ref{ass:2}.

It is worth mentioning that computing the \texttt{DU-Shapley} of a player $i \in \mathcal{I}$ requires $I$ function evaluations. As pointed out in Section \ref{sec:preliminaries}, the Monte Carlo approximation incurs an approximation error of $\varepsilon(T)$ (with probability $1-\delta$), with $T$ being the number of sampled permutations, equal to $\varepsilon(T) = \frac{2}{T}r_w^2I\log\bigl(\frac{2I}{\delta}\bigr)$. In particular, fixing the sampling budget to $I$, \emph{i.e.,} equal to the number of terms needed to compute \texttt{DU-Shapley}, the Monte Carlo approximation error becomes
\vspace{-0.1cm}
\begin{align}\label{eq:MC_error_for_I_samples}
    \varepsilon(I) = 2w(n_{\mathcal{I}})^2\log\biggl(\frac{2I}{\delta}\biggr),
\end{align}
where we have supposed that $0 = w(n_{\varnothing}) \leq w(n_{\mathcal{S}}) \leq w(n_{\mathcal{I}})$ for any $\mathcal{S}\subseteq \mathcal{I}$, so that $r_w = w(n_{\mathcal{I}})$. Notice that $\varepsilon(I)$ is increasing with $I$, while the bias of \texttt{DU-Shapley} in Theorem \ref{theorem:DU_shapley_error_bound} decreases as the number of players grows. 
Figure \ref{fig:Bias_vs_MC_error} compares (\ref{eq:MC_error_for_I_samples}) with \texttt{DU-Shapley}'s bias \eqref{eq:bound_error} as the number of players increases. 
By considering the same computational budget for both approaches, we can observe that Monte Carlo approximation of the Shapley value is associated with a larger approximation error than \texttt{DU-Shapley}, when the number of players becomes sufficiently large.
This result confirms that \texttt{DU-Shapley} is indeed relevant for dataset valuation in this regime.

\begin{figure}[h]
    \centering
    \includegraphics[scale = 0.3]{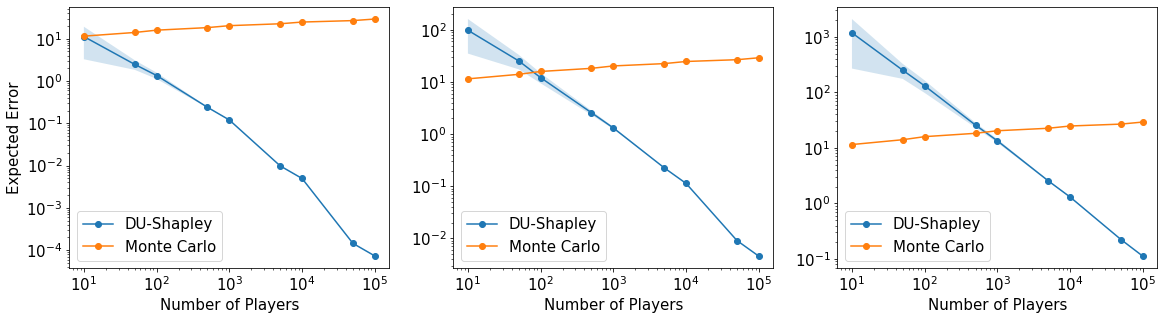}
    \caption{Monte Carlo's expected error for limited sampling budget ($T = I$) versus \texttt{DU-Shapley}'s expected bias for a value function $w(n) = 1 - \frac{10^{k(\mathcal{I})}}{10^{k(\mathcal{I})} + n}$ where $k(\mathcal{I}) = \lfloor \log(n_{\mathcal{I}}) \rfloor - 1$. For each value of $I$, we drew $100$ times the data points of each player from $\mathrm{U}([n_{\mathrm{max}}])$, with (left) $n_{\mathrm{max}} = 10^2$, (center) $n_{\mathrm{max}} = 10^3$, and (right) $n_{\mathrm{max}} = 10^4$.}
    \label{fig:Bias_vs_MC_error}
\end{figure}
\vspace{-0.2cm}
\section{Numerical Experiments}

In this section, we illustrate the benefits of our methodology on several dataset valuation benchmarks associated to both synthetic and real data.
More precisely, we aim at assessing numerically how well the proposed methodology \texttt{DU-Shapley} approximates the Shapley value by considering two experiments.
The first one is associated to synthetic toy datasets.
On the other hand, the second experiment considers classical real-world datasets associated to both classification and regression problems, which have been considered in \citet{JMLR:v23:21-0439}. 

For tractability reasons, it is unfortunately not possible to compute the exact Shapley values when $I$ is large, so we can only measure approximation errors for small values of $I$, say smaller than 20. This case is actually the least favorable for \texttt{DU-Shapley}, whose approximation improves with the number of players. The purpose of this section is  to illustrate that, even in those ``worst-case'' instances for \texttt{DU-Shapley}, the latter performs as good (or even better) that state-of-the-art techniques that would not profit for a larger number of players (as expected in motivating applications).

\subsection{Synthetic Data}
\label{subsec:synthetic}

We consider a toy dataset valuation problem associated to the linear regression problem of our running example.  
The corresponding utility function is defined in \eqref{eq:utility_lin_reg}, and we set $d=10$ and $\sigma_{\varepsilon} = 1$. In order to benchmark the performances of \texttt{DU-Shapley}, we are considering four competitive approaches, relying on Monte Carlo (MC) approximation strategies \citep{JMLR:v23:21-0439}. The first one, referred to as \texttt{MC-Shapley} is the standard MC approximation defined in \eqref{eq:MC}.
The second one, coined \texttt{MC-anti-Shapley} is a variance-reduced version of \texttt{MC-Shapley} that considers antithetic sampling.
The third one coined \texttt{Owen-Shapley} stands for the multilinear extension of \citet{Owen} which represents the Shapley value as two nested expectations.
Finally, the fourth approach, coined \texttt{Orthogonal-Shapley}, relies on efficient permutation sampling techniques on the hypersphere to draw permutations in \eqref{eq:Shapley_def1} in a dependent way \citep{JMLR:v23:21-0439}.
All these approaches are detailed in the supplementary material.
To assess the performance of the aforementioned Shapley value estimators, we used the mean square error (MSE) averaged over all players.
For \texttt{DU-Shapley}, we used \eqref{eq:DU-shapley} and for each MC-based estimator, we performed 25 estimations to compute the MSE, and did it 10 times to obtain confidence intervals for the MSE.

Figure \ref{fig:toy_example} depicts associated results for various dataset size distributions and numbers of players $I \in \{5,10,20\}$; we could not go beyond $I=20$ because the computation of the exact Shapley value needed for the MSE became intractable.
The top row corresponds to a setting where the dataset size is sampled from a uniform distribution $\mathrm{U}(\{10,\ldots,10^3\})$.
We can see that, even for the small numbers of players considered, \texttt{DU-Shapley} provides competitive Shapley value approximations compared to MC-based approaches for an equivalent budget, \emph{i.e.} for a number of value function evaluations equal to $I$ (illustrated by the vertical black line).
On the other hand, the bottom row corresponds to a setting where the dataset size of player $i \in [I]$ is $2^i$, leading to a scenario where the maximum discrepancy between datasets' sizes becomes large.
As highlighted by Theorem \ref{theorem:DU_shapley_error_bound}, this scenario is not favorable to \texttt{DU-Shapley}. 
To address this issue, we propose an enhanced version of the proposed methodology, referred to as \texttt{DU-Shapley++} and defined by
\vspace{-0.2cm}\begin{align}
    \psi_i^{++}(w) = \frac{1}{I}\bigl(w(n_i) + w(n_{\mathcal{I}}) - w(n_{\mathcal{I}\setminus\{i\}})\bigr) + \frac{1}{I}\sum_{k=1}^{I-2} \br{w(k\mu_{-i} + n_i) - w(k\mu_{-i})}\eqsp,
\end{align}\vspace{-0.2cm}

which corresponds to treating the extreme cases separately and approximating the rest of the terms with an expectation associated to a discrete uniform random variable; this technique can be seen as finding a different bias/variance tradeoff in the estimation of the Shapley value.
An illustration of why considering extreme cases might benefit our approximation is given in Figure \ref{fig:approx_uniform} (see left figure). 
Again, we can denote that \texttt{DU-Shapley++} competes with MC-based approaches at iso-budget, even in the unfavorable approximation regime where the number of players is small.

\begin{figure}[h]
    \centering
    \includegraphics[scale = 0.31]{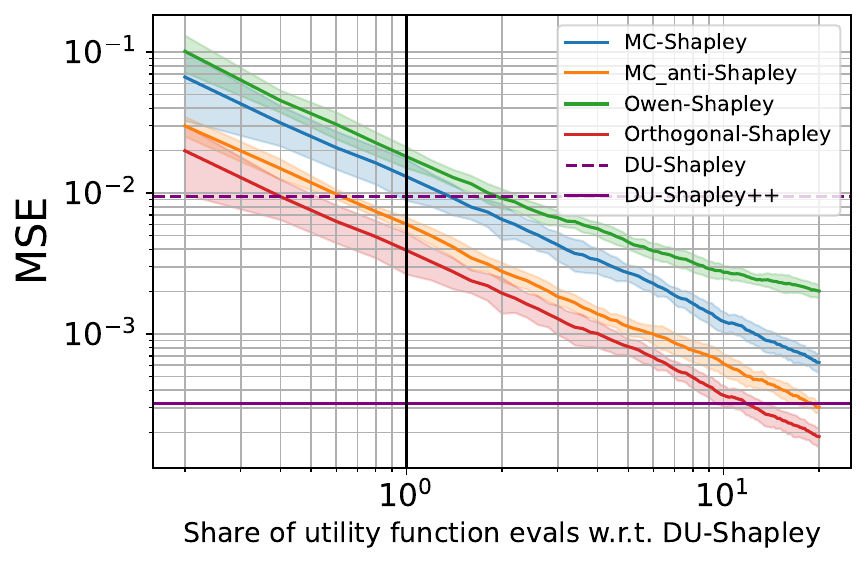}
    \includegraphics[scale = 0.31]{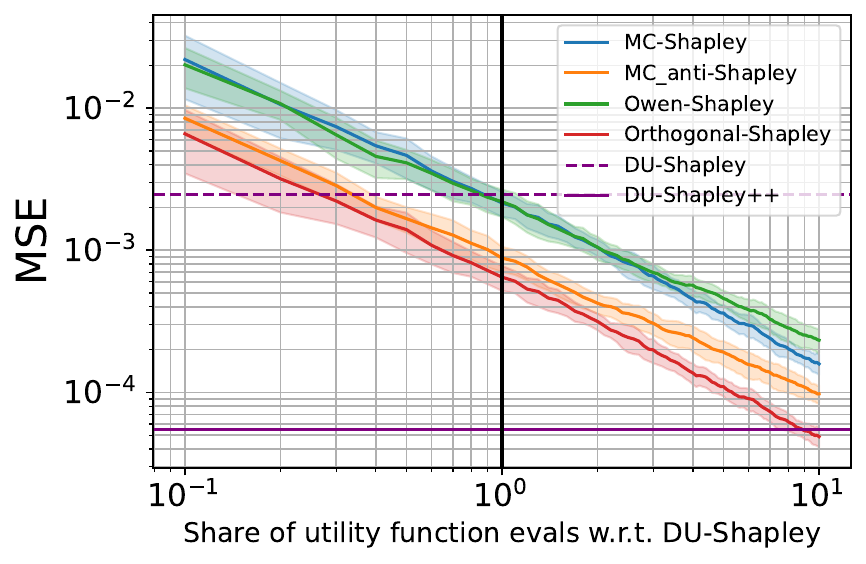}
    \includegraphics[scale = 0.31]{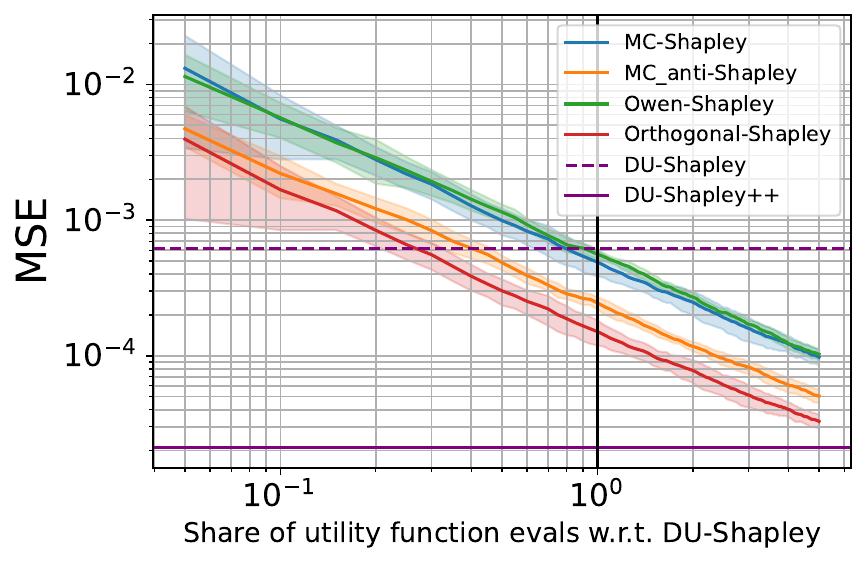}
    \includegraphics[scale = 0.31]{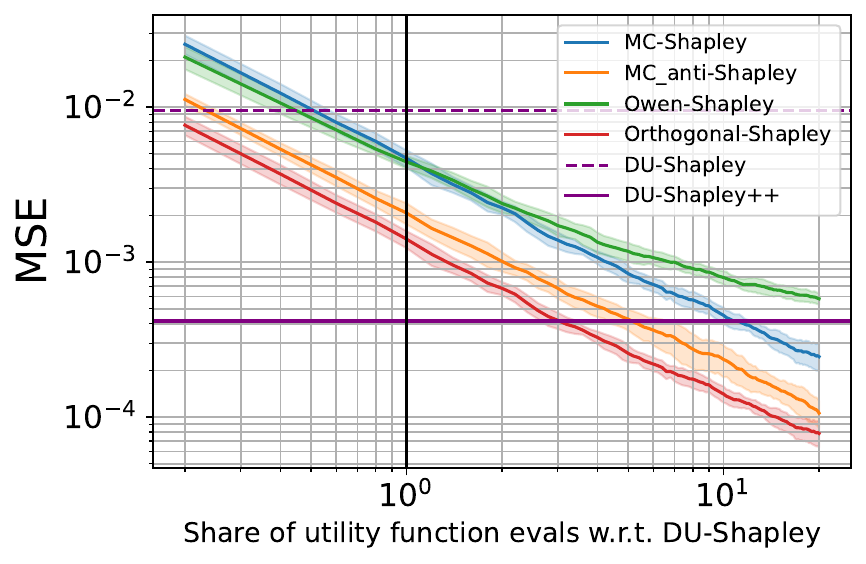}
    \includegraphics[scale = 0.31]{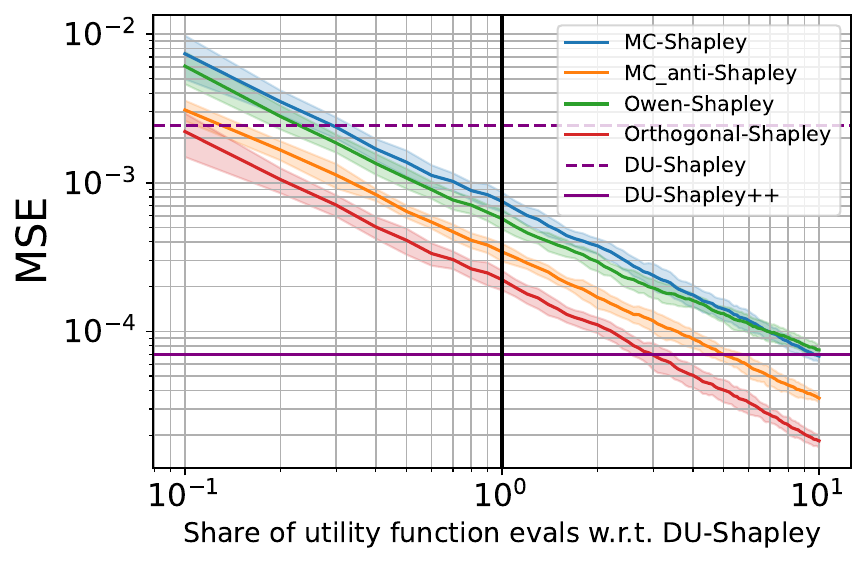}
    \includegraphics[scale = 0.31]{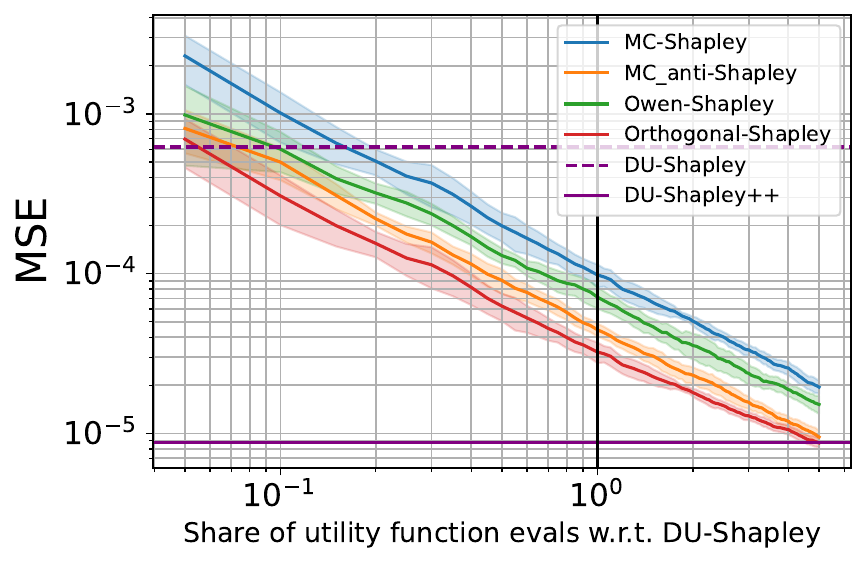}
    \caption{Worst-case comparison between the proposed methodology (constant number of utility function evaluations equal to $I$, illustrated by the vertical black line), and MC-based approximations on synthetic datasets. From left to right, $I = 5$, $I = 10$ and $I = 20$. (top) Dataset size drawn from the Uniform distribution $\mathrm{U}(\{10,\ldots,10^3\})$; (bottom) dataset size of player $i$ set to $2^i$.}
    \label{fig:toy_example}    
\end{figure}
\vspace{-0.2cm}
\subsection{Real-World Data}
\label{subsec:expe2}

In this second experiment, we consider real-world datasets, also considered in \citet{JMLR:v23:21-0439}, and whose details are provided in Table \ref{table:dataset}.
To tackle these problems, we are considering logistic regression models and gradient-boosted decision trees (GBDT).
For classification tasks, the utility function has been taken as the expected accuracy of the trained logistic regression model over a hold-out testing set corresponding to 10\% of the size of the training dataset.
For regression tasks, the utility function corresponds to the averaged MSE over a hold-out testing set corresponding to 10\% of the training dataset. 
For each dataset, we considered two worst-case scenarios for benchmarking the proposed methodology, namely $I=10$ players and $I=20$ players.

Starting from the initial dataset in Table \ref{table:dataset}, we affect a subset of random size to each player.
As in Section \ref{subsec:synthetic}, we consider several competitors leveraging MC approximation strategies. 
Since standard data valuation approaches in the ML literature, such as \texttt{Data Shapley} \citep{pmlr-v97-ghorbani19c},  only consider simple MC sampling based on permutation sampling, we compare ourselves with \texttt{MC-Shapley} and \texttt{MC-anti-Shapley}, which correspond to standard baselines.
To assess the benefits of the proposed methodology, namely \texttt{DU-Shapley} and \texttt{DU-Shapley++}, we compute as in Section \ref{subsec:synthetic} the averaged MSE across all players between the true Shapley value and each estimator. 
In contrast to Section \ref{subsec:synthetic} where we considered an utility function in closed form and not requiring to re-train a ML model, evaluating the two utility functions considered in this experiment requires re-training each ML model.
Regarding the computation of the Shapley value in \eqref{eq:Shapley_def1}, it is clearly not feasible to train a ML model for a large number of epochs.
As such, we chose to restrict ourselves to 20 steps of stochastic gradient descent for logistic regression and 20 boosting iterations for GBDTs.
For MC-based approaches, we only considered $I$ MC samples to compare those approximations with the proposed methodology on a fair basis, \emph{i.e.} associated to the same computational budget.

Table \ref{table:expe2} depicts the results. Again, we  clearly see that even in the worst-case scenario where the number of players is small, \texttt{DU-Shapley++} already competes favorably with MC-based approximations.

\begin{wraptable}{r}{10.5cm}
    \vspace{-1cm}
	\caption{Datasets considered in Section \ref{subsec:expe2}.}
    \vspace{0.3cm}
	\centering%
	{\small
		\begin{tabular}{lccc}%
			\toprule
			Dataset & Size & $d$ & Task\\
			\midrule
            adult \citep{10.5555/3001460.3001502} & 48,842 & 107 & classif.\\
            breast-cancer \citep{10.1287/opre.43.4.570} & 699 & 30 & classif.  \\
            bank \citep{MORO201422} & 45,211 & 16 & classif. \\
            cal-housing \citep{KELLEYPACE1997291} & 20,640 & 8 & regression \\
            make-regression \citep{JMLR:v12:pedregosa11a}& 1,000 & 10 & regression  \\
            year \citep{JMLR:v12:pedregosa11a} & 515,345 & 90 & regression  \\
			\midrule
	\end{tabular}
	}
	\label{table:dataset}
\end{wraptable}

\begin{table}
    \vspace{-1cm}
	\caption{Worst-case comparison between the proposed approach and competitors, for real-world datasets considered in Table \ref{table:dataset}. For each Shapley value approximation, we report the averaged MSE across all players $I$ w.r.t. the exact Shapley value calculated }
	\centering%
	{\small
		\begin{tabular}{lccccc}%
			\toprule
			Dataset & $I$ & \texttt{DU-Shapley} & \texttt{DU-Shapley++} & \texttt{MC-Shapley} & \texttt{MC-anti-Shapley} \\
            \midrule
           adult & 10 & $2.10^{-2}$ &  $4.10^{-3}$ & $9.10^{-3}$ & $7.10^{-3}$ \\
           adult & 20 & $6.10^{-3}$ & $7.10^{-4}$ & $4.10^{-3}$ & $2.10^{-3}$ \\
           \midrule
           breast-cancer & 10 & $3.10^{-2}$ & $5.10^{-3}$ & $3.10^{-2}$ & $1.10^{-2}$\\
           breast-cancer & 20 & $3.10^{-3}$ & $7.10^{-4}$ & $1.10^{-3}$ & $8.10^{-4}$\\
           \midrule
           bank & 10 & $1.10^{-1}$ & $5.10^{-2}$ & $9.10^{-2}$ & $8.10^{-2}$\\
           bank & 20 & $5.10^{-2}$ & $3.10^{-2}$ & $6.10^{-2}$ & $4.10^{-2}$\\
           \midrule
           cal-housing  & 10& $7.10^{-2}$ & $3.10^{-2}$ & $5.10^{-2}$ & $3.10^{-2}$\\
            cal-housing  & 20& $2.10^{-2}$ & $8.10^{-3}$ & $2.10^{-2}$ & $1.10^{-2}$\\
            \midrule
           make-regression & 10& $5.10^{-1}$ & $1.10^{-1}$ & $4.10^{-1}$ & $4.10^{-1}$\\
           make-regression & 20& $3.10^{-1}$ & $8.10^{-2}$ & $3.10^{-1}$ & $2.10^{-1}$\\
           \midrule
           year & 10& $6.10^{-3}$ & $1.10^{-3}$ & $5.10^{-3}$ & $5.10^{-3}$\\
           year & 20 & $2.10^{-3}$ & $9.10^{-4}$ & $1.10^{-3}$ & $1.10^{-3}$\\
        \midrule
	\end{tabular}
	}
	\label{table:expe2}
    \vspace{-0.4cm}
\end{table}
\vspace{-0.3cm}

\section{Conclusion}

We proposed a general dataset valuation methodology based on the Shapley value, by exploiting the underlying structure of the utility function.
The proposed framework, referred to as \texttt{DU-Shapley} and \texttt{DU-Shapley++}, allows for efficient dataset valuation, especially in the usual setting where many data owners are willing to collaborate by sharing their data.
In addition, we have shown that \texttt{DU-Shapley} has favorable convergence properties via both asymptotic and non-asymptotic results; justifying and illustrating those claims on a standard running example.
Interestingly, our numerical experiments showcases that the proposed methodology also competes with other state-of-the-art Shapley value approximations when the number of data owners is small; a regime where the bias of our approximation does not vanish.
Finally, some limitations associated to the proposed methodology pave the way for more advanced dataset valuation techniques.
As an example, we could extend \texttt{DU-Shapley} to heterogeneous data scenarii where both the quantity, quality and types of data points brought by players play a role in the utility function. 
This envisioned setting would imply, in particular, to consider multi-dimensional utility functions  taking into account the size of the dataset and key parameters associated to their local distributions.
As a seminal example, this generalisation of our methodology would allow to provide data sharing incentives under the federated learning paradigm, where local data distributions of players could widely differ.

\newpage

\bibliography{biblio}
\bibliographystyle{plainnat}

\clearpage
\newpage

{\begin{center}\Large\textbf{SUPPLEMENTARY MATERIAL \\ -- \\ DU-Shapley: A Shapley Value Proxy for Efficient Dataset Valuation}\end{center}}\vspace{1cm}

\paragraph{Notations and conventions.}

 We denote by $\mathcal{B}\parentheseLigne{\mathbb{R}^d}$ the Borel $\sigma$-field of $\mathbb{R}^d$, $\mathbb{M}\parentheseLigne{\mathbb{R}^d}$ the set of all Borel measurable functions $f$ on $\mathbb{R}^d$ and $\norm{\cdot}$ the Euclidean norm on $\mathbb{R}^d$.
For the sake of simplicity, with little abuse, we shall use the same notations for
a probability distribution and its associated probability density function.
For $n \ge 1$, we refer to the set of integers between $1$ and $n$ with the notation $[n]$.
The $d$-multidimensional Gaussian probability distribution with mean $\mu \in \Rd$ and covariance matrix $\Sigma \in \mathbb{R}^{d \times d}$ is denoted by $\gauss\parentheseLigne{\mu,\Sigma}$.
Equations of the form (1) (resp. (S1)) refer to equations in the main paper (resp. in the supplement).

\appendix
\setcounter{theorem}{0}
\setcounter{proposition}{0}
\setcounter{lemma}{0}
\setcounter{remark}{0}

\addcontentsline{toc}{section}{} 
\part{} 
\parttoc 

\newtheorem{unlemma}{Lemma S}
\newtheorem{unproposition}{Proposition S}
\newtheorem{uncorollary}{Corollary S}
\newtheorem{untheorem}{Theorem S}

\setcounter{equation}{0}
\setcounter{figure}{0}
\setcounter{table}{0}
\setcounter{assumption}{0}
\makeatletter
\renewcommand{\theequation}{S\arabic{equation}}
\renewcommand{\thefigure}{S\arabic{figure}}
\renewcommand{\thetheorem}{S\arabic{theorem}}
\renewcommand{\thelemma}{S\arabic{lemma}}
\renewcommand{\thetable}{S\arabic{table}}
\renewcommand{\thesection}{S\arabic{section}}
\renewcommand{\theremark}{S\arabic{remark}}
\renewcommand{\theproposition}{S\arabic{proposition}}
\renewcommand{\thecorollary}{S\arabic{corollary}}


\section{Additional details regarding the running example}

In this section, we provide additional details regarding the running example which has been considered in the main paper, see Section \ref{sec:preliminaries}.
Note that we are considering a more general setting than in the main paper, which obviously encompasses the specific instance associated to the running example.

\subsection{Problem formulation}
\label{subsec:problem_formulation}

\noindent \textbf{Context.} For the sake of completeness, we first briefly recall the problem we are considering. 
We are investigating a framework involving $I \in \N^*$ players respectively owning a local dataset $\mathrm{D}_i= \{x_i^{(j)},y_i^{(j)}\}_{j=1}^{n_i}$ of size $n_i = |\mathrm{D}_i|$ where for any $j \in [n_i]$, $x_i^{(j)} \in \mathsf{X} \subseteq \R^d$ and $y_i^{(j)} \in \mathsf{Y} \subseteq \R$. 

For any $i \in [I]$, we consider the following generative linear model regarding the dataset $\mathrm{D}_i$ owner by the $i$-th player:
\begin{align}
    &Y_i = X_i\theta + \eta_i\eqsp, \ \eta_i \sim \mathrm{N}(0_{n_i},\varepsilon_i^2\mathrm{I}_{n_i})\eqsp,\label{eq:likelihood_linear}\\
    &x_i^{(j)} \sim p_X\eqsp, \forall j \in [n_i] \eqsp,\\
    &\varepsilon_i \sim p_\varepsilon\eqsp,
\end{align}
where $X_i \in \R^{n_i \times d}$ is defined by $X_i = ([x_i^{(1)}]^\top,\ldots,[x_{i}^{(n_i)}]^\top)^\top$ and $Y_i \in \R^{n_i}$ is defined by $Y_i = (y_i^{(1)},\ldots,y_i^{(n_i)})^\top$.
For the sake of simplicity, note that we do not consider further sources of heterogeneity across players such as feature spaces with heterogeneous semantic and dimension (\emph{e.g.} one player having image features and another one text features).

\noindent \textbf{Loss function.} Based on the data brought by each player, we are interested in learning a linear prediction function $g_\theta: x \mapsto x^\top \theta$, where $\theta \in \Theta  \subseteq \R^d$ stands for a weight vector. This boils down to finding an estimator $\hat{\theta}$ of $\theta$ based on $\{\mathrm{D}_i\}_{i \in [I]}$.
Without loss of generality, we assume that $\hat{\theta}$ is found by minimising a weighted sum of empirical risk functions given by
    \begin{equation}
        F(\theta) = \lambda h(\theta) + \sum_{i=1}^I \alpha_i f_i(\theta)\eqsp, \label{eq:F}
    \end{equation}
where, $\lambda \geq 0$, $h:\Theta \rightarrow \R$ is a regularisation term and for any $i \in [I]$, $\alpha_i \in [0,1]$ are weights such that $\sum_{i \in [I]} \alpha_i = 1$ and $f_i: \Theta \rightarrow \R$ only depends on $\mathrm{D}_i$.\\

We define  
\begin{align}
    \sigma^2_{\varepsilon} &= \int_{\mathbb{R_+}}\varepsilon^2 \ p_\varepsilon(\dd \varepsilon)\eqsp, \label{eq:mu_e}\\
    X &= [(\alpha_1/n_1)^{\half} X_1^\top,\ldots,(\alpha_I/n_I)^{\half} X_I^\top]^\top\eqsp,\label{eq:X}\\
    Y &= [(\alpha_1/n_1)^{\half} Y_1^\top,\ldots,(\alpha_I/n_I)^{\half} Y_I^\top]^\top \label{eq:Y}\eqsp.
\end{align}

The following proposition provides a closed-form expression of the minimum of $F$ in \eqref{eq:F} denoted by $\theta_{1:I}^\star$ when the functions $\{f_i\}_{i \in [I]}$ are chosen to be quadratic.

\begin{proposition}
    For any $i \in [I]$ and $\theta \in \R^d$, let $f_i(\theta) = (1/n_i) \|Y_i - X_i \theta\|^2$ and $h(\theta) = (1/2)\|\theta\|^2$.
    Then, the global minimiser of $F$ defined in \eqref{eq:F} writes
    \begin{align}
        \theta_{1:I}^\star &= \pr{\lambda \mathrm{I}_d + X^\top X}^{-1} X^\top Y\eqsp, \text{ for } \lambda > 0\eqsp,\\\label{eq:MLE_multiple}
        \theta_{1:I}^\star &= X^\dagger Y\eqsp, \text{ for } \lambda = 0\eqsp,
    \end{align}
    where $\{Y_i,X_i\}_{i \in [I]}$ are defined in \eqref{eq:likelihood_linear}, $X$ in \eqref{eq:X}, $Y$ in \eqref{eq:Y}, and $X^\dagger$ refers to the Moore-Penrose inverse of the matrix $X$.
\end{proposition}
\begin{proof}
    As a sum of differentiable functions, $F$ in \eqref{eq:F} is differentiable and its gradient writes for any $\theta \in \R^d$,
    $$
    \nabla F(\theta) = \lambda \theta + \sum_{i=1}^I \frac{\alpha_i}{n_i} X_i^\top (X_i \theta - Y_i) \eqsp.
    $$
    The proof is concluded using the fact that $F$ is strongly convex and by using the first-order guarantee $\nabla F(\theta_{1:I}^\star) = 0_d$.
\end{proof}

\begin{remark} 
    Let $n = \sum_{i \in [I]} n_i$ and consider $Y \in \R^n$ and $X \in \mathbb{R}^{n \times d}$ the vertical concatenations of the $I$ datasets owned by each agent, respectively.
    Then, defining $F(\theta) = (1/n)\|Y - X\theta\|^2$ is equivalent to set for any $i \in [I]$, $\alpha_i = n_i/n$ and $f_i(\theta) = (1/n_i)\|Y_i - X_i \theta\|^2$.
\end{remark}

\subsection{Technical lemmata}

We have the following result.

\begin{lemma}\label{theo:MSE_multiple_player}
    Let $\lambda \geq 0$, $\pi$ be a probability distribution defined on $(\R^d,\mathcal{B}(\R^d))$, and $C = \int_{\mathbb{R}^d}x x^\top \pi(\mathrm{d}x)$.
    In addition, let $\hat{\theta} = \theta_{1:I}^\star$ defined in \eqref{eq:MLE_multiple} and define $X_\lambda = \lambda \mathrm{I}_d + \sum_{i=1}^I \frac{\alpha_i}{n_i} X_i^\top X_i$ with $X_i$ defined in \eqref{eq:likelihood_linear}. 
    If $\lambda > 0$, then $X_\lambda$ is invertible and $X_\lambda^{-1}$ exists with probability $1$. 
    On the other hand, if $\lambda = 0$, we additionally assume that $X_\lambda$ is positive definite with probability $1$. 
    Under these assumptions, consider
    \begin{align}
        E_{1:I} 
        &= \int \br{\mathrm{I}_d - \pr{\sum_{i=1}^I \frac{\alpha_i}{n_i} X_i^\top X_i} \pr{X_\lambda}^{-1}} \theta  \theta^\top \br{ \mathrm{I}_d - \pr{X_\lambda}^{-1} \sum_{i=1}^I \frac{\alpha_i}{n_i} X_i^\top X_i}   \otimes_{i=1}^I \otimes_{j=1}^{n_i} p_X(\mathrm{d}x_i^{(j)}) \\
        &+\sigma^2_{\varepsilon}\int \pr{X_\lambda}^{-1}\pr{\sum_{i=1}^I\frac{\alpha_i^2}{n_i^2} X_i^\top X_i}\pr{X_\lambda}^{-1} \otimes_{i=1}^I \otimes_{j=1}^{n_i} p_X(\mathrm{d}x_i^{(j)})\eqsp,
    \end{align}
    where $\sigma^2_{\varepsilon}$ is defined in \eqref{eq:mu_e}.
    Then, for any $i \in [I]$ and $x \sim \pi$, we have
    \begin{equation}
      \mathbb{E}\br{\pr{x^\top \hat{\theta} - x^\top \theta}^2} =  \mathrm{Tr}\br{C \cdot  E_{1:I}}\eqsp,\label{eq:theta_star}
    \end{equation}
    where the expectation is taken over the randomness of $x$ and $\hat{\theta}$.
\end{lemma}
\begin{proof}
    Let $\lambda \geq 0$, $i \in [I]$ and $x \sim \pi$.
    Using \eqref{eq:likelihood_linear} and \eqref{eq:MLE_multiple}, notice that 
    \begin{align}
     x^\top\theta - x^\top \theta_{1:I}^\star
     &= x^\top \br{\mathrm{I}_d - \pr{\lambda\mathrm{I}_d + \sum_{i=1}^I \frac{\alpha_i}{n_i} X_i^\top X_i}^{-1}\sum_{i=1}^I \frac{\alpha_i}{n_i} X_i^\top X_i } \theta \\
     &- x^\top \pr{\lambda \mathrm{I}_d + \sum_{i=1}^I \frac{\alpha_i}{n_i} X_i^\top X_i}^{-1} \sum_{i=1}^I \frac{\alpha_i}{n_i} X_i^\top\eta_i \eqsp.
    \end{align}

By using the notation
$$
X_\lambda= \lambda \mathrm{I}_d + \sum_{i=1}^I \frac{\alpha_i}{n_i} X_i^\top X_i\eqsp,
$$
the previous element is a scalar, therefore,
\begin{align}
    &(x^\top\theta - x^\top \theta_{1:I}^\star)^2 \\
    &= \theta^\top \br{\mathrm{I}_d - \pr{\sum_{i=1}^I \frac{\alpha_i}{n_i} X_i^\top X_i} \pr{X_\lambda}^{-1}} x  x^\top \br{ \mathrm{I}_d - \pr{X_\lambda}^{-1} \sum_{i=1}^I \frac{\alpha_i}{n_i} X_i^\top X_i} \theta \label{eq:1st_term}\\ 
    &- 2\theta^\top \br{\mathrm{I}_d - \pr{\sum_{i=1}^I \frac{\alpha_i}{n_i} X_i^\top X_i} \pr{X_\lambda}^{-1}} x  x^\top \pr{X_\lambda}^{-1} \sum_{i=1}^I \frac{\alpha_i}{n_i} X_i^\top \eta_i \label{eq:2nd_term}\\ 
    &+ \pr{\sum_{i=1}^I \frac{\alpha_i}{n_i} \eta_i^\top X_i} \pr{X_\lambda}^{-1} x  x^\top\pr{X_\lambda}^{-1} \sum_{i=1}^I \frac{\alpha_i}{n_i} X_i^\top \eta_i \eqsp, \label{eq:third_term}
\end{align}
where we have used that $X_\lambda$ is a symmetric matrix, so its inverse is symmetric as well. 

We now focus on the third term in the previous equality, namely  \eqref{eq:third_term}. 
Taking the trace operator and using cyclic permutations, we have 
\begin{align}
    &\pr{\sum_{i=1}^I \frac{\alpha_i}{n_i} \eta_i^\top X_i} \pr{X_\lambda}^{-1} x  x^\top\pr{X_\lambda}^{-1} \sum_{i=1}^I \frac{\alpha_i}{n_i} X_i^\top \eta_i \\
    &= \mathrm{Tr} \left[\pr{\sum_{i=1}^I \frac{\alpha_i}{n_i} \eta_i^\top X_i} \pr{X_\lambda}^{-1} x  x^\top\pr{X_\lambda}^{-1} \sum_{i=1}^I \frac{\alpha_i}{n_i} X_i^\top \eta_i\right]\\
    &= \mathrm{Tr} \left[ x x^\top \pr{X_\lambda}^{-1}\pr{\sum_{i=1}^I\frac{\alpha_i}{n_i} X_i^\top \eta_i}\pr{\sum_{i=1}^I \frac{\alpha_i}{n_i} \eta_i^\top X_i}\pr{X_\lambda}^{-1}\right]\\
    &= \mathrm{Tr} \left[ x x^\top \pr{X_\lambda}^{-1}\pr{\sum_{i=1}^I\frac{\alpha_i^2}{n_i^2} X_i^\top \eta_i \eta_i^\top X_i  + \sum_{i \neq j}\frac{\alpha_i}{n_i} \frac{\alpha_j}{n_j} X_i^\top \eta_i \eta_j^\top X_j}\pr{X_\lambda}^{-1}\right] \eqsp.    
\end{align}
For any $k \in [n_i]$ and $l \in [n_j]$, notice that the $(k,l)$ entry of the $n_i \times n_j$ matrix $\eta_i\eta_j^\top$ is $\eta_i^{(k)}\eta_j^{(\ell)}$ where $\eta_i^{(k)} \sim \mathrm{N}(0,\varepsilon_i^2)$ and $\eta_j^{(l)} \sim \mathrm{N}(0,\varepsilon_j^2)$. 
Therefore, for any $(k,l) \in [n_i] \times [n_j]$, $\mathbb{E}[\eta_i^{(k)}] = \mathbb{E}[\eta_j^{(l)}] = 0$, $\mathbb{E}[(\eta_i^{(k)})^2] = \varepsilon_i^2$ and $\mathbb{E}[(\eta_j^{(l)})^2] = \varepsilon_i^2$. 
It follows, for any $i,j \in [I]$ such that $i \neq j$, that $\eta_i\eta_j^\top$ has expected value equal to $0$. 
By denoting, for any $i \in [I]$, $p_\eta^{(i)} = \mathrm{N}(\eta_i ; 0_{n_i},\varepsilon_i^2\mathrm{I}_{n_i})$, we obtain 
\begin{align}
     &\int_{\R^{n_1} \times \ldots \R^{n_I}}\pr{\sum_{i=1}^I \frac{\alpha_i}{n_i} \eta_i^\top X_i} \pr{X_\lambda}^{-1} x  x^\top\pr{X_\lambda}^{-1} \pr{\sum_{i=1}^I \frac{\alpha_i}{n_i} X_i^\top \eta_i} \otimes_{i \in [I]} p_\eta^{(i)}(\dd \eta_i) \\ &= \mathrm{Tr} \left[ x x^\top \pr{X_\lambda}^{-1}\pr{\sum_{i=1}^I\frac{\alpha_i^2}{n_i^2} \varepsilon_i^2 X_i^\top X_i}\pr{X_\lambda}^{-1}\right]\eqsp.  
\end{align}
In addition, we have
\begin{align}
     &\int_{\R_+^I}\br{\int_{\R^{n_1} \times \ldots \R^{n_I}}\pr{\sum_{i=1}^I \frac{\alpha_i}{n_i} \eta_i^\top X_i} \pr{X_\lambda}^{-1} x  x^\top\pr{X_\lambda}^{-1} \pr{\sum_{i=1}^I \frac{\alpha_i}{n_i} X_i^\top \eta_i} \otimes_{i \in [I]} p_\eta^{(i)}(\dd \eta_i)} \otimes_{i \in [I]} p_\varepsilon(\dd \varepsilon_i) \\ 
     &= \sigma^2_{\varepsilon}\mathrm{Tr} \left[ x x^\top \pr{X_\lambda}^{-1}\pr{\sum_{i=1}^I\frac{\alpha_i^2}{n_i^2} X_i^\top X_i}\pr{X_\lambda}^{-1}\right]\eqsp,
\end{align}
where $\sigma^2_{\varepsilon}$ is defined in \eqref{eq:mu_e}. 

For \eqref{eq:1st_term} and \eqref{eq:2nd_term}, we follow similar steps.
The proof is concluded by integrating against the probability measures of $x$ and $\{x_i^{(j)} ; j \in [n_i]\}_{i \in [I]}$, namely $\pi$ and $p_X$.
\end{proof}

\begin{remark}
    Note that when $\lambda \rightarrow \infty$, we have
    $$
    \lim_{\lambda \rightarrow \infty} \mathbb{E}\br{\pr{x^\top \hat{\theta} - x^\top \theta}^2} = \mathrm{Tr}\br{C\cdot  \theta \theta^\top}\eqsp.
    $$
\end{remark}

\begin{lemma}
    \label{lem:1}
    Let $\pi$ be a probability distribution defined on $(\R^d,\mathcal{B}(\R^d))$, and let $C = \int_{\mathbb{R}^d}x x^\top \mathrm{d}\pi(x)$.
    Set $\lambda = 0$, let $\hat{\theta} = \theta_{1:I}^\star$ defined in \eqref{eq:theta_star} and assume for any $i \in [I]$ that $\alpha_i = n_i/n$.
    In addition, suppose that $\sum_{i=1}^I X_i^\top X_i$ is positive definite with probability $1$ with $X_i$ defined in \eqref{eq:likelihood_linear}. 
    Then, for any $i \in [I]$ and $x \sim \pi$, we have
    \begin{equation}
      \mathbb{E}\br{\pr{x^\top \hat{\theta} - x^\top \theta}^2} = \sigma^2_{\varepsilon} \cdot  \mathrm{Tr}\br{C \cdot  \int \pr{\sum_{i=1}^I \sum_{j=1}^{n_i} x_i^{(j)} [x_i^{(j)}]^\top}^{-1} \otimes_{i=1}^I\otimes_{j=1}^{n_i} p_X(\mathrm{d}x_i^{(j)})}\eqsp.\label{eq:theta_star}
    \end{equation}
\end{lemma}
\begin{proof}
    The proof directly follows from \Cref{theo:MSE_multiple_player}.
    A similar proof was presented in \citet{Donahue_Kleinberg_AAAI_2021}.
\end{proof}

\subsection{Proof of Proposition \ref{lemma:MSE_homogeneous_case}}

For some particular choices of the probability measures $p_X$ and $\pi$, the following results show that we can end up with closed-form expressions for the expected mean square error defined in \eqref{eq:theta_star}.

\begin{proposition}
    \label{lemma:MSE_homogeneous_case_sup}
     Let $\pi$ be a probability distribution defined on $(\R^d,\mathcal{B}(\R^d))$, and define $C = \int_{\mathbb{R}^d}x x^\top \mathrm{d}\pi(x)$.
    In addition, set $\lambda = 0$, let $\hat{\theta} = \theta_{1:I}^\star$ defined in \eqref{eq:MLE_multiple} and assume for any $i \in [I]$ that $\alpha_i = n_i/n$ and $p_X = \mathrm{N}(0_d, \Sigma)$ with $\Sigma \in \mathbb{R}^{d \times d}$ a positive definite matrix. 
    Then, for any $i \in [I]$ and $x \sim \pi$, we have for $n_{1:I} > d + 1$,
    \begin{equation}
      \mathbb{E}\br{\pr{x^\top \hat{\theta} - x^\top \theta}^2} = \frac{\sigma^2_\varepsilon}{n_{1:I} -d -1}\mathrm{Tr}\br{C\cdot\Sigma^{-1}}\eqsp,\label{eq:closed_form_homogeneous}
    \end{equation}
    where $n_{1:I} = \sum_{i=1}^I n_i$. For the specific choice $\pi = \mathrm{N}(0_d,\Sigma)$ then, for $n_{1:I} > d + 1$,  we have,
    \begin{equation}
      \mathbb{E}\br{\pr{x^\top \hat{\theta} - x^\top \theta}^2} = \frac{d\sigma^2_\varepsilon}{n_{1:I} -d -1}\eqsp.\label{eq:closed_form_homogeneous_2}
    \end{equation}
\end{proposition}

\begin{proof}

By definition of the Wishart probability distribution, we have $X_i^\top X_i \sim \mathrm{Wishart}(\Sigma, n_i)$.
Therefore, it follows that 
$X_{1:I} \sim \mathrm{Wishart}(\Sigma, n_{1:I})$.
Since $n_{1:I} \geq d$ and $\Sigma$ is invertible, then $X_{1:I}^{-1}$ exists with probability 1 and $X_{1:I}^{-1} \sim \mathrm{Inverse Wishart}$ $(\Sigma^{-1}, n_{1:I})$. Moreover, we have
\begin{align}
\mathbb{E}\br{\pr{\sum_{i \in [I]} X_i^\top X_i}^{-1}} = \frac{\Sigma^{-1}} {n_{1:I}- d - 1}\eqsp,
\end{align}
which concludes the proof by plugging this result in Lemma \ref{lem:1}, and using the fact that $C = \Sigma$ and Tr$(\mathrm{I}_d) = d$.
\end{proof}

\section{Proof of our main results}

In this section, we prove the major results stated in the main paper, namely Theorems \ref{theorem:convergence_uniform} and \ref{theorem:DU_shapley_error_bound}.

\subsection{Proof of Theorem \ref{theorem:convergence_uniform}}

For the sake of completeness, we recall below the full statement of Theorem \ref{theorem:convergence_uniform}.

\textbf{Theorem} \ref{theorem:convergence_uniform}.
\textit{Let $n = \{n_i\}_{i\in [I]}$ be a sequence of positive integer numbers such that the following limits exist, 
$$\lim_{I\to \infty}\frac{1}{I}\sum_{j=1}^{I}n_j= \mu \text{ and } \lim_{I\to \infty}\frac{1}{I}\sum_{j=1}^I (n_j -\mu)^2=\sigma^2\eqsp,$$
where $\mu, \sigma > 0$.
For any $i \in \mathcal{I}$, let $K \sim \mathrm{U}(\{0,\ldots,I-1\})$ and $\mathcal{S}_{K}^{(i)} \sim \mathrm{U}([2^{\mathcal{I}\setminus\{i\}}_{K}])$; and define $n_{\mathcal{S}_K^{(i)}} = \sum_{j \in \mathcal{S}_K^{(i)}} n_j$ to be the random variable corresponding to the total number of data points brought by the random set $\mathcal{S}_K^{(i)}$ of $K$ players. 
Then, for any $i \in \mathcal{I}$, it holds,
\begin{align}
\frac{n_{\mathcal{S}_{K}^{(i)}}}{\sum_{j \in \mathcal{I} \setminus \{i\}}n_j} \xrightarrow{I \to \infty} \rm{U}([0,1])\eqsp, \quad \text{almost surely} \eqsp.
\end{align}    
}

\begin{proof}
Introduce, for any $t,t_0 \in (0,1)$ and any $s \bi 0$, 
\begin{align}
    \mu_{-i}(I) &= \frac{1}{I}\sum_{j \in \mathcal{I} \setminus \{i\}}n_j,\\
    Y(t,I) &= n_{\mathcal{S}_{\lfloor It\rfloor}^{(i)}}, \\
    R^\star(I,t_0,s) &= \mathbb{P}\biggl( \sup_{t>t_0} \biggl| \frac{Y(t,I)}{\lfloor I t\rfloor} - \mu_{-i}(I)\biggr| > s\biggr).
\end{align}
By construction, $Y(t,I)$ is the sum of a sampling without replacement of $\lfloor It\rfloor$ elements in $\{n_j : j \in \mathcal{I}\setminus \{i\} \}$. Therefore, by Corollary 1.3 in \cite{serfling1974probability}, for $s$ fixed, there exists $I_0 \in \mathbb{N}$ such that,
\begin{align}
    R^\star(I,t_0,s)\leq \frac{(1-t_0)\sigma^2}{\lfloor It_0\rfloor s^2} + \varepsilon, \forall I\geq I_0.
\end{align}
Hence, for $I \geq I_0$ large enough, 
\begin{align}
    R^\star(I,t_0,s)\leq 2\varepsilon,
\end{align}
and then, almost surely,
\begin{align}
 \lim_{I\to +\infty}  \frac{Y(t,I)}{\sum_{j \in \mathcal{I} \setminus \{i\}}n_j} = t. 
\end{align}
We conclude remarking that $K = \lfloor I U\rfloor$ with $U\sim \rm{U}([0,1])$.
\end{proof}

\subsection{Proof of Theorem \ref{theorem:DU_shapley_error_bound}}
To prove Theorem \ref{theorem:DU_shapley_error_bound}, we need two preliminary results: Lemma \ref{lemma:DU_Shapley_approx_error_bound}, which itself needs two supplementary results (Lemmas \ref{theorem:Hoeffing} and \ref{lemma:pre_error_bound_result}), and Lemma \ref{lemma:Assumption_on_second_derivative}, which is directly proved. 

Lemma \ref{lemma:Assumption_on_second_derivative} deduces the existence of the constant $\rho$ stated in Theorem \ref{theorem:DU_shapley_error_bound} from \textbf{H}-\ref{ass:2}. 
Lemma \ref{lemma:DU_Shapley_approx_error_bound} bounds the expected difference between $w(n_{\mathcal{S}_{K}^{(i)}})$ and $w(K\mu_{-i})$.

\subsubsection{Technical lemmata}

\begin{lemma}
\label{theorem:Hoeffing}
Consider a set of $I$ values $N = \{n_1, \ldots, n_I\}$. Let $X_1, \ldots, X_k$ and $Y_1, \ldots, Y_k$ denote, respectively, $k$ random samples with and without replacement from $N$. For any continuous and convex function $f$, it follows,
\begin{align}
\mathbb{E}\biggl[f\biggl(\sum_{i = 1}^k Y_i\biggr)\biggr] \leq \mathbb{E}\biggl[f\biggl(\sum_{i = 1}^k X_i\biggr)\biggr] 
\end{align}
\end{lemma}
\begin{proof}
    The proof follows from \cite{doi:10.1080/01621459.1963.10500830}.
\end{proof}

\begin{lemma}\label{lemma:pre_error_bound_result}
Let $I \in \N$, $N:= \{n_1,\ldots ,n_I\}\in\R_+^I$, $\mu=\frac{1}{I}\sum_{i=1}^I n_i$ be their mean value and $\sigma^2=\frac{1}{I}\sum_{i=1}^I (n_i-\mu)^2$ be their variance. 
For $k \in \{0,\ldots,n\}$, let $\mathcal{S}_k \sim U(\{S_k \subseteq [I]: |S_k| = k\})$ be a uniform random variable on the subsets of $\{1,\ldots,I\}$ of size $k$, and $n_{\mathcal{S}_k} =\sum_{i\in\mathcal{S}_k} n_i$ be the random variable defined by the sum of the elements of $\mathcal{S}_k$.
Let $K \sim \mathrm{U}(\{0,\ldots,I\})$ and define $\mathbf{Y} = n_{\mathcal{S}_K}$. Then, 
\begin{align}
&\E [\Yy - \mu K \mid K = k] = 0,\label{eq:Y_and_muK_have_the_same_exp_value}\\
&\E \bigl[\left( \Yy - \mu K\right)^2\mid K = k\bigr]\leq k\sigma^2.\label{eq:upper_bound_variance_Y_and_muK}
\end{align}
\end{lemma}

\begin{proof}
We prove \eqref{eq:Y_and_muK_have_the_same_exp_value} directly. Notice that,
\begin{align}
\E[\Yy \mid \Kk = k] &= \sum_{S_k \subseteq [I]:  |S_k| = k} n_{S_k} \frac{1}{\binom{I}{k}} = \frac{1}{\binom{I}{k}} \sum_{S_k \subseteq [I]: |S_k| = k} \sum_{i \in S_k} n_i\\
&= \frac{1}{\binom{I}{k}} \sum_{i \in [I]} \sum_{\substack{S_k \subseteq [I] : |S_k| = k \\i \in S_k}} n_i = \frac{1}{\binom{I}{k}} \sum_{i \in [I]} n_i \binom{I-1}{k -1}\\
&= \frac{(I-1)!}{(k-1)!(I-k)!}\cdot\frac{(I-k)!k!}{I!} \sum_{i \in [I]} n_i = \mu k. 
\end{align}
Thus, \eqref{eq:Y_and_muK_have_the_same_exp_value} follows as $\E[\mu\Kk \mid \Kk = k ] = \mu k$.
To prove \eqref{eq:upper_bound_variance_Y_and_muK}, let $(\Xx_i)_{i=1}^k$ be $k$ independent samples from the set $N$. From Lemma~\ref{theorem:Hoeffing} it holds,
\begin{align}
\E \bigl[\left( \Yy - \mu \Kk\right)^2\mid \Kk = k\bigr] \leq \E \biggl[\bigl(\mu\Kk-\sum_{i=1}^{\Kk}\Xx_i\bigr)^2 \mid \Kk = k\biggr] = \E \biggl[\biggl(\sum_{i=1}^{\Kk} \left(\mu- \Xx_i\right)\biggr)^2 \mid \Kk = k\biggr].
\end{align}
Therefore,
\begin{align}
\E \bigl[\left( \Yy - \mu \Kk\right)^2\mid \Kk = k\bigr] &\leq \E \biggl[\biggl(\sum_{i=1}^{\Kk}\sum_{j=1}^{\Kk} \left(\mu- \Xx_i\right)\left(\mu- \Xx_j\right)\biggr) \mid \Kk = k\biggr]\\
&= \E \biggl[\biggl(\sum_{i=1}^{\Kk}\sum_{j=1}^{\Kk} \left(\mu^2 - \mu(\Xx_i + \Xx_j) + \Xx_i\Xx_j\right)\biggr) \mid \Kk = k\biggr]\\
&= \sum_{i=1}^k\sum_{j=1}^k \left(\mu^2 - \mu(\E[\Xx_i\mid \Kk = k] + \E[\Xx_j\mid \Kk = k]) + \E[\Xx_i\Xx_j\mid \Kk = k]\right)\\
&= \sum_{i=1}^k\sum_{j=1}^k \left(\mu^2 - \mu(\E[\Xx_i] + \E[\Xx_j]) + \E[\Xx_i\Xx_j]\right)\\
&= \sum_{i=1}^k \left(\mu^2 - 2\mu\E[\Xx_i] + \E[\Xx_i^2]\right) \\
&\quad + \sum_{i=1}^k\sum_{\substack{j=1\\ j\neq i}}^k \left(\mu^2 - \mu(\E[\Xx_i] + \E[\Xx_j]) + \E[\Xx_i]\E[\Xx_j]\right)\\
&= \sum_{i=1}^k \E\bigl[\left(\mu - \Xx_i\right)^2\bigr] + \sum_{i=1}^k\sum_{\substack{j=1\\ j\neq i}}^k \left(\mu^2 - 2\mu^2 + \mu^2\right)\\
&= \sum_{i=1}^k \E\bigl[\left(\mu - \Xx_i\right)^2\bigr] = \sum_{i=1}^k \text{Var}\left(\mu - \Xx_i\right) = k \sigma^2.
\end{align}
The steps come from rearranging the terms, using the independence of $\Xx_i$ with respect to $\Kk$, the independence of $\Xx_i, \Xx_j$ for $i \neq j$, and finally that $\E[\Xx_i] = \mu$ and $\text{Var}\left(\Xx_i\right) = \sigma^2$.

\end{proof}

\begin{lemma}\label{lemma:DU_Shapley_approx_error_bound}
Let $I\in\N$, $N := \{n_1,\ldots ,n_I\} \in \R_+^I$, and define,
\begin{align}
&\mu = \frac{1}{I} \sum_{i=1}^I n_i,\quad \sigma^2 = \frac{1}{I}\sum_{i=1}^I (n_i-\mu)^2,\\
&R := \max_{i \in [I]} |n_i - \mu|, \quad n^{\text{max}} = \max_{i\in [I]} n_i.
\end{align}
Consider $\mathcal{S}_k$, $n_{\mathcal{S}_k}$, $K$, and $\Yy$ as in Lemma \ref{lemma:pre_error_bound_result}. 
Let $w:\R_+\to \R$ be a smooth and increasing function, 
and $\rho > 0$, such that,
\begin{align}
\bigl|w^{(2)}(n)\bigr|\leq \rho \frac{|w(n)|}{n^2}, \forall n \bi 0,
\end{align} 
where $w^{(k)}$ is the k-th derivative of $w$. Then, it holds,
\begin{align}\label{eq:upper_and_lower_bound_general_Shaple_value}
\bigl|\E[w(\mu\Kk) - w(\Yy)]\bigr| \leq \frac{\rho |w(\mu I)|}{2 \mu^2I} \left(9\sigma^2 (1+\ln(I)) + 2R^2n^{\text{max}}\right).
\end{align}
\end{lemma}

\begin{proof}
The proof considers a second-order Taylor extension of $w$ at $\mu k$ to recover the expected value of $\E[w(\mu\Kk) - w(\Yy)]$. Noticing that the first derivative has a null expected value, the upper bound stated on the Lemma comes from bounding the expected value of the second derivative.

The Taylor-Lagrange Theorem on $w$ at $\mu k>0$ provides,
\begin{align}
       w(y)  =  w(\mu k ) + w^{(1)}(\mu k )(\mu k - y ) + w^{(2)}(\tau)\frac{(\mu k - y )^2 }{2},
\end{align}
for some $\tau$ between $y$ and $\mu k$. Therefore, there exists a random variable $\mathrm{T}$, almost surely between $\mu \Kk_+$ and $\Yy$, such that,
\begin{align}
      \E[w(\Yy) - w(\mu\Kk_+)]& = \E\biggl[ w^{(1)}(\mu\Kk_+)(\mu\Kk_+ -\Yy) + \frac{1}{2} w^{(2)}(\mathrm{T})(\mu\Kk_+ -\Yy) ^2\biggr].
\end{align}
where $\Kk_+$ corresponds to $\Kk$ conditioned to be positive. To avoid overcharging the notation, we drop the index from $\Kk_+$. We observe that,   
\begin{align}
    \E \biggl[w^{(1)}(\mu\Kk)(\mu\Kk -\Yy)\biggr] &= \E\biggl[\E \bigl[w^{(1)}(\mu\Kk)(\mu\Kk -\Yy)\mid \Kk = k \bigr] \biggr]\\
    &= \E\biggl[w^{(1)}(\mu k) \E \bigl[(\mu\Kk -\Yy)\mid \Kk = k\bigr] \biggr] = 0,
\end{align}
by Lemma \ref{lemma:pre_error_bound_result}, Equation \eqref{eq:Y_and_muK_have_the_same_exp_value}. Therefore, 
 \begin{align}
\bigl|\E[w(\Yy) - w(\mu\Kk)]\bigr| &= \frac{1}{2}\bigl| \E\bigl[ w^{(2)}(\mathrm{T})(\mu\Kk -\Yy)^2\bigr]\bigr| \leq \frac{1}{2} \E\bigl[ \bigl| w^{(2)}(\mathrm{T}) \bigr| (\mu\Kk -\Yy) ^2 \bigr] \\
&\leq \frac{\rho}{2} \E\biggl[\frac{|w(\mathrm{T})|}{\mathrm{T}^2}(\mu \Kk - \Yy)^2 \biggr] \leq \frac{\rho |w(I\mu)|}{2} \E\biggl[\frac{1}{\mathrm{T}^2}(\mu \Kk - \Yy)^2 \biggr].
\end{align}
Setting $\Ii:=\bigl\{|\mu\Kk -\Yy|\leq\frac{1}{2}(\mu\Kk+\Yy)\bigr\}$, the previous expected value can be expressed as,
\begin{align}
\E\biggl[\frac{1}{\mathrm{T}^2}(\mu \Kk - \Yy)^2 \biggr] = &\E \biggl[ \frac{1}{\mathrm{T}^2}(\mu\Kk -\Yy) ^2 \cdot \Ii\biggr] + \E \biggl[ \frac{1}{\mathrm{T}^2}(\mu\Kk -\Yy) ^2 \cdot \Ii^c\biggr].
\end{align}
We deal with each term separately. Notice that, as $\mathrm{T}$ is almost surely between $\Yy$ and $\mu\Kk$,
\begin{align}
|\mu\Kk - \Yy| \leq \frac{1}{2}(\mu\Kk + \Yy) \Longrightarrow \mathrm{T} \geq \frac{1}{3}\mu\Kk.
\end{align}
Thus,
\begin{align}
\E\left[ \frac{(\mu\Kk -\Yy) ^2}{\mathrm{T}^2} \cdot  \Ii \right]
&\leq \E\left[ \frac{(\mu\Kk -\Yy) ^2}{(\frac{\mu\Kk}{3})^2} \cdot  \Ii \right]\\
&= \frac{9}{\mu^2} \sum_{k=1}^I \frac{1}{I}\cdot\E\left[ \frac{(\mu k -\Yy) ^2}{k^2} \cdot  \Ii \mid\Kk=k\right]\\
&\leq \frac{9}{I\mu^2} \sum_{k=1}^I\E\left[ \frac{(\mu k -\Yy) ^2}{k^2} \mid\Kk=k\right]\\
&\leq \frac{9}{I\mu^2} \sum_{k=1}^I \frac{k\sigma^2}{k^2}
= \frac{9\sigma^2}{I \mu^2} \sum_{k=1}^I\frac{1}{k} 
\leq \frac{9\sigma^2}{I\mu^2}\cdot(1+\ln(I)).
\end{align}
Regarding the second term, as $\Kk \leq \min\{\mu\Kk,\Yy\} \leq \mathrm{T}$, we have,
\begin{align}
\E\left[ \frac{(\mu\Kk -\Yy) ^2}{\mathrm{T}^2} \cdot\Ii^c\right] 
&\leq \E\left[\frac{(R\Kk)^2}{\mathrm{T}^2} \cdot\Ii^c\right] \leq \E\left[ \frac{(R\Kk)^2}{\Kk^2} \cdot\Ii^c\right]\\
&= \frac{R^2}{I}  \sum_{k=1}^I \E\left[\frac{1}{k^2} k^2 \cdot\Ii^c\mid \Kk =k\right]\\  
&=\frac{R^2}{I}  \sum_{k=1}^I \mathbb{P}\left(|\mu k -\Yy| > \frac{1}{2}(\mu k+\Yy)\mid \Kk =k\right)\\
&\leq \frac{R^2}{I} \sum_{k=1}^I \mathbb{P}\left(|\mu k -\Yy| > \frac{ \mu k}{2}\mid \Kk =k\right)\\
&\leq \frac{R^2}{I} \sum_{k=1}^I \exp\biggl(- \frac{\mu^2 k}{2n^{\text{max}}}\biggr)\\
&= \frac{2R^2n^{\text{max}}}{I\mu^2} \sum_{k=1}^{I} \frac{\mu^2}{2n^{\text{max}}} \exp\biggl(- \frac{\mu^2 k}{2n^{\text{max}}}\biggr)\\
&\leq \frac{2R^2n^{\text{max}}}{I\mu^2} \int_{0}^{\infty} \frac{\mu^2}{2n^{\text{max}}} \exp\biggl(- \frac{\mu^2 k}{2n^{\text{max}}}\biggr) dk = \frac{2R^2n^{\text{max}}}{I\mu^2},
\end{align}
as the integral corresponds to the cumulative distribution function of an exponential random variable of parameter $\lambda = \mu^2/2n^{\text{max}}$. The upper bound on the theorem's statement is obtained when putting all together.
\end{proof}

\begin{lemma}\label{lemma:Assumption_on_second_derivative}
Let $w: \mathbb{R}_+ \to \mathbb{R}_+$ be a smooth and increasing function such that $$\lim_{n\to \infty} n^2 |w^{(2)}(n)| \sm \infty.$$ Then, there exists $\rho \bi 0$ such that $n^2 |w^{(2)}(n)|\leq \rho w(n)$.
\end{lemma}

\begin{proof}
Notice that the assumptions imply, in particular, that $|w^{(2)}(n)|$ is bounded. We argue by contradiction. Suppose that for any $m \bi 0$, there exists $n_{m}$ such that 
$$n_{m}^2 |w^{(2)}(n_{m})| > m w(n_{m}).$$ 
Suppose the sequence $(n_m)_m$ converges to a point $n^*$. Then,
\begin{align}
    \lim_{m\to \infty} n_m^2 |w^{(2)}(n_m)| \bi \lim_{m\to \infty} m w(n_m) = \infty,
\end{align}
which is a contradiction with $|w^{(2)}(n_m)|$ being bounded. Therefore, necessarily $(n_m)_m$ has to diverge. However, this implies,
\begin{align}
    \lim_{n\to \infty} n^2 |w^{(2)}(n)| = \lim_{m\to \infty} n_m^2 |w^{(2)}(n_m)| \bi \lim_{m\to \infty} m w(n_m) = \infty,
\end{align}
as $w$ is increasing, obtaining again a contradiction.
 
\end{proof}

\subsubsection{Proof of Theorem \ref{theorem:DU_shapley_error_bound}}

We are ready to prove Theorem \ref{theorem:DU_shapley_error_bound}.

\textbf{Theorem} \ref{theorem:DU_shapley_error_bound}.
\textit{
Assume \textbf{H}\ref{ass} holds. Then, there exists a constant $\rho \bi 0$ such that, for any $i \in \mathcal{I}$, the approximation error of \texttt{DU-Shapley} for player $i \in \mathcal{I}$ is upper bounded by
\begin{align}
    \bigl|\varphi_i - \psi_i \bigr| \leq \frac{\rho |w(n_{\mathcal{I}\setminus\{i\}})|}{(I-1) \mu_{-i}^2} \left(9\sigma_{-i}^2 (1+\ln(I-1)) + 2R_{-i}^2n^{\mathrm{max}}_{-i}\right),
\end{align}
where, $\varphi_i$ and $\psi_i$ are defined in \eqref{eq:Shapley_def2} and \eqref{eq:DU-shapley}, respectively; and where $\mu_{-i} = \frac{1}{I-1}n_{\mathcal{I}\setminus\{i\}}$, $\sigma^2_{-i} = \frac{1}{I-1}\sum_{j\in \mathcal{I}\setminus\{i\}} (n_j-\mu_{-i})^2$, $R_{-i} := \max_{j \in \mathcal{I}\setminus\{i\}} |n_j - \mu_{-i}|$, and $n^{\mathrm{max}}_{-i} := \max_{j \in \mathcal{I} \setminus\{i\}} n_j$.
}

\begin{proof}
Under Assumption \textbf{H}\ref{ass}, Lemma \ref{lemma:Assumption_on_second_derivative} implies the existence of $\rho \bi 0$ such that the value function $w$ satisfies all assumptions from Lemma \ref{lemma:DU_Shapley_approx_error_bound}. Theorem \ref{theorem:DU_shapley_error_bound} comes from (a) noticing that 
\begin{align}
\varphi_i &= \mathbb{E}[w(\Yy_{-i} + n_i) - w(\Yy_{-i})]\\
\psi_i &= \mathbb{E}[w(K\mu_{-i} + n_i) - w(K\mu_{-i})]    
\end{align}
where $K \sim \mathrm{U}([I-1])$ and $\Yy_{-i} = n_{\mathcal{S}^{(i)}_K}$ with $\mathcal{S}^{(i)}_K$ taking values on the subsets of $\mathcal{I}\setminus\{i\}$ of size $K$, (b) writing
\begin{align}
    |\varphi_i - \psi_i| &\leq |\mathbb{E}[w(\Yy + n_i) - w(K\mu_{-i} + n_i) ]| +  |\mathbb{E}[w(\Yy) - w(K\mu_{-i})]|,
\end{align}
and (c) applying Lemma \ref{lemma:DU_Shapley_approx_error_bound} to each of the expected values, as the function $n \to w(n + n_i)$ also satisfies \textbf{H}\ref{ass}.
\end{proof}

\section{DU-Shapley++ deduction}

In this section we briefly provide an intuition on how we shifted from DU-Shapley to DU-Shapley++,  
\begin{align}\label{eq:DU_Shapley_++}
    \psi_i^{++}(w) = \frac{1}{I}\bigl(w(n_i) + w(n_{\mathcal{I}}) - w(n_{\mathcal{I}\setminus\{i\}})\bigr) + \frac{1}{I}\sum_{k=1}^{I-2} \br{w(k\mu_{-i} + n_i) - w(k\mu_{-i})}\eqsp.
\end{align}
First of all, observe  that DU-Shapley++ is at least as good as DU-Shapley. Recall the Shapley value expression 
\begin{align}
    \varphi_i(u) &= \frac{1}{I}\sum_{k=0}^{I-1} \sum_{\mathcal{S} \subseteq \mathcal{I} \setminus \{i\}: |\mathcal{S}| = k} \frac{1}{\binom{I-1}{|\mathcal{S}|}} [w(n_\mathcal{S} + n_i) - w(n_\mathcal{S})],
\end{align}
which is obtained from Equation (\ref{eq:Shapley_def2}) when sorting coalitions per cardinality. Notice there exists only one coalition of size $0$ (the empty set) and only one of size $I-1$ (the whole coalition $\mathcal{I}\setminus \{i\})$, thus,
\begin{align}
    \varphi_i(u) &= \frac{1}{I} \bigl(w(n_{\varnothing} + n_i) - w(n_{\varnothing}) + w(n_{\mathcal{I}\setminus\{i\}} + n_i) - w(n_{\mathcal{I}\setminus\{i\}})\bigr)\\
    &\quad +
    \frac{1}{I}\sum_{k=1}^{I-2} \sum_{\mathcal{S} \subseteq \mathcal{I} \setminus \{i\}: |\mathcal{S}| = k} \frac{1}{\binom{I-1}{|\mathcal{S}|}} [w(n_\mathcal{S} + n_i) - w(n_\mathcal{S})]\\
    &= \frac{1}{I}\bigl(w(n_i) + w(n_{\mathcal{I}}) - w(n_{\mathcal{I}\setminus\{i\}})\bigr)
    +
    \frac{1}{I}\sum_{k=1}^{I-2} \sum_{\mathcal{S} \subseteq \mathcal{I} \setminus \{i\}: |\mathcal{S}| = k} \frac{1}{\binom{I-1}{|\mathcal{S}|}} [w(n_\mathcal{S} + n_i) - w(n_\mathcal{S})],
\end{align}
where we have used that both $w(n_{\varnothing})$ and $n_{\varnothing}$ are null. DU-Shapley++ comes from approximating only the \textit{middle terms}, i.e., from imposing:
$$\forall k \in \{1,...,I-2\}:  \sum_{\mathcal{S} \subseteq \mathcal{I} \setminus \{i\}: |\mathcal{S}| = k} \frac{1}{\binom{I-1}{|\mathcal{S}|}} [w(n_\mathcal{S} + n_i) - w(n_\mathcal{S})] \approx w(k\mu_{-i} + n_i) - w(k\mu_{-i}),$$
obtaining Equation (\ref{eq:DU_Shapley_++}). Notice that $\psi_i^{++}$ needs the same number of valuations than $\psi_i$. An illustration of why considering extreme cases might benefit our approximation is given in Figure \ref{fig:approx_uniform}.

\section{Additional details regarding numerical experiments}

In this section, we provide further details regarding the numerical experiments conducted in the main paper.

\subsection{Owen's Shapley value approximation}

In Section \ref{subsec:synthetic}, we considered the Shapley value approximation referred to as \texttt{Owen-Shapley} as a state-of-the-art competitor to \texttt{DU-Shapley}. 
We provide in the following additional details regarding \texttt{Owen-Shapley}.
For the other competitors, we directly refer the interested reader to \citet{JMLR:v23:21-0439}.

Owen \citep{Owen} studied the multilinear extension of a cooperative game and an alternative way to express the Shapley value. Formally, a cooperative game $G = (\mathcal{I},u)$ consists on a set of $I$ players $\mathcal{I} = \{1,2,...,I\}$ and a value function $u: 2^{\mathcal{I}} \to \mathbb{R}$ such that, for any $S \subseteq \mathcal{I}$, $u(S)$ corresponds to the value generated by the coalition $S$. The multilinear extension of $G$, denoted $\bar{G} = (\mathcal{I},\bar{u})$, is obtained when considering the value function $\bar{u} : [0,1]^{\mathcal{I}} \to \mathbb{R}$ given by,
\begin{align}
    \bar{u}(x_1,x_2,...,x_I) = \sum_{S \subseteq \mathcal{I}} \prod_{i \in S} x_i \prod_{j\notin S} (1-x_i) u(S).
\end{align}
Intuitively, $\bar{u}(x_1,x_2,...,x_I)$ corresponds to the expected value of a coalition when each player $i \in \mathcal{I}$ joins the coalition with probability $x_i$. Theorem 5 in \cite{Owen} gives an alternative way to compute the Shapley value $\varphi_i(u)$ of player $i$ in game $G$, namely,
\begin{align}
    \varphi_i(u) &= \int_{0}^1 \frac{\partial \bar{u}}{\partial x_i}(\tau,...,\tau) \mathrm{d}\tau = \int_{0}^1 \sum_{S \subseteq \mathcal{I}\setminus\{i\}} \tau^{|S|}(1-\tau)^{I-|S|-1}[u(S\cup\{i\}) - u(S)] \mathrm{d}\tau \\
    &=\int_0^1 \mathbb{E}\bigl[u(\mathcal{E}_i(\tau) \cup i) - u(\mathcal{E}_i(\tau))\bigr]\mathrm{d}\tau = \mathbb{E}_{\tau\sim \mathrm{U}([0,1])} \biggl[ \mathbb{E}\bigl[u(\mathcal{E}_i(\tau) \cup i) - u(\mathcal{E}_i(\tau))\bigr]\biggr],
\end{align}
where $\mathcal{E}_i(\tau)$ is a random subset of $\mathcal{I}\setminus\{i\}$, such that, each agent belongs to it with probability $t$. In words, the Shapley value of player $i$ corresponds to the expected marginal contributions of the random set $\mathcal{E}_i(\tau)$, when $\tau$ is uniformly distributed on $[0,1]$. This brings an alternative way to use Monte Carlo to approximate the Shapley value $\varphi_i(u)$, coined Owen-Shapley, as,
\begin{align}
    \hat{\varphi}_i^{\text{Owen}}(u) = \frac{1}{T}\sum_{t=1}^T u(\mathcal{E}_i(\tau_t) \cup i) - u(\mathcal{E}_i^t(\tau_t)),
\end{align}
where for each $t \in \{1,...,T\}$, we draw $\tau_t$ independently and uniformly in   $[0,1]$ and then, create a random set $\mathcal{E}_i(\tau_t)$ by adding each agent $j \in \mathcal{I} \setminus\{i\}$ to it with probability $\tau_t$.